\documentclass{article}

    \PassOptionsToPackage{numbers, sort}{natbib}
 \usepackage[preprint]{neurips_2025}

\usepackage[utf8]{inputenc} % allow utf-8 input
\usepackage[T1]{fontenc}    % use 8-bit T1 fonts
\usepackage{url}            % simple URL typesetting
\usepackage{booktabs}       % professional-quality tables
\usepackage{amsfonts}       % blackboard math symbols
\usepackage{nicefrac}       % compact symbols for 1/2, etc.
\usepackage{microtype}      % microtypography
\usepackage{xcolor}         % colors
\usepackage{graphicx}
\usepackage{amsmath}
\usepackage{algorithm}
\usepackage{algorithmic}
\usepackage{amssymb}
\usepackage[english]{babel}
\usepackage{wrapfig}

\usepackage{tikz}
\usepackage{microtype}
\usetikzlibrary{positioning,calc}

\definecolor{progressgreen}{HTML}{30B265}
\definecolor{progresslightgreen}{HTML}{D6F0E0}
\definecolor{paperblack}{HTML}{111111}
\definecolor{paperdarkgray}{HTML}{252525}
\definecolor{papermidgray}{HTML}{707070}
\definecolor{paperlightgray}{HTML}{C9C9C9}
\definecolor{paperpalegray}{HTML}{F2F2F2}
\definecolor{papergridgray}{HTML}{D8D8D8}
\colorlet{softgreen}{progresslightgreen}
\colorlet{softneutral}{paperpalegray}
\colorlet{highlightgreen}{progresslightgreen}
\colorlet{bordergray}{papermidgray}
\colorlet{textgray}{paperdarkgray}

\usetikzlibrary{
  arrows.meta,
  calc,
  decorations.markings,
  matrix
}
\colorlet{ppmgreen}{progressgreen}
\colorlet{ppmlightgreen}{progresslightgreen}
\colorlet{ppmgridgray}{papergridgray}
\colorlet{ppmdarkpath}{paperdarkgray}
\colorlet{ppmmidpath}{papermidgray}
\colorlet{ppmlightpath}{paperlightgray}
\colorlet{ppmrefpath}{paperblack}
\colorlet{ppmblack}{papermidgray}

\newcommand{\greenhl}[1]{%
  \begingroup
  \setlength{\fboxsep}{3pt}%
  \setlength{\fboxrule}{0.45pt}%
  \fcolorbox{progressgreen}{highlightgreen}{#1}%
  \endgroup
}

\usepackage{xcolor}
\usepackage{tikz}

\usetikzlibrary{arrows.meta}

\colorlet{titletext}{paperblack}
\colorlet{nodetext}{paperblack}
\colorlet{edgegray}{papermidgray}

\colorlet{neutralfill}{paperpalegray}
\colorlet{neutralborder}{papermidgray}

\colorlet{pointfill}{progresslightgreen}
\colorlet{pointborder}{progressgreen}

\colorlet{goalgreen}{progressgreen}

\tikzset{
    reasoning graph edge/.style={
        draw=edgegray,
        line width=0.85pt,
        -{Latex[length=2.4mm,width=1.9mm]},
        shorten >=0.42cm,
        shorten <=0.42cm
    },
    graph point/.style={
        circle,
        minimum size=0.78cm,
        inner sep=0pt,
        line width=0.8pt,
        font=\rmfamily\fontsize{16}{18}\selectfont,
        text=nodetext
    },
    start point/.style={
        graph point,
        draw=neutralborder,
        fill=white
    },
    reasoning point/.style={
        graph point,
        draw=pointborder,
        fill=pointfill
    },
    goal point/.style={
        graph point,
        draw=goalgreen,
        fill=goalgreen,
        text=white
    },
    graph title/.style={
        font=\rmfamily\fontsize{18}{20}\selectfont,
        text=titletext
    }
}

\tikzset{
  feedback/.style={
    draw=bordergray,
    line width=0.65pt,
    rounded corners=2pt,
    text=textgray,
    text width=5cm,
    minimum height=1.66cm,
    inner xsep=10pt,
    inner ysep=8pt,
    align=center,
    font=\large
  },
  positive/.style={feedback, draw=ppmgreen, fill=softgreen},
  negative/.style={feedback, fill=white},
  explanation/.style={
    draw=bordergray,
    fill=softneutral,
    line width=0.8pt,
    rounded corners=2pt,
    text=textgray,
    text width=6cm,
    minimum height=6.95cm,
    inner xsep=10pt,
    inner ysep=10pt,
    align=left,
    font=\large
  }
}

\tikzset{
  ppm grid edge/.style={
    draw=ppmgridgray,
    line width=0.55pt
  },
  ppm state/.style={
    circle,
    draw=ppmblack,
    fill=white,
    minimum size=5.0mm,
    inner sep=0pt,
    line width=0.6pt
  },
  ppm start state/.style={
    ppm state,
    font=\itshape
  },
  ppm goal state/.style={
    ppm state,
    draw=ppmgreen,
    line width=0.9pt,
    font=\itshape
  },
  ppm goal rewarded/.style={
    ppm goal state,
    fill=ppmgreen,
    text=white
  },
  ppm credited/.style={
    ppm state,
    fill=ppmlightgreen,
    draw=ppmgreen,
    line width=0.8pt
  },
  ppm reference/.style={
    draw=ppmrefpath,
    line width=0.8pt,
    postaction={decorate},
    decoration={
      markings,
      mark=at position 0.58 with {
        \arrow{Stealth[length=2.7mm,width=1.9mm]}
      }
    }
  },
  ppm rollout/.style={
    draw=ppmdarkpath,
    dashed,
    line width=1pt,
    line cap=round,
    line join=round
  },
  ppm rollout two/.style={
    draw=ppmmidpath,
    dashed,
    line width=1pt,
    line cap=round,
    line join=round
  },
  ppm rollout three/.style={
    draw=ppmlightpath,
    dashed,
    line width=1pt,
    line cap=round,
    line join=round
  },
  ppm legend label/.style={
    anchor=west,
    font=\small,
    text=ppmrefpath
  },
  ppm legend ref arrow/.style={
    draw=ppmrefpath,
    line width=0.9pt,
    -{Stealth[length=2.7mm,width=1.9mm]}
  },
  ppm legend matrix/.style={
    matrix of nodes,
    row sep=0pt,
    column sep=1mm,
    nodes={anchor=center}
  }
}

\newcommand{\canonicalterm}[1]{\textcolor{red}{#1}}
\newcommand{\alternativeterm}[1]{\textcolor{purple}{#1}}

\newcommand{\EE}{\mathbb E}
\newcommand{\Var}{\mathrm{Var}}
\newcommand{\PP}{\mathbb P}
\newcommand{\RR}{\mathbb R}
\newcommand{\gS}{\mathcal S}
\newcommand{\gA}{\mathcal A}
\newcommand{\sR}{\mathbb R}
\newcommand{\pl}{}

\newcommand{\snr}{\mathrm{SNR}}
\newcommand{\T}{^T}

\newcommand{\x}{\times}

\newcommand{\set}[1]{\left \{ #1 \right \}}
\newcommand{\parens}[1]{\left( #1 \right)}

\renewcommand{\vec}[1]{\if#1\relax\bm{#1}\else\mathbf{#1}\fi}

\newcommand{\ind}{\mathbf 1}

\newcommand{\methodname}{progressive point matching}
\newcommand{\Methodname}{Progressive point matching}
\newcommand{\METHODNAME}{Progressive Point Matching}

\usepackage{wrapfig}
\usepackage{enumitem}

\setlist[itemize]{
    itemsep=0.5em,
    topsep=0.2em,
    parsep=0pt,
    partopsep=0pt,
    leftmargin=2em
}
\setlist[enumerate]{
    itemsep=0.5em,
    topsep=0.2em,
    parsep=0pt,
    partopsep=0pt,
    leftmargin=2em
}

\usepackage{amsthm}
\usepackage{aliascnt}
\usepackage{xcolor}
\usepackage{hyperref}

\colorlet{linkgreen}{progressgreen}

\hypersetup{
  colorlinks=true,
  linkcolor=linkgreen,  % Internal links
  citecolor=linkgreen,  % Citation links
  urlcolor=linkgreen    % External URLs
}
\usepackage{cleveref}

\newtheorem{theorem}{Theorem}[section]

\newaliascnt{prop}{theorem}
\newtheorem{prop}[prop]{Proposition}
\aliascntresetthe{prop}

\newaliascnt{lemma}{theorem}
\newtheorem{lemma}[lemma]{Lemma}
\aliascntresetthe{lemma}

\newaliascnt{example}{theorem}
\newtheorem{example}[example]{Example}
\aliascntresetthe{example}

\newaliascnt{corollary}{theorem}
\newtheorem{corollary}[corollary]{Corollary}
\aliascntresetthe{corollary}

\Crefname{prop}{Proposition}{Propositions}
\Crefname{lemma}{Lemma}{Lemmas}
\Crefname{example}{Example}{Examples}
\Crefname{corollary}{Corollary}{Corollaries}

\usepackage{etoolbox}

\makeatletter
\patchcmd{\proof}
  {\topsep6\p@\@plus6\p@\relax}
  {\topsep0pt\@plus0pt\@minus0pt\relax}
  {}{}
\makeatother

\usepackage[most]{tcolorbox}

\definecolor{promptbg}{HTML}{F7F7F8}
\definecolor{promptframe}{HTML}{C9CDD3}
\definecolor{prompttitle}{HTML}{2F3A4A}

\usepackage{soul}

\sethlcolor{yellow!25}
\newcommand{\diff}[1]{\hl{#1}}

\newcommand{\promptheading}[1]{%
  \normalfont\bfseries\footnotesize #1%
}

\newtcblisting{llmprompt}[1][LLM Prompt]{
  listing only,
  breakable,
  enhanced,
  colback=promptbg,
  colframe=promptframe,
  boxrule=0.6pt,
  arc=2mm,
  left=3mm,
  right=3mm,
  top=1mm,
  bottom=1mm,
  title={#1},
  fonttitle=\bfseries\small,
  coltitle=prompttitle,
  attach boxed title to top left={xshift=2mm, yshift=-2mm},
  boxed title style={
    colback=white,
    colframe=promptframe,
    boxrule=0.5pt,
    arc=1.5mm,
    left=1.5mm,
    right=1.5mm,
    top=0.5mm,
    bottom=0.5mm
  },
  listing options={
    basicstyle=\ttfamily\scriptsize,
    breaklines=true,
    breakautoindent=false,
    breakindent=0pt,
    columns=fullflexible,
    keepspaces=true,
    showstringspaces=false,
    escapeinside={(*@}{@*)}
  }
}

\usepackage{titlesec}

\titlespacing*{\section}
  {0pt}    % left indent
  {0.5em}  % space before
  {0.3em}  % space after

\titlespacing*{\subsection}
  {0pt}    % left indent
  {0.4em}  % space before
  {0.25em}  % space after

\usetikzlibrary{arrows.meta}

\title{Long-Horizon Language Model Reinforcement Learning via \METHODNAME}

\author{%
  Preston Fu$^{1, *}$
  \quad
  Kevin Frans$^{1, *}$
  \quad
  Oleh Rybkin$^{1}$
  \quad
  \textbf{Sergey Levine}$^{1}$  
  \quad
  \textbf{Aviral Kumar}$^{2}$ \\[0.5em]
  $^{1}$UC Berkeley \quad $^{2}$Carnegie Mellon University\\[0.5em]
  \href{https://prestonfu.com/notes/ppm/}{\texttt{prestonfu.com/notes/ppm}}
}

\begin{document}

\maketitle

% Some todos:
% Right a proof about goal-reaching letting you skip states. This is more efficient under some assumption -> coverage under demonstrations vs. under optimal policy (concentratability coefficient)
% look at: CPI paper
% we will try to prove: unbiased and better sample complexity
% look at: alex agrawal PCPG
% what's the main insight: If we can someone construc this markovian state -> get this out of the rubrics
% focus: each part is kind of similar, but the fundamental insight is that if you can assign furthest state with markovian, adn this depends on all the parts.
% the big puzzle: how to reconcile the IL steps without being biased
% another graph: sparse + state-matching gets to asymptotic answer, but process rewards don't get there.

\begin{abstract}
Current paradigms for training language models via reinforcement learning rely heavily on sparse outcome rewards. However, as we pursue tasks that require longer and more complicated trajectories, such strategies result in slow learning. Prior work has attempted to address this problem by rewarding partial progress; however, naive formulations are often biased and converge to suboptimal policies. We show that a simple and \textit{unbiased} dense reward formulation, which we term \textbf{progressive point matching}, scales \textit{exponentially} more efficiently to long-horizon tasks by rewarding partial progress on a segment level, both theoretically and empirically via synthetic environments. We then show how progressive point matching can be practically instantiated using a single reference trajectory per task. On extremely hard math reasoning problems, sparse outcome rewards cannot make any progress, whereas segment-level rewards enable improvements at larger test-time token budgets when measured by success rate or pass@$k$.\begingroup
\renewcommand{\thefootnote}{\fnsymbol{footnote}}
\footnotetext[1]{Equal contribution. Correspondence to \url{prestonfu@berkeley.edu}.}
\endgroup
\end{abstract}

% For rl writing:
% - Start from the credit assignment problem. Exploration comes next as a sub point.
% - Motivate via IRL? Or state matching? Or estimate of the value function?
% 	- Let's do IRL route. We don't have a teacher. SFT don't work.

\vspace{-0.4em}
\section{Introduction}
\vspace{-0.4em}

\begin{figure}
    \centering
    \includegraphics[width=\linewidth]{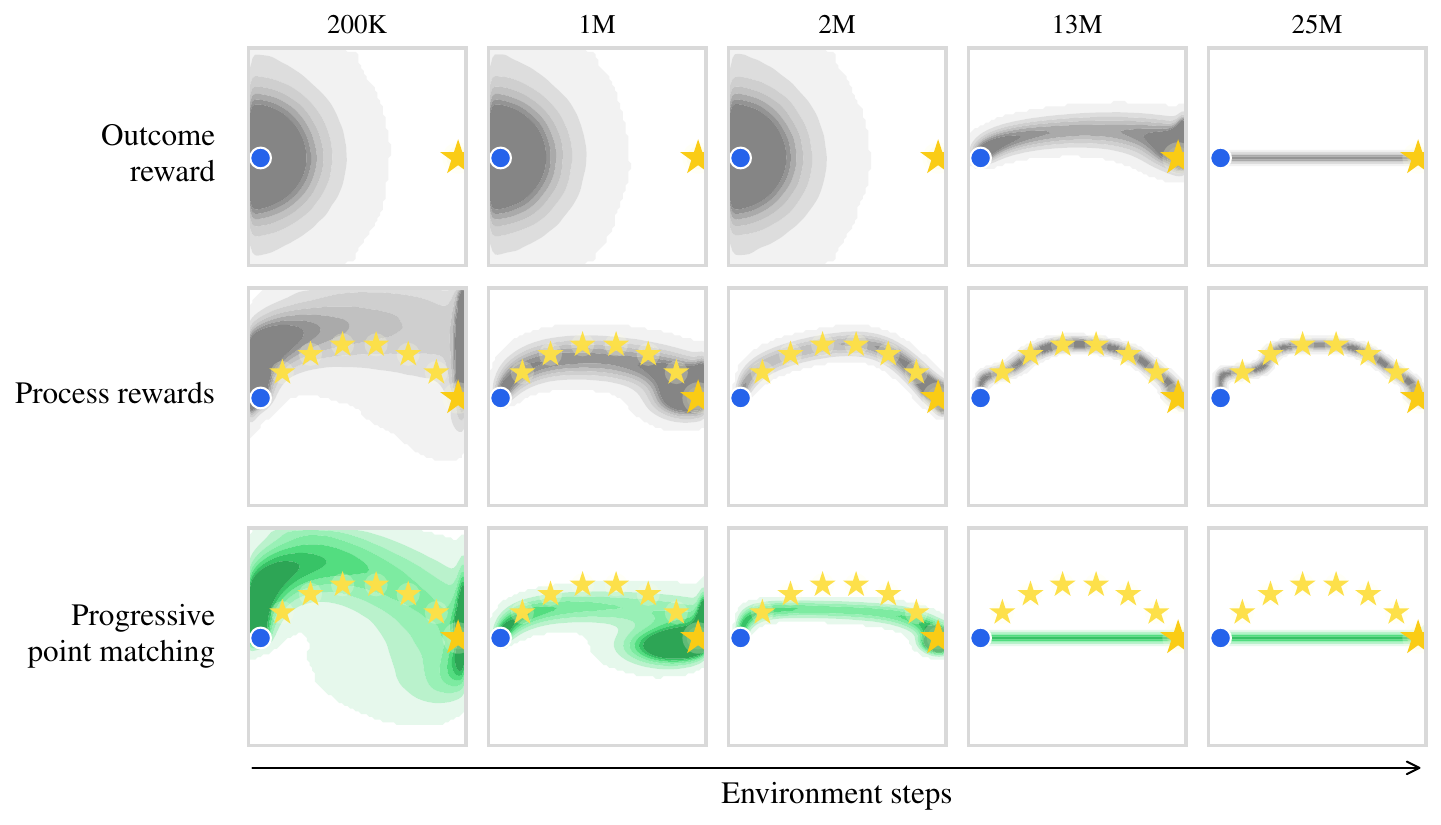}
    \vspace{-2em}
    % \caption{\textbf{
    % % With outcome-level rewards, the required number of samples for policy improvement increases exponentially with task horizon}, manifesting in slow learning.
    % Naive outcome-level rewards require exponentially more samples as task horizon increases, while \methodname\ reduces this requirement while converging to the optimal solution.}
    % \textbf{Top:} In a long-horizon simulated task, outcome-level supervision requires over 12M samples to see a single success. \textbf{Middle:} Given a set of reference points, one may combat this by providing auxiliary rewards for reaching each reference point, but biased objectives can produce suboptimal policies. \textbf{Bottom:} \Methodname\ reduces the effective horizon and can converge to the optimal policy.}
    \caption{
    As an analogy, consider a long-horizon problem in which an agent is rewarded for the correctness of its final answer. In this figure, we apply a discount factor to demonstrate that the goal can be reached suboptimally. \textbf{Top:} Sparse outcome rewards are unbiased, but successful trajectories become exponentially rare as the task horizon grows -- in our simulated task, over 12M samples are required to observe a single success. \textbf{Middle:} One may combat this by rewarding the agent for reaching each reasoning point, but this biases the agent toward suboptimality when the hints themselves are suboptimal. \textbf{Bottom:} \Methodname\ can efficiently learn the optimal policy under outcome rewards.
    }
    \label{fig:teaser}
\end{figure}

Reinforcement learning (RL) has proven a powerful tool for improving the reasoning and agentic capabilities of large language models (LLMs). The typical paradigm is to utilize large-scale policy gradients, with a sparse outcome reward generated at the end of a full trajectory. While straightforward to implement and verify, sparse rewards often result in weak signal per trajectory \citep{schulman2025lora} and brings a range of potential difficulties such as slow training and poor exploration \citep{ladosz2022exploration, andrychowicz2018hindsightexperiencereplay}. 

Such issues are especially prominent in long-horizon tasks, which require solving many subproblems before ultimately reaching the goal. % stringing together dozens of \confusingterm{correct actions}
%%AK.final: so i generally get the sense of what we mean by correct, but it is not clear to me what it means in reasoning
% before receiving any positive reward signal. 
% For example, in hard math problems, waypoints may include correctly computing intermediate quantities or proving lemmas. 
As we show in \Cref{fig:teaser}, sparse outcome reward methods require exponentially more trajectories to reach a desired pass rate as the task horizon increases, due to exponentially degrading signal-to-noise ratio of the policy gradient 
% \citep{Cheng2026isocompute}
(Theorem~\ref{thm:snr}). Indeed, the standard approach to tackle difficult problems wastes compute by generating many more trajectories than needed, then discarding a large chunk of those trajectories altogether~\citep{yu2025dapoopensourcellmreinforcement,Cheng2026isocompute}. In the worst case, the policy may be unable to sample a single success, resulting in no reward signal nor RL improvement.

How can we learn efficiently as we scale to longer and more difficult tasks? One natural answer is to utilize sources of reward signal other than the typical sparse outcome reward, which does not carry discriminative power beyond the final answer -- if a trajectory makes progress on dozens of subproblems but fails at the final stretch, it nevertheless is assigned a reward of zero. By properly assigning partial credit to each trajectory, we may alleviate this trend.

Our key insight is that reasoning problems can be regarded as discovering paths through a Markovian state space. One may decompose a long-horizon task into \textit{reasoning points}, which are sufficient statistics for subsequent reasoning. Thus, different prefixes of the trajectory that reach the same reasoning point are interchangeable with respect to future reasoning. For example, two trajectory prefixes that provide different proofs to the same lemma reach the same reasoning point, as subsequent reasoning can simply condition on the lemma without its proof. The \textit{reasoning state} is the \textit{set} of reached reasoning points, and the task becomes reaching a state that contains the goal. Unlike process rewards \cite{lightman2023letsverifystepstep}, which may not correlate with true task progress \cite{setlur2024rewarding, zhang2025lessons}, our method directly rewards reaching reasoning states that are closer to the goal state.

% Our key insight is that by framing reasoning problems as discovering paths through a \confusingterm{Markovian state space}, 
% %%AK.final: why a markovian state space?
% we can construct a \confusingterm{dense reward} particularly amenable to RL -- specifically, we aim to directly reward \textit{partial progress to the answer} as defined by a set of human or reference \alternativeterm{intermediate steps} \canonicalterm{[reasoning points]}. Our methodology is similar in motivation to \alternativeterm{process rewards} \canonicalterm{[process reward]} \cite{lightman2023letsverifystepstep,zhang2025lessons}, however unlike correctness-based methods which may simple not correlate with true task progress \cite{setlur2024rewarding, zhang2025lessons}, our method directly rewards reaching \confusingterm{states} closer and closer to the answer.

%%AK.final: this para is quite abrupt i feel.. Feels like it comes out of nowhere
By designing the reasoning points, \methodname\ enables interpolation between imitation learning and reinforcement learning. At one extreme, treating every token-level prefix of a reference trajectory as a
reasoning point approximates token-level imitation. At another, treating the final answer as the
sole reasoning point recovers RL with sparse outcome rewards. Between these extremes, reasoning points can be selected to compress disjoint segments of a reference trajectory. Rewarding reaching reasoning points via on-policy RL is less ``strict'' than directly imitating reference trajectories, which enables the policy to reach points through potentially more efficient reasoning paths (\Cref{fig:teaser}) while avoiding distribution-shift pitfalls of standard imitation learning \citep{Ross2011Reduction, agarwal2024policy}. Theoretically, \methodname\ leads to exponentially higher signal-to-noise than sparse outcome rewards (\Cref{thm:snr}) while enabling convergence to the optimal policy
  under outcome rewards (\Cref{prop:optimality-equivalence}).

% Similarly to imitation learning, \methodname\ provides \confusingterm{dense signal} by making use of task-solving reference trajectories. However, our method is fundamentally an RL algorithm that learns \textit{on-policy}, avoiding pitfalls from distribution shift that naive imitation learning often suffers from \citep{Ross2011Reduction, agarwal2024policy}. Intentionally, \methodname\ is less ``strict'' than directly imitating a teacher model -- the policy is not trained to mimic reference actions or trajectories exactly, but rather policies are free to reach \alternativeterm{intermediate points} \canonicalterm{[reasoning points]} in \textit{alternative} and potentially more efficient and easy-to-model ways (\Cref{fig:teaser}). As shown in Theorem~\ref{thm:snr}, we prove that under certain conditions, \methodname\ rewards lead to faster convergence than sparse outcome rewards while not limiting \confusingterm{asymptotic performance} \pf{we don't show this -- though this would be good to study actually}.

We instantiate our method as a modular modification to an off-the-shelf language model RL algorithm, e.g., PPO \cite{schulman2017proximal} or GRPO \citep{shao2024deepseekmath}. Reasoning points are constructed via privileged information, e.g.\ by using an LLM to extract them from a human or oracle solution. Importantly, we construct segment-level rewards, which allows to assign different advantages to tokens in a trajectory based on their relative contribution.

Our experimental setup isolates the effect of long-horizon scaling by focusing on three synthetic environments: a multi-turn version of Countdown \citep{gandhi2024stream}, a multi-turn Matrix Manipulation task \cite{stojanovski2025reasoninggymreasoningenvironments}, and synthetic math problems from GSM-Infinite \cite{zhou2025gsminfinitellmsbehaveinfinitely}. We show that as reasoning problems become harder, as characterized by requiring reaching more intermediate points to ultimately reach the goal, policies trained with sparse outcome rewards (i.e., standard GRPO) require exponentially more trajectories to converge. In contrast, policies trained via \methodname\ converge faster, with the gap in relative performance increasing with task horizon. We additionally show that \methodname\ allows us to train policies on \textit{near-zero success rate} math problems, where methods using sparse outcome rewards fail to learn completely. When evaluated at larger test-time token budgets, policies trained with \methodname\ extrapolate their ability to achieve a higher success rate and pass@$k$ rate according to outcome rewards.

% \begin{figure}
%     \centering
%     % \resizebox{1.02\linewidth}{!}{%
%     \resizebox{0.75\linewidth}{!}{%
%         % \input{imgs/fig1}%
%         \input{imgs/fig1_new}%
%     }
%     \vspace{-0.5em}
%     % \input{imgs/fig1}
%     % \includegraphics[width=0.9\linewidth]{imgs/fig-rewards.png}
%     %%SL.5.2: Generally I really like the idea behind this figure and I think it's really important to clearly explain the "shortcut" concept. I wonder if we can somehow make it even clearer
%     \caption{
%     \textbf{\Methodname\ provides dense signal to reward partial progress, while not limiting asymptotic performance.} 
%     Outcome rewards (left) are sparse and hard to extract signal from. In contrast, we reward \textit{partial progress} of policies as measured by the \textit{furthest} reached point of a reference trajectory (right). Thus, the learned policy is not asymptotically constrained by the reference dataset. \pf{add a simple language example}
%     % Policies trained with \methodname\ can discover diverse and alternate reasoning paths to those in the reference, such as paths that ``shortcut'' intermediate points (right). Thus, the resulting policy is not asymptotically constrained by the reference dataset.
%     % At the same time, \textbf{pol can improve \textit{beyond} the reference trajectory}, by locating paths that ``shortcut" intermediate states, thus policies are free to discover alternative (and more efficient) ways of reaching the answer.
%     }
%     \label{fig:method}
% \end{figure}

\vspace{-0.4em}
\section{Related Work}
\vspace{-0.4em}

%%SL.5.2: you're using \cite instead of \citep, which renders as non-parenthetical citations; assuming this is unintended, fix by using \citep

\textbf{Partial progress rewards.} 
Learning from sparse outcome rewards has been a long-standing challenge in RL due to their low signal and high exploration requirements. A number of prior works have proposed methodologies to reward \textit{partial progress}. 
In principle, multi-task value learning methods \citep{schaul2015universal, yue2025vapo, schulman2017proximal, snell2022offline} provide a notion of partial credit, as trajectories that fail to reach the answer can nevertheless be rewarded by a learned estimate of their answer-reaching likelihood. 
Hindsight relabeling methods \cite{andrychowicz2018hindsightexperiencereplay, zhang2023wisdom} also provide signal by rewriting the original goal such that failing trajectories now succeed (for a different goal). 

Particularly for language models, a large body of work has studied process rewards, which provide feedback to intermediate reasoning by checking whether each segment is logically correct. This feedback may come from a human annotated dataset \cite{lightman2023letsverifystepstep, uesato2022solving, zheng2025processbench} or via querying a language model judge \cite{miao2023selfcheck, khalifa2025process}. However, training on correctness rewards may not converge to an optimal policy under outcome rewards, as logical correctness may not correlate with task success \citep{setlur2024rewarding, zhang2025lessons}. Closest to our work are methods that do not rely on external definitions of correctness and attempt to predict success rate directly, largely via Monte Carlo estimation \cite{luo2024improvemathematicalreasoninglanguage,wang2024mathshepherdverifyreinforcellms, cui2025processreinforcementimplicitrewards}. While we share a similar viewpoint of directly rewarding task success, our methodology avoids learning an expensive reward or value model entirely, and instead uses per-task reference trajectories to provide partial credit.

\textbf{Imitation learning and inverse reinforcement learning.}
%%SL.5.2: maybe start with a sentence or so that just explains why we are even talking about imitation and IRL? eg our method has some similarity...
% Another source of dense signal is to learn from reference trajectories. 
Our work takes inspiration from imitation learning methods that derive signal from reference policies or trajectories. 
The simplest of these are behavioral cloning methods \cite{Pomerleau1988ALVINN}, which perform supervised learning on actions over a dataset of reference trajectories. Online imitation learning methods \cite{Ross2011Reduction, Ross2014Reinforcement, Sun2017DeeplyAggreVaTeD} instead aim to minimize distribution drift by directly operating on sampled trajectories, copying teacher actions at each state. 
If an explicit teacher model is not available, some methods resort to learning an adversarial critic \cite{Ho2016Generative, Ke2021Imitation, Zhang2020FGAIL}, which can be used to provide signal to the student. 
Other methods pose imitation learning as a \textit{state occupancy matching} problem \cite{Nachum2019DualDICE, Kostrikov2019ValueDICE, Ma2022SMODICE, Sun2019ProvablyEfficient, garg2022iqlearninversesoftqlearning}, which train policies to exactly match the steady-state distribution of a teacher. Our approach is most closely related to \cite{Schroecker2017StateAware, Reddy2020SQIL, Chiang2024ExpertProximity},
which directly construct on-policy rewards derived from demonstrations. While similar in inspiration, our method critically differs in that we do not aim to match state \textit{occupancies}, which involves matching the distribution of all intermediate states, but rather we reward progress as the set of reached reasoning points, giving credit to intermediate points via a \textit{shortcutting} mechanism. Thus, \methodname\ enables policies to learn fundamentally different goal-reaching strategies than the demonstrations.
% \kf{maybe rework that our contribution is how to do state occupancies in LLM land, and we do that via decomposing states into points.}

%%SL.5.2: I feel like this paragraph doesn't really answer how our method differs from imitation learning. The last sentence comes across as more of an implementation detail. Could it make sense to more directly address precisely what the distinction between our method and imitation really is?

\textbf{Language model training via privileged information.} Prior works have examined mechanics to best extract learning signal from demonstrations or reference solutions. Strategies involve placing privileged information in-context, e.g. partial solutions \citep{qu2026pope} or reasoning abstractions \citep{qu2025rlad, chen2025nudging}. These in-context policies can also be treated as a target for distillation \citep{zhao2026self, hubotter2026reinforcement, shenfeld2026self}. However, directly providing privileged information creates a risk of distribution shift \citep{zelikman2022star} or overfitting, characterized by a collapse in exploration ability \citep{qu2026pope}. In contrast, our method is most similar to works that learn fully on-policy and use privileged information only to construct rewards. This has often taken the form of using likelihood-based rewards \citep{chen2024language, tang2025learning, zhou2025reinforcing} or adversarial critics \citep{cai2025escaping}. However, instead of relying on model likelihoods, we draw a tight connection to goal-reaching MDPs and opt to directly reward partial progress.

\vspace{-0.4em}
\section{Preliminaries}
\vspace{-0.4em}

% We study RL post-training of a large language model with parameters $\theta$. Let $(s_0, \dots, s_{T-1}) \sim \rho$ denote a collection of input prompts, and let $h_t = (s_0, a_0, s_1, a_1, \dots, s_{t-1})$ denote the full interaction history before the $t$-th assistant turn. We model the interaction as an MDP whose state is the full history $h_t$, whose action is the next assistant message $a_t$, and whose transition appends the generated assistant message and next input prompt $s_{t+1}$ to the history. The policy $\pi_\theta(a_t \mid h_t)$ defines a distribution of state-action trajectories $p_{\pi_\theta}(\tau) = \rho(s_0, \dots, s_{T-1}) \prod_{t=0}^{T-1} \pi_\theta(a_t \mid h_t)$, where $\tau = (h_0, a_0, h_1, a_1, \dots, h_{T-1}, a_{T-1})$. Each trajectory terminates with a final answer, for example an answer appearing in a \texttt{\textbackslash boxed\{\}} block. The reward function is 

\textbf{Markov decision processes.} A Markov decision process (MDP) is defined by a state space, an action space $\gA$,
a transition function $p: \gS \times \gA \to \Delta(\gS)$, a reward function $r: \gS \to \sR$,
and an initial state distribution $p(\pl{s_0}) \in \Delta(\gS)$. 
For our setting, we are concerned only with a subset of MDPs suitable for language reasoning -- the reward function is nonzero \textit{only} at the outcome state, transition functions are deterministic, and the state is fully observable. 
Our policy is defined as a probability distribution over actions $\pi(\pl{a} \mid \pl{s}): \gS \to \Delta(A)$,
which defines the distribution of state-action trajectories
$p(\tau \mid \pi) = p(s_0) \prod_{t=0}^{T-1} p(s_{t+1} \mid s_t, a_t) \pi(a_t \mid s_t)$,
where $\tau = (s_0, a_0, \ldots, s_{T-1}, a_{T-1}, s_T)$.
The standard RL objective is to learn a parameterized policy $\pi_\theta$ that maximizes the expected sum of rewards: $J(\pi_\theta) = \EE_{\tau \sim p(\pl{\tau} \mid \pi_\theta)} \sum_{t=0}^{T} \gamma^t r(s_t)$.
% \begin{equation}
%     J(\pi_\theta) = \mathbb{E}_{\tau \sim p(\pl{\tau} \mid \pi_\theta)}  \sum_{t=0}^\infty \gamma^t r(s_t). \label{equation:rl}
% \end{equation}

\textbf{Language reasoning models as policies.} 
% Throughout this work, we consider the fine-tuning of large language models that utilize natural language reasoning. This can take the form of a chain-of-thought or an explicit "thinking" block. 
Throughout this work, we consider the fine-tuning of large language models that reason via autoregressive generation.
% We denote the prompt tokens as $x^\text{prompt}$, and the model-generated tokens as $x_{0..T}$.
We define a sequence of tokens as $x_{0..T}$ where a prefix of these tokens comes from an external prompt, and the following are generated autoregressively. We never take a loss on the prompt tokens.
We assume access to a base model, defined as a neural network that parameterizes a probability distribution $\pi_\theta(x_t \mid x_{<t})$.
To measure the performance of such models, we consider a verifiable setting in which a set of training problems are associated with a verifiable success criterion defined as an outcome reward $r(x)$.

% \textbf{Reinforcement learning objective.} We can define our core learning problem as maximizing the expected reward over the distribution of problems. In this work we utilize policy gradient methods to achieve this, which can be generally written as following the direction:
% \begin{equation}
%     \Delta \theta = \mathbb{E}_{s,a \sim \pi} \left[ \nabla_\theta \log \pi(a|s) \cdot A(s,a) \right]
% \end{equation}
% where $a$ represents an action, $s$ a Markovian state, and $A(s,a)$ is an estimate of future advantages. For reasoning models, we define $s$ as the sequence of previous tokens $x_{0..t}$. 

% \kf{Let's use trajectory instead of rollout.}

\vspace{-0.4em}
\section{Rewarding Partial Progress via \METHODNAME}
\vspace{-0.4em}

In this section, we describe a simple framework for assigning partial credit in RL, then follow with how language reasoning tasks can be best formulated into this framework. We aim to satisfy the following desiderata:
\begin{enumerate}
    \item[(1)] Segments within trajectories are assigned rewards proportional to their partial progress toward the solution,  even for problems where no trajectories reach the goal.
    \item[(2)] Optimal policies under  per-segment rewards are optimal under outcome rewards, and vice versa.
    %%AK.final: asymptotic in what sense? asymptotic over what variable
\end{enumerate}

% - Reasoning states are on the goal-reaching MDP.
% 	- Partial progress at any state is defined by that state's distance to the goal.
% 	- A state is the entire history of intermediate reasoning steps.
% - Reference trajectories.
% 	- One assumption we make is that the reasoning MDP is non-decreasing in partial progress. 
% 	- Approximate distance to the goal as distance to the goal under the reference.
% - Claims:

\begin{figure}
    \centering
        \resizebox{\linewidth}{!}{%
            \begin{tikzpicture}[node distance=2.4mm]

  \tikzset{
    rubric two heading/.style={
      font=\large,
      text=paperblack,
      align=center
    },
    rubric two grid/.style={
      draw=paperlightgray,
      line width=0.45pt,
      densely dotted
    },
    rubric two axis/.style={
      draw=papermidgray,
      line width=0.7pt
    },
    rubric two score/.style={
      draw=ppmgreen,
      line width=1.5pt,
      line cap=round,
      line join=round
    },
    rubric two score point/.style={
      circle,
      draw=ppmgreen,
      fill=white,
      line width=1.5pt,
      minimum size=3.5mm,
      inner sep=0pt
    },
    rubric two chunk/.style={
      draw=bordergray,
      fill=white,
      rounded corners=2pt,
      line width=0.75pt,
      inner sep=0pt
    },
    rubric two chunk text/.style={
      anchor=north west,
      text=paperdarkgray,
      text width=6.10cm,
      align=left,
      font=\normalsize,
      inner sep=0pt
    },
    rubric two chunk muted/.style={
      rubric two chunk,
      draw=paperlightgray,
      fill=paperpalegray
    }
  }

  % This is the rubric and response from rubric_example.tex.
  \node[positive] (step1)
    {Express $AB$ and $AC$ in terms of complex numbers.};

  \node[positive, below=of step1] (step2)
    {Determine the orientation of triangles $ABP$ and $BCQ$.};

  \node[negative, below=of step2] (step3)
    {The coefficients of $b^2$, $c^2$, and $bc$ must be equal.};

  \node[negative, below=of step3] (step4)
    {The final answer is correctly stated as $\tfrac{2}{3}$.};

  % The sampled response is exposed chunk-by-chunk. Chunk 2 starts exactly
  % where the second matched reasoning point becomes explicit.
  \node[
    rubric two chunk,
    anchor=north west,
    minimum width=6.80cm,
    minimum height=4.00cm
  ] (chunk1) at ($(step1.north east)+(4.5mm,0)$) {};
  \node[rubric two chunk text, anchor=center]
    at (chunk1.center)
    {%
      \textbf{Segment 1} --
      Let the vertices of the triangles be represented by the complex
      numbers $a,b,c,p,q,r,m$. We may place triangle $ABC$ in the complex
      plane by setting $a=0$. Then
      \greenhl{$u=b-a=b$} and \greenhl{$v=c-a=c$}
      are the complex numbers representing the vectors
      \greenhl{$\overrightarrow{AB}$ and $\overrightarrow{AC}$}.%
    };

  \node[
    rubric two chunk,
    anchor=north west,
    minimum width=6.80cm,
    minimum height=2.20cm
  ] (chunk2) at ($(chunk1.south west)+(0,-1.8mm)$) {};
  \node[rubric two chunk text, anchor=center]
    at (chunk2.center)
    {%
      \textbf{Segment 2} -- 
      Assume that $A,B,C$ are in counter-clockwise order. Then ``outside''
      means that the triangles $ABP$, $BCQ$, and $CAR$ are in
      \greenhl{clockwise order}.%
    };

  \node[
    rubric two chunk muted,
    anchor=north west,
    minimum width=6.80cm,
    minimum height=0.80cm
  ] (chunk3) at ($(chunk2.south west)+(0,-1.8mm)$)
    {};
  \node[rubric two chunk text, anchor=center, text=papermidgray]
    at (chunk3.center)
    {\textbf{Segment 3} -- Set $t=AR/AC$. Expand:};

  % A shared vertical offset keeps all three heading baselines aligned.
  \node[rubric two heading, anchor=base]
    at ($(step1.north)+(0,3mm)$)
    {Reference reasoning points};
  \node[rubric two heading, anchor=base]
    at ($(chunk1.north)+(0,3mm)$)
    {Sampled trajectory};
  \node[rubric two heading, anchor=base]
    at ($(chunk1.north east)+(3.85cm,3mm)$)
    {\Methodname\ measure};

  % Progressive point-matching score over the response prefix.
  \begin{scope}[shift={($(chunk1.north east)+(0.95,0)$)}]
    \foreach \y in {-6.35,-4.80,-3.25,-1.70,-0.15} {
      \draw[rubric two grid] (0,\y) -- (5.80,\y);
    }

    \draw[rubric two axis] (0,-6.35) -- (5.80,-6.35);
    \draw[rubric two axis] (0,-6.35) -- (0,-0.15);

    \foreach \x/\label in {
      0/0,
      1.160/1,
      2.320/2,
      3.480/3,
      4.640/4,
      5.800/5
    } {
      \draw[papermidgray, line width=0.5pt]
        (\x,-6.42) -- (\x,-6.28);
      \node[anchor=north, font=\normalsize, text=paperdarkgray]
        at (\x,-6.50) {\label};
    }

    \foreach \y/\label in {
      -6.35/$\nicefrac{0}{4}$,
      -4.80/$\nicefrac{1}{4}$,
      -3.25/$\nicefrac{2}{4}$,
      -1.70/$\nicefrac{3}{4}$,
      -0.15/$\nicefrac{4}{4}$
    } {
      \draw[papermidgray, line width=0.5pt]
        (-0.07,\y) -- (0.07,\y);
      \node[anchor=east, font=\normalsize, text=paperdarkgray]
        at (-0.16,\y) {\label};
    }

    \node[anchor=north, font=\large, text=paperdarkgray]
      at (2.90,-6.95) {Timestep};

    % Matching progress sampled at evenly spaced response prefixes.
    \draw[rubric two score]
      (1.16,-4.80)
      -- (2.32,-3.25)
      -- (3.48,-3.25)
      -- (4.64,-1.70)
      -- (5.80,-1.70);

    \foreach \x/\y in {
      1.16/-4.80,
      2.32/-3.25,
      3.48/-3.25,
      4.64/-1.70,
      5.80/-1.70 
    } {
      \node[rubric two score point] at (\x,\y) {};
    }
  \end{scope}

\end{tikzpicture}%
        }
    %     \begin{minipage}[t]{0.62\textwidth}
    %     \vspace{0pt}
    %     \centering
    %     \resizebox{\linewidth}{!}{%
    %         \input{imgs/rubric_example}%
    %     }
    % \end{minipage}
    % \hspace{-0.02\textwidth}
    % \hfill
    % \begin{minipage}[t]{0.375\textwidth}
    %     \vspace{-3pt}
    %     \centering
    %     \includegraphics[width=\linewidth]{imgs/prompt_climbing.pdf}
    % \end{minipage}
    % \includegraphics[width=1.0\linewidth]{imgs/fig-progress.png}
    \caption{\textbf{\Methodname\ rewards enable segment-level credit assignment.} Shown above is an example challenging problem from Omni-MATH. 
    \textbf{Left:} Rewards are constructed by measuring the number of reference reasoning points that a trajectory reaches.
    \textbf{Right:} \Methodname\ is a monotonic function over tokens. Successful trajectories will accumulate reward from 0 to 1 throughout the reasoning process, whereas unsuccessful trajectories may plateau when they deviate from reference trajectory.
    }
    \label{fig:progress}
\end{figure}
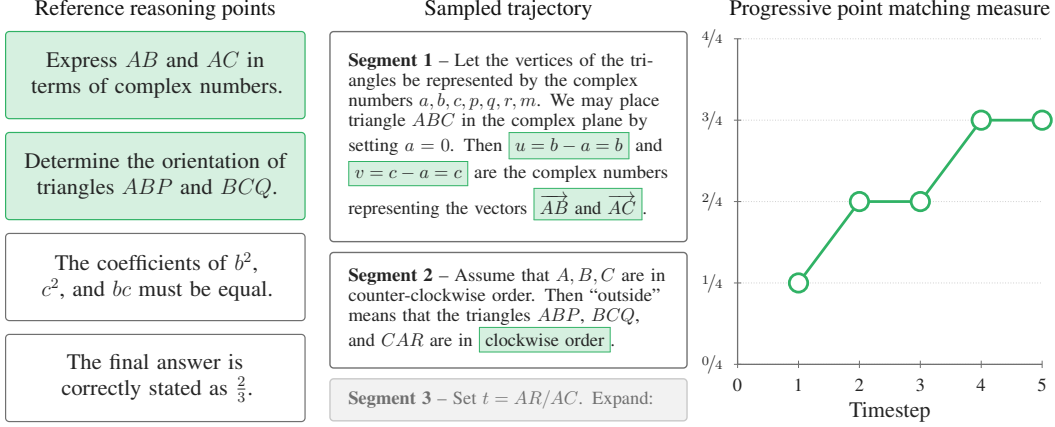

We start from the observation that, in typical long chain-of-thought reasoning trajectories, intermediate results compress earlier portions of the reasoning process. For example, in a math problem, a trajectory may prove a lemma that can then be used in subsequent steps of a longer proof without revisiting its derivation. We call such intermediate results \textit{reasoning points}. Prior work~\citep{qu2026pope} has used a similar mental model to develop guidance-conditioned RL training approaches. Accordingly, we operate under the assumption that reasoning trajectories admit compact state representations: the \textit{set} of reasoning points reached by the current trajectory prefix. We can then view a reasoning trajectory as a path through the state space of a goal-reaching MDP, where the goal state contains the goal point (i.e., the final answer), and each action adds one or more new reasoning points to the current set.

One implication of this formulation is that progress can never be harmed as a reasoning trace proceeds -- each new trajectory segment may add new reasoning point(s) to this set, or leave the set unchanged. This is intentional and often desirable, e.g.\ deriving intermediate lemmas may be helpful to proving a particular theorem, but in general do not prevent further progress -- one can always backtrack to a previous lemma. Thus, each reasoning trajectory is \textit{non-decreasing in partial progress} (\Cref{fig:progress}). Successful trajectories therefore follow a desirable pattern: starting at an empty set, they make non-decreasing progress, and end by reaching the goal point.

\textbf{\Methodname\ rewards}. To construct a measure of partial progress, we rely on the reasoning points spanned by a reference trajectory. 
%%AK.final: note the limitation that "while this may result in only learning one plausible solution approach to a problem, this is ok [why in 1 line]".
Over this set, we define the \methodname\ measure $\phi(s, s^{\text{ref}})$ (or simply $\phi(s)$) as the number of reference reasoning points $s^{\text{ref}} = \set{s_1^{\text{ref}}, \dots, s_{n-1}^{\text{ref}}, s_n^{\text{ref}} = g}$ contained at a given state $s$:
\begin{align}
    \phi(s, s^{\text{ref}}) & := \sum_{i=1}^n \mathbf{1}(s^{\text{ref}}_i \in s).
    \label{eq:point-matching}
\end{align}
%%AK.final: I think we have not given much intuition on what this reasoning point is supposed to be?
% The \methodname\ measure above is a non-decreasing function of the current \confusingterm{state}.
To derive per-segment rewards satisfying desiderata \textbf{(1)}, we can simply calculate the increase in the measure caused by that segment\footnote{This framework matches potential shaping \cite{ng1999policy} at non-terminal states. Let $r(s) := \mathbf{1}\{g \in s\}$ denote the sparse outcome reward, and define
$\Phi(s) := \phi(s) - r(s)$. Then, for non-terminal transitions,
\begin{equation}
\hat r(s_t,a_t)
= r(s_{t+1}) + \Phi(s_{t+1}) - \Phi(s_t),
\end{equation}
where $\gamma = 1$. Unlike potential shaping, however, \methodname\ retains the terminal potential for unsuccessful trajectories, providing partial credit for progress.}:
\begin{align}
    \hat{r}(s_t,a_t) = \phi(s_{t+1}) - \phi(s_t).
    \label{eq:reward}
\end{align}
% \begin{theorem}[\METHODNAME\ Reward Sufficiency]
% Every policy that maximizes $\hat{r}$ reaches the goal reasoning point $s^{\text{ref}}_n$, satisfying desiderata \textbf{(2)}.
% \end{theorem}
% \begin{proof}
% The reference trajectory contains all $n$ reference points by construction, so $\max_\pi \phi(s_T) = n$. Since $\sum_{t=0}^{T-1} \hat{r}(s_t, a_t)$ telescopes to $\phi(s_T) - \phi(s_0)$, any $\hat{r}$-optimal policy attains $\phi(s_T) = n$. Attaining the maximum value of $\phi$ requires containing every reference reasoning point, including the final goal point $s^{\text{ref}}_n$.  Hence any $\hat{r}$-optimal policy reaches the goal.
% \end{proof}

% \subsection{\METHODNAME\ on the Reasoning Graph}

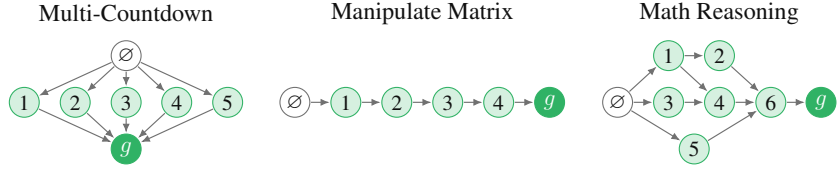
\begin{figure}[t]
    \centering
    \resizebox{0.8\textwidth}{!}{%
\begin{tikzpicture}[
    x=1cm,
    y=1cm
]

% Fix the overall dimensions and alignment.
\path[use as bounding box]
    (0,-0.45) rectangle (21.88,3.85);

% ================================================================
% Multi-Countdown
% ================================================================

\begin{scope}
    \node[graph title]
        at (3.20,3.55)
        {Multi-Countdown};

    \coordinate (m0) at (3.20,2.44);

    \coordinate (m1) at (0.54,1.22);
    \coordinate (m2) at (1.87,1.22);
    \coordinate (m3) at (3.20,1.22);
    \coordinate (m4) at (4.52,1.22);
    \coordinate (m5) at (5.85,1.22);

    \coordinate (mg) at (3.20,0.00);

    % Dependencies from the initial state.
    \draw[reasoning graph edge] (m0) -- (m1);
    \draw[reasoning graph edge] (m0) -- (m2);
    \draw[reasoning graph edge] (m0) -- (m3);
    \draw[reasoning graph edge] (m0) -- (m4);
    \draw[reasoning graph edge] (m0) -- (m5);

    % Each independent point contributes to the goal.
    \draw[reasoning graph edge] (m1) -- (mg);
    \draw[reasoning graph edge] (m2) -- (mg);
    \draw[reasoning graph edge] (m3) -- (mg);
    \draw[reasoning graph edge] (m4) -- (mg);
    \draw[reasoning graph edge] (m5) -- (mg);

    \node[start point]
        at (m0)
        {$\varnothing$};

    \node[reasoning point]
        at (m1)
        {1};

    \node[reasoning point]
        at (m2)
        {2};

    \node[reasoning point]
        at (m3)
        {3};

    \node[reasoning point]
        at (m4)
        {4};

    \node[reasoning point]
        at (m5)
        {5};

    \node[goal point]
        at (mg)
        {$g$};
\end{scope}

% ================================================================
% Manipulate Matrix
% ================================================================

\begin{scope}[shift={(7.605,0)}]
    \node[graph title]
        at (3.325,3.55)
        {Manipulate Matrix};

    \coordinate (a0) at (0.00,1.22);
    \coordinate (a1) at (1.33,1.22);
    \coordinate (a2) at (2.66,1.22);
    \coordinate (a3) at (3.99,1.22);
    \coordinate (a4) at (5.32,1.22);
    \coordinate (ag) at (6.65,1.22);

    \draw[reasoning graph edge] (a0) -- (a1);
    \draw[reasoning graph edge] (a1) -- (a2);
    \draw[reasoning graph edge] (a2) -- (a3);
    \draw[reasoning graph edge] (a3) -- (a4);
    \draw[reasoning graph edge] (a4) -- (ag);

    \node[start point]
        at (a0)
        {$\varnothing$};

    \node[reasoning point]
        at (a1)
        {1};

    \node[reasoning point]
        at (a2)
        {2};

    \node[reasoning point]
        at (a3)
        {3};

    \node[reasoning point]
        at (a4)
        {4};

    \node[goal point]
        at (ag)
        {$g$};
\end{scope}

% ================================================================
% Math Reasoning
% ================================================================

\begin{scope}[shift={(16.03,0)}]
    \node[graph title]
        at (2.65,3.55)
        {Math Reasoning};

    \coordinate (r0) at (0.00,1.22);

    \coordinate (r1) at (1.33,2.44);
    \coordinate (r2) at (2.66,2.44);

    \coordinate (r3) at (1.33,1.22);
    \coordinate (r4) at (2.66,1.22);

    \coordinate (r5) at (2.00,0.00);

    \coordinate (r6) at (4.00,1.22);
    \coordinate (rg) at (5.31,1.22);

    % Alternative branches from the initial state.
    \draw[reasoning graph edge] (r0) -- (r1);
    \draw[reasoning graph edge] (r0) -- (r3);
    \draw[reasoning graph edge] (r0) -- (r5);

    % Intermediate dependencies.
    \draw[reasoning graph edge] (r1) -- (r2);
    \draw[reasoning graph edge] (r1) -- (r4);
    \draw[reasoning graph edge] (r3) -- (r4);

    % Paths merge before reaching the goal.
    \draw[reasoning graph edge] (r2) -- (r6);
    \draw[reasoning graph edge] (r4) -- (r6);
    \draw[reasoning graph edge] (r5) -- (r6);
    \draw[reasoning graph edge] (r6) -- (rg);

    \node[start point]
        at (r0)
        {$\varnothing$};

    \node[reasoning point]
        at (r1)
        {1};

    \node[reasoning point]
        at (r2)
        {2};

    \node[reasoning point]
        at (r3)
        {3};

    \node[reasoning point]
        at (r4)
        {4};

    \node[reasoning point]
        at (r5)
        {5};

    \node[reasoning point]
        at (r6)
        {6};

    \node[goal point]
        at (rg)
        {$g$};
\end{scope}

\end{tikzpicture}
}
\caption{Reference trajectories from various MDPs consist of \textit{reasoning points}, which can be arranged in \textit{reasoning graphs}. By construction, every point is a prerequisite of the goal.}
    \label{fig:dags}
\end{figure}

% \begin{figure}
%     \centering
%     % \includegraphics[width=\linewidth]{imgs/dags3.png}
%     % \input{dags.tex}
%     \caption{Enter Caption}
%     \label{fig:dags3}
% \end{figure}
%%AK.final: I like the above presentation, but just that some of the terms are not defined and that causes some confusion tbh.
% \kf{Aug25: Still debating whether to to present the alg this way, or just define phi(s) as the distance required to reach g from s, which makes shortcutting/obsoletion fall out naturally.}

\textbf{Optimality equivalence.} %\pf{let's change this to instead focus on optimality equivalence.}
Training on \methodname\ rewards $\hat r$ directly can lead to suboptimal policies under outcome rewards. For example, a trajectory that reaches many intermediate reasoning points but fails to reach the goal may receive more reward than a trajectory that successfully reaches the goal via another path in the MDP.

% However, pure \methodname\ reward can result in undesired behavior -- a trajectory that reaches many intermediate reasoning points but fails to reach the goal may receive more reward than a trajectory that reaches the goal via another path. 

Our solution to this problem is to give full credit for intermediate reasoning points that have been made obsolete by reaching further reasoning points. Specifically, we construct a latent \textit{reasoning graph} of points (\Cref{fig:dags}), within which many paths (including the reference trajectory) can lead to the answer. The graph has vertices that are reasoning points, and the edge $(u, v)$ is present if point $u$ is a prerequisite of $v$. This notion of ``shortcutting'' is derived from the structure of the reasoning graph: a point is considered reached if all points that depend on this point have already been reached. In particular, reaching the goal makes \textit{all} reasoning points reached. This framework allows \methodname\ to satisfy a key property:
%%AK.final: Preston asked me if we should keep obsoletion -- I think we should, just maybe lets call it shortcutting?

\vspace{5pt}
\begin{prop}%[\METHODNAME\ Optimality Equivalence]
\label{prop:optimality-equivalence}
Assume there exists a policy that reaches the goal $g$ almost surely. Then the optimal policies under the \methodname\ and the sparse outcome rewards coincide.
\end{prop}
\vspace{-10pt}
\begin{proof}
Since $\hat{r}$ telescopes and $\phi(s_0)=0$,
\begin{equation}
\sum_{t=0}^{T-1} \hat{r}(s_t,a_t) = \phi(s_T).
\end{equation}
Because $0 \le \phi(s_T) \le n$ and we are using shortcutting, the maximum expected return $n$ is achieved iff the goal is reached almost surely. The sparse outcome reward has exactly the same set of optimal policies.
\end{proof}
\vspace{-5pt}

\section{Empirical Study on Language Model Training}

We now conduct a systematic study of language model RL, including our proposed \methodname\ method, specifically focusing on \textit{scalability with task horizon}. To isolate this axis, we focus on reasoning tasks that are constructed out of a set of sub-problems that must first be solved before the final answer can be derived.

We instantiate \methodname\ via plug-and-play modifications to an off-the-shelf policy gradient algorithm. We consider the family of policy gradients that generally follow the REINFORCE objective with token-level advantages:
\begin{equation}
    \Delta \theta = \EE_{x \sim \pi_\theta}[\nabla \log\pi_\theta(x_t | x_{<t}) \cdot A_t ],
\end{equation}
where $x_t$ is a sequence of tokens, and $A_t$ is the token-level advantage. The common baseline of GRPO is a specific case of the above formulation in which a single advantage is assigned to all tokens in a given trajectory. We base our hyperparameter choices off GRPO unless otherwise described.

In the experiments below, we construct rewards by splitting each sampled trajectory into segments. These segments are constructed via natural indicators (e.g.\ turns in multi-turn environments), or simply via equal-length chunks. We then calculate the \methodname\ measure over every segment-wise \textit{trajectory prefix}. 
Segment-level rewards are then calculated as the improvement between two trajectory prefixes, as described in \Cref{eq:reward}.

\begin{algorithm}[t]
\caption{\METHODNAME}
\label{algo:method}
\begin{algorithmic}
   \REQUIRE Policy parameters $\theta$, reference trajectory $x^{\text{ref}}$
   \WHILE{not converged}
        \STATE Sample on-policy trajectory $x = x_{0..T} \sim \pi_\theta$
        \STATE Split into segments $x = (c_1, c_2, \ldots, c_K)$ (by turn boundaries or equal-length chunks)
        \FOR{each segment $k = 1, \ldots, K$}
            \STATE Compute PPM measure over trajectory prefix: $m_k = \phi((c_1, \ldots, c_k),\, x^{\text{ref}})$
        \ENDFOR
        \STATE Compute segment-level rewards: $r_k = m_k - m_{k-1}$, with $m_0 = 0$
        \STATE Compute token-level advantages: $A_t = f(r_1, \ldots, r_K)$ (e.g.\ GRPO group normalization)
        \STATE Update parameters: $\theta \leftarrow \theta + \alpha \, \mathbb{E}_{\pi_\theta} \left[ \nabla \log \pi_\theta(x_t | x_{<t}) \cdot A_t \right]$
   \ENDWHILE
\end{algorithmic}
\end{algorithm}
% \end{figure}

% (We just modify the definition of A.)

% The main design choice is how to estimate $r(s,a)$, which as defined above, involves a state-matching comparison. In the case of language models, we assume full observability and define the historical list of tokens as the state. While it is possible to do a 1-1 comparison of exact token matching, this is clearly a weird thing to do, as many token sequences may differ trivially but result in the same semantic meaning.

% Rather, we assert that the state-matching \textit{judge model} is precisely where the well-trained prior of language models can play its strength. Rather than define a strict algorithmic comparison operator, we instead query a language model judge with both the expert trajectory, and a prefix of the student trajectory, such that:
% \begin{equation}
%     R(s_{0..t}) = LLM(s_{expert}, s_{0...t}).
% \end{equation}

%%AK.5.4: maybe we should put the contents below into an experimental section. And put the theoretical analysis in the previous section?

% \subsection{How do \methodname\ and sparse outcome reward compare as tasks increase in horizon?}\label{sec:turn-based-horizon}
\subsection{How does \methodname\ compare to baselines as tasks increase in horizon?}\label{sec:turn-based-horizon}

% \begin{figure}
%     \centering
%     \includegraphics[width=1.0\linewidth]{imgs/scaling_v2.pdf}
%     % \includegraphics[width=1.0\linewidth]{imgs/reward_variants_vs_n_clean.pdf}
%     \caption{\textbf{As tasks increase in horizon, \methodname\ results in a larger performance improvement over sparse reward.} We examine RL performance on Multi-Countdown (shown) and Matrix Manipulation (see appendix) tasks with increasing numbers of steps ($n$). As $n$ increases, the gap in training speed between \methodname\ variants increases as well. Per-turn rewards learn faster than partial outcome rewards, and both outperform a naive sparse reward. 
%     % \pf{the x-axis is perhaps misleading, these are trajectories sampled right? maybe we can come up with a better name}
%     }
%     \label{fig:multicountdown}
% \end{figure}

As one of our main motivations is to train on long-horizon tasks, we now experimentally and analytically analyze the effect of \methodname\ on synthetic environments as we increase horizon. %We consider two multi-turn domains which naturally provide a way to increase task horizon and difficulty.
% We consider two turn-based domains, which naturally provide per-turn rewards. 
Each represents a different reasoning graph structure:
\begin{itemize}
    \item \textbf{Multi-Countdown.} The standard Countdown \citep{gandhi2024stream} task requires constructing an arithmetic from four given numbers that evaluates to a given target value. In Multi-Countdown, an agent is instead tasked with solving $n$ such problems \textit{independently}, with an outcome reward awarded if all problems are solved correctly.
    \item \textbf{Matrix Manipulation.} In this environment, adapted from Reasoning Gym \cite{stojanovski2025reasoninggymreasoningenvironments}, the policy is provided with an initial matrix and a sequence of operations such as rotation, reflection, and element-wise operations, with one operation per turn. Each operation must be correctly performed in sequential order to progress to the next operation. This represents an opposite extreme, where each sub-problem is completely dependent on all preceding sub-problems.
\end{itemize}

Across these tasks, we consider three main classes of reward signal:
\begin{itemize}
    \item \textbf{Sparse outcome} rewards broadcasted over the trajectory that indicate whether \textit{all} sub-problems have been solved correctly.
    \item \textbf{\Methodname\ trajectory-level} rewards broadcasted over the trajectory that indicate the \textit{fraction} of sub-problems that have been solved correctly.
    \item \textbf{\Methodname\ turn-level} rewards calculated via \Cref{eq:reward}.
\end{itemize}

% The two domains described above are examples of the multi-turn setting where the horizon $T = n$ matches the number of reference reasoning points, and at time step $t$ we always attempt task $t$.
\textbf{Theoretical analysis.} As the settings above can be seen as simple reasoning graphs, we can analyze the expected performance of various methods from the lens of signal-to-noise ratio (SNR) of the policy gradient estimator.

\begin{theorem}[Informal]\label{thm:snr}
Let $G$ be a reasoning graph on points $V = \set{\varnothing, 1, \dots, n, g}$, where $\varnothing$ is an initial null point and $g$ is the goal. For each $v$, conditioned on reaching all prerequisites of $v$, suppose that $v$ is reached independently with probability $p$. The policy gradient estimators have signal-to-noise ratios (SNRs):
\begin{itemize}
    \item Sparse outcome reward: $\Theta(p^{n/2})$.
    \item \Methodname\ trajectory-level: $\Theta(n^{-1/2})$, under regularity conditions.
    \item \Methodname\ turn-level: letting $C$ denote the number of reached reasoning points, the SNR scales as $\Theta\big(\EE[C] / \sqrt{\Var(C)}\big)$. Under regularity conditions, this is bounded below by $\Omega(1/\log n)$.
    % \item \Methodname\ per-turn: $\Omega(1/\log n)$, under the same regularity conditions.
% \vspace{-1em}/
%     \item \Methodname\ per-turn: 
% $
% \Theta\!\left( \frac{M_n}{\sqrt{M_n + \Delta_n}} \right)
% $,
% where
%     \begin{align}
%         M_n &= \sum_{v=1}^n p^{d_v - 1}, \quad
%         \Delta_n = \sum_{v \neq w} \left( p^{|P_v \cup P_w| - 2} - p^{d_v + d_w - 2} \right).
%     \end{align}
% % $\Omega\!\left(
% %         \frac{1}{1 + \log_{1/p}(n/M_n)}
% %     \right)
% %     = \Omega\!\left(\frac{1}{\log n}\right)$,
%     $M_n$ is (up to a constant multiplier $p$) the expected number of reachable intermediate points in the graph, and $\Delta_n$ 
\end{itemize}
\end{theorem}
The full assumptions, proof, and more general version are given in Appendix~\ref{app:theory}. We can interpret Theorem~\ref{thm:snr} as follows: sparse outcome rewards become exponentially sparse with horizon, leading to exponentially small SNR in the policy gradient estimator. On the other hand, turn-level \methodname\ can yield \textit{improving SNR with task horizon}.

As we show in Appendix~\ref{app:theory}, the best case SNR of $\Theta(\sqrt N)$ for \methodname\ occurs when each reasoning point is reached independently (Multi-Countdown MDP), resulting in independent gradients across turns. The sequential Matrix Manipulation MDP, despite having high correlations between reachability events, has a turn-level SNR of $\Theta(1)$, significantly outperforming trajectory-level \methodname.

\begin{wrapfigure}[23]{r}{0.6\textwidth}
    \centering
    \vspace{-1.6em}
    \includegraphics[width=\linewidth]{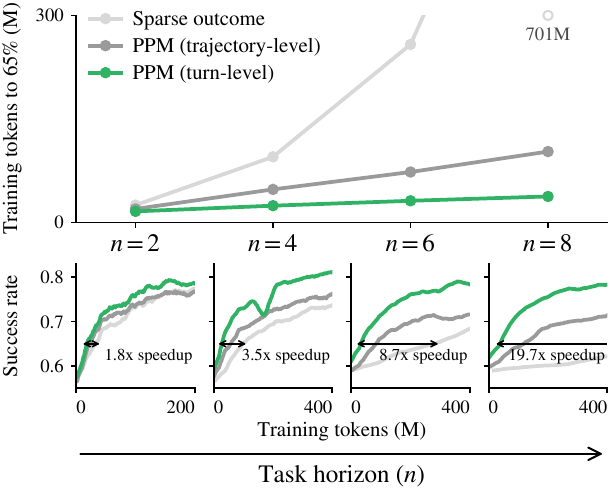}
    \vspace{-1.9em}
    \caption{As the Multi-Countdown task horizon increases, \methodname\ results in larger performance improvements over trajectory-level partial credit  and sparse outcome rewards. Theoretically and empirically, the training speed improves \textit{exponentially} over sparse outcome rewards.% \textbf{As tasks increase in horizon, \methodname\ results in a larger performance improvement over sparse outcome rewards.} We examine RL performance on Multi-Countdown (shown) and Matrix Manipulation (see appendix \pf{link this}) tasks with increasing task horizon ($n$). As $n$ increases, the gap in training speed between \methodname\ variants increases as well. Per-turn rewards learn faster than partial outcome rewards, and both outperform a naive sparse reward. 
    % \pf{the x-axis is perhaps misleading, these are trajectories sampled right? maybe we can come up with a better name}
    }
    \label{fig:multicountdown}
\end{wrapfigure}

\textbf{Empirical analysis.} We train Qwen3 models \cite{yang2025qwen3technicalreport} on the two tasks above across all three reward designs, using GRPO as the base method. Hyperparameters are kept the same across trials. We use Qwen3-1.7B for Multi-Countdown, and Qwen3-4B-Instruct for Matrix Manipulation. In both cases, the environment consists of $n$ interactions with subproblems, and the model is given a budget of 512 tokens per turn to think and provide a final verifiable answer. 

To instantiate \methodname, we utilize the symbolic structure of the tasks to identify the reference reasoning points. For Multi-Countdown, we check whether the solution to each subtask is correctly derived. For Matrix Manipulation, we check whether each matrix operation has been correctly performed. We intentionally chose easy-to-verify reasoning points to disentangle verification from the impact of \methodname\ on policy improvement -- in the following section, we examine how to verify intermediate reasoning on less structured settings.

As shown in \Cref{fig:multicountdown} (Multi-Countdown) and \Cref{fig:matrix-manipulation} (Matrix Manipulation), our main finding is that, across the board, \methodname\ consistently outperforms sparse outcome rewards. More importantly, \textbf{the margin between sparse outcome rewards and \methodname\ grows exponentially with the task horizon}. This matches our theoretical results and highlights the importance of dense partial credit assignment. 
% Sparse outcome rewards display an exponential increase in required training trajectories before an equivalent performance can be reached; \methodname\ reduces this to sublinear requirement.

% Additionally, we find that it is critical to utilize \alternativeterm{\textit{per-turn} rewards} \canonicalterm{[\methodname\ segment-level rewards]} to reach the highest efficiency. When \methodname\ is used to calculate partial progress, but these rewards are collapsed into a \confusingterm{single trajectory-level advantage}, we see that performance scales less favorably than the \alternativeterm{per-turn case} \canonicalterm{[\methodname\ segment-level]} -- although better than the \alternativeterm{naive sparse reward} \canonicalterm{[sparse outcome reward]} baseline.
% Additionally, we find that \textit{per-turn} rewards are critical to boost the efficiency of the RL algorithm. In \Cref{fig:multicountdown}, it is shown that per-turn rewards even decrease the amount of total rollouts needed, as each rollout contains $n$ independent sources of signal.

% \begin{figure}
%     \centering
    
%     \caption{Enter Caption \pf{if needed I can rerun the orange n=16 line. the n=24 results are not good because I have not tuned the hyperparameters at all, training has length 15k tokens. the plot is also a bit ugly and should perhaps have a different takeaway than the previous figure.}}
%     \label{fig:manipulate-matrix}
% \end{figure}

\subsection{Can \methodname\ be adapted to general reasoning without explicit turns?}

The environments in \Cref{sec:turn-based-horizon} benefit from (i) natural turn-based structures that enable decomposition into subtasks and (ii) ground truth rewards per subtask. However, real-world long-horizon tasks typically do not admit this structure. Instead, many tasks implicitly demand extended reasoning \cite{research2026composer}, with only the final answer directly verifiable. We investigate whether \methodname\ can yield stable improvements without environment-provided per-segment rewards.

We use GSM-Infinite \cite{zhou2025gsminfinitellmsbehaveinfinitely} as an environment for long-horizon reasoning. In this environment, math problems are generated synthetically from graphs, in which nodes represent quantities and edges encode relationships between them. The task horizon is parameterized by the length of the shortest path between the given information and the desired quantity. GSM-Infinite also represents an intermediate reasoning graph between the two extremes of Multi-Countdown and Matrix Manipulation, and serves as a controllable proxy for more general reasoning tasks.

We compare against a representative set of on-policy RL baselines, focusing on methods that use reference trajectories:
% We compare against a representative set of on-policy RL baselines, focusing on methods that assign per-segment rewards or per-token advantages via external judge models or by utilizing a reference solution:
\begin{itemize}
    \item \textbf{GRPO} \citep{shao2024deepseekmath} with sparse outcome rewards.
    \item \textbf{Verifree} \citep{zhou2025reinforcing}, a representative likelihood-based method that rewards reasoning based on the per-token likelihood of the solution conditioned on that reasoning.
    \item \textbf{On-policy self distillation (OPSD)} \citep{zhao2026self, hubotter2026reinforcement, shenfeld2026self}, a representative context distillation method which places the solution in the context of the model to form a teacher distribution, whose logits are then distilled into the current policy.
    % \item \textbf{POPE} \citep{qu2026pope}, a methodology which trains by placing partial prefixes of reference trajectories in context during training.
    \item A \textbf{process reward} baseline, whereby we utilize an off-the-shelf Gemini model to reward correctness by categorizing whether each segment is logically consistent \citep{miao2023selfcheck, khalifa2025process}.
    % \pf{maybe highlight that the only thing different between our method and process rewards is the prompt and chunk selection.}
    % \item \pf{opsd, sft, etc -- also add more details in the appendix}
\end{itemize}

\begin{wraptable}[12]{o}{0.5\textwidth}
    \centering
    \vspace{-1.9em}
    \caption{GSM-Infinite success rates at 400M training tokens.}
    \vspace{0.2em}
    \setlength{\tabcolsep}{3pt}
\begin{tabular}{@{}lccc@{}}
    \toprule
    Method & $n=8$ & $n=16$ & $n=24$ \\
    \midrule
    % Sparse outcome reward & 0.632 & 0.070 & 0.090 \\
    % Verifree & \textbf{0.794} & 0.084 & 0.084 \\
    % OPSD & 0.237 & 0.068 & 0.051 \\
    % Process rewards & 0.268 & 0.071 & 0.070 \\
    % PPM (trajectory-level) & 0.690 & 0.327 & 0.089 \\
    % PPM (segment-level) & 0.588 & \textbf{0.384} & \textbf{0.141} \\
Base model & 0.134 & 0.056 & 0.056 \\
Sparse outcome reward & 0.684 & 0.084 & 0.091 \\
Verifree & \textbf{0.828} & 0.126 & 0.098 \\
OPSD & 0.283 & 0.102 & 0.078 \\
Process reward & 0.329 & 0.075 & 0.080 \\
PPM (trajectory-level) & 0.758 & 0.334 & 0.089 \\
PPM (segment-level) & 0.652 & \textbf{0.415} & \textbf{0.191} \\
    \bottomrule
    \end{tabular}
    \label{tab:gsminf}
  \end{wraptable}

All baseline approaches are trained using Qwen3-1.7B as the base policy. We use GRPO with similar configurations across the board, as described in \Cref{tab:experiment-details-rl}.

We then implement \methodname, considering both the per-trajectory and per-segment variants. We utilize a text rendering of the ground-truth solution graph as the set of intermediate reasoning points. An off-the-shelf Gemini judge, which we evaluate in Appendix~\ref{sec:instantiating-judge}, is then used to measure how many of these points a given trajectory reaches. In the segment-level variant, we split each sampled trajectory into four uniformly sliced chunks, and pass each chunk independently (including all previous chunks) into the judge model.

We measure performance on GSM-Infinite tasks generated from increasingly long-horizon graphs. As shown in \Cref{tab:gsminf}, while many baselines are comparable in the regime of easy problems ($n=8$), \textbf{\methodname\ enables larger performance gains with task horizon}. At the hardest difficulty level ($n=24$), there is a clear benefit from assigning rewards at a segment level over trajectory level.
% We instantiate \methodname\ in three variants -- in the first (exact order), reward is only assigned when trajectories match the reasoning points in exactly the same order that the reference trajectory does. In the second (no graph), reasoning points can be reached in any order, but we do not utilize the reasoning graph to allow for receiving credit for shortcuts. Finally, the main setting utilizes full \methodname, including the reasoning graph.

% (figure on prompt climbing -- show a naive prompt, old prompt, new prompt)

\subsection{How should one instantiate the \methodname\ judge $\phi$?}\label{sec:q-judge}

% Language models are powerful tools for tasks such as grading, but used naively, they can produce high-variance numerical scores that are difficult to train on. We provide concrete guidance on how to evaluate the base model as a judge \textit{without the need for training}.

% Language models are powerful tools for tasks such as grading, but instantiating an effective \methodname\ judge . Part of the motivation for the state-matching formalization is to provide concrete guidance on how to evaluate the effectiveness of a judge prompt \textit{without the need for training.}

% \Methodname\ depends on a judge model that can effectively judge intermediate states $s_t$, which correspond to trajectory prefixes. The resulting \methodname\ measure $\phi(s_t)$ must be non-decreasing in $t$.

In synthetic environments such as Multi-Countdown or Matrix Manipulation, one
can instantiate a ground-truth $\phi$ by matching substrings directly. However,
in open-ended reasoning tasks such as GSM-Infinite, there are no ground-truth
labels indicating whether trajectory segments match reasoning points without
annotations from a human or a stronger model. We provide concrete guidance on
how to instantiate $\phi$ as the \methodname\ measure \textit{without  annotations}.

We make the following observation: conditioned on reaching $s_t$, the Monte
Carlo success rate $\bar r(s_t)$ of subsequent reasoning should grow \textit{linearly} with the \methodname\ measure $\phi(s_t)$. By definition, conditioning on
the reasoning points in $s_t$ is equivalent to conditioning on any trajectory prefix that
reaches them. We can therefore estimate $\bar r(s_t)$ by placing \textit{prefixes of the reference trajectory}
in the prompt and sampling continuations from the policy.

In our experiments, we interpolate between the following two settings: (i) on
hard problems where the policy cannot successfully solve the task, both the
Monte Carlo return and the mean terminal \methodname\ measure should be low; and
(ii) after providing the full reference trajectory in context, asking the
policy to reproduce the solution should almost always succeed, and both
quantities should be near one. We sweep between these settings by providing
reference prefixes of increasing length and compare the Monte Carlo return at
each prefix with the corresponding mean terminal \methodname\ measure
(Figure~\ref{fig:gemini-prefix}). Our \methodname\ judge, which we instantiate as an off-the-shelf Gemini model with a specific prompt that maximizes the above criterion, yields a Kendall rank
correlation of $0.856$, compared with $0.768$ for a standard rubric prompt (see Appendix~\ref{app:prompts}).

% One heuristic is as follows: we measure the Kendall rank correlation coefficient between the \textit{length of an oracle solution prefix} provided in context and the judge rubric score (Figure~\ref{fig:gemini-prefix}). Our \methodname\ judge yields a score of 0.856, compared to a naive rubric prompt's score of 0.768. 

%A rough approximation, which we posit approximately holds at long horizons, is that linearly increasing the \textit{length of an oracle solution prefix} provided in context should correspond to linearly improving judge scores.

% Specifically, our strategy relies on a particular structure in state-matching rewards -- given an expert trajectory, and a \textit{prefix} of that trajectory, the total state-matching reward should be proportional to the prefix ratio. This is a rough heuristic, as tokens don't translate directly into semantic states, however at long horizons this property should roughly hold.

We can additionally couple the above criterion with a small set of hand-designed validation trajectories, where manual inspection reveals where the trajectory deviates from empirical returns. A strong judge prompt should correctly show this deviation in its \methodname\ measure. %The full prompts are shown in Appendix~\ref{app:prompts}.

% \begin{figure}
%     \centering
% \includegraphics[width=\linewidth]{imgs/gsminf.png}
%     \caption{GSM-Infinite \Methodname\ improves over sparse }
%     \label{fig:gsminf}
% \end{figure}

% \begin{table}[t]
%     \centering
%     \setlength{\tabcolsep}{5pt}
%     \caption{GSM-Infinite success rates at 240M training tokens.}
%     \begin{tabular}{@{}lcccccc@{}}
%     \toprule
%     & \multicolumn{2}{c}{$n=8$}
%     & \multicolumn{2}{c}{$n=16$}
%     & \multicolumn{2}{c}{$n=24$} \\
%     \cmidrule(lr){2-3} \cmidrule(lr){4-5} \cmidrule(l){6-7}
%     Method
%     & p@1 & p@8
%     & p@1 & p@8
%     & p@1 & p@8 \\
%     \midrule
%     Sparse outcome reward & 0.404 & 0.867 & 0.047 & 0.191 & 0.082 & \textbf{0.309} \\
%     OPSD & 0.309 & 0.750 & 0.048 & 0.266 & 0.048 & 0.242 \\
%     Verifree & \textbf{0.680} & \textbf{0.933} & 0.060 & 0.235 & 0.059 & 0.243 \\
%     Process reward & 0.174 & 0.594 & 0.053 & 0.224 & 0.055 & 0.226 \\
%     Progressive point matching (outcome) & 0.417 & 0.817 & 0.130 & \textbf{0.521} & 0.040 & 0.208 \\
%     Progressive point matching & 0.368 & 0.895 & \textbf{0.154} & 0.480 & \textbf{0.228} & 0.292 \\
%     \bottomrule
%     \end{tabular}
%     \label{tab:gsminf}
%   \end{table}

\subsection{Does \methodname\ support extrapolation to longer inference budgets?}\label{sec:extrapolation}

% \pf{The table does not really answer this question. Perhaps we can then focus on the POPE results for both this and the below subsection? \Cref{tab:polaris-dataset} is important to show that our method is stable on a standard dataset, but is separate from the main point of the paper.}

\begin{table}[t]
  \centering
  \setlength{\tabcolsep}{5pt}
  \caption{Qwen3-4B trained on Polaris dataset at 8K token budget and evaluated at 16K.}
  \begin{tabular}{@{}lcccccccc@{}}
  \toprule
  & \multicolumn{2}{c}{Train}
  & \multicolumn{2}{c}{Test}
  & \multicolumn{2}{c}{AIME25}
  & \multicolumn{2}{c}{HMMT25} \\
  \cmidrule(lr){2-3} \cmidrule(lr){4-5} \cmidrule(lr){6-7} \cmidrule(l){8-9}
  Method
  & p@1 & p@8
  & p@1 & p@8
  & p@1 & p@8
  & p@1 & p@8 \\
  \midrule
  Base model & 0.156 & 0.292 & 0.416 & 0.745 & 0.510 & 0.740 & 0.346 & 0.569 \\
  % SFT & 0.189 & 0.581 & 0.218 & 0.585 & 0.083 & 0.195 & \\
  SFT & 0.210 & \textbf{0.609} & 0.264 & 0.651 & 0.098 & 0.246 & 0.060 & 0.175 \\
  GRPO & 0.203 & 0.401 & 0.441 & 0.806 & \textbf{0.560} & 0.727 & 0.354 & 0.556 \\
  % GRPO + POPE (needs rerun) & \textbf{0.313} & \textbf{0.549} & \textbf{0.467} & 0.820 & 0.535 & 0.713 & 0.356 & 0.583 \\
  Verifree & 0.238 & 0.412 & \textbf{0.461} & 0.799 & \textbf{0.569} & \textbf{0.757} & 0.358 & 0.578 \\
  Process rewards & 0.195 & 0.396 & 0.451 & \textbf{0.832} & 0.542 & \textbf{0.756} & 0.367 & 0.537 \\
  OPSD & 0.215 & 0.433 & 0.436 & 0.815 & 0.529 & 0.729 & 0.358 & \textbf{0.597} \\
  % PPM (outcome) & 0.253 & 0.440 & 0.443 & 0.802 & 0.556 & 0.709 & 0.365 & 0.583 \\
  % PPM (4 chunks, no graph) & 0.255 & 0.443 & 0.461 & \textbf{0.880} & 0.512 & 0.744 & \textbf{0.373} & 0.561 \\
  % PPM (4 chunks, sparse graph) & 0.293 & 0.495 & 0.416 & 0.785 & 0.535 & 0.746 & \textbf{0.371} & 0.588 \\
  % PPM (4 chunks, medium graph) & 0.275 & 0.502 & \textbf{0.477} & \textbf{0.871} & 0.506 & 0.727 & \textbf{0.371} & 0.581 \\
  % PPM (4 chunks, dense graph) & 0.258 & 0.452 & 0.459 & 0.807 & 0.535 & \textbf{0.748} & \textbf{0.379} & \textbf{0.598} \\
  \Methodname & \textbf{0.258} & 0.452 & \textbf{0.459} & 0.807 & 0.535 & \textbf{0.748} & \textbf{0.379} & \textbf{0.598} \\

  % PPM + POPE (outcome) \\
  % PPM + POPE (4 chunks, no graph) & 0.261 & 0.464 & 0.426 & 0.804 & 0.550 & 0.738 & \textbf{0.379} & 0.584 \\
  % PPM + POPE (4 chunks, sparse graph) & 0.277 & 0.498 & 0.436 & 0.783 & 0.546 & 0.745 & 0.338 & 0.540 \\
  % PPM + POPE (4 chunks, medium graph) & 0.245 & 0.472 & 0.422 & 0.808 & 0.542 & 0.730 & 0.360 & \textbf{0.628} \\
  % PPM + POPE (4 chunks, dense graph) & 0.284 & 0.501 & 0.424 & 0.792 & \textbf{0.573} & 0.730 & 0.369 & 0.548 \\
  % Base model & -- & -- & 0.416 & 0.745 & 0.510 & 0.740 & 0.346 & 0.569 \\
  % SFT & -- & -- & -- & -- & -- & -- & -- & -- \\
  % GRPO & 0.203 & 0.407 & 0.443 & 0.809 & \textbf{0.560} & \textbf{0.757} & 0.355 & 0.559 \\
  % GRPO + POPE & 0.343 & 0.660 & 0.467 & 0.820 & 0.554 & 0.765 & 0.367 & 0.573 \\
  % Verifree & 0.238 & 0.425 & 0.457 & 0.819 & \textbf{0.559} & 0.745 & 0.352 & 0.562 \\
  % Process rewards & 0.214 & 0.396 & 0.451 & 0.814 & 0.529 & 0.741 & 0.364 & 0.536 \\
  % OPSD & 0.215 & 0.433 & 0.436 & 0.815 & 0.529 & 0.729 & 0.358 & \textbf{0.597} \\
  % Prog.\ point matching (not chunk) \\
  % Prog.\ point matching (no graph) & 0.255 & 0.443 & 0.461 & \textbf{0.880} & 0.526 & 0.744 & 0.373 & 0.561 \\
  % Prog.\ point matching (sparse graph) & \textbf{0.293} & \textbf{0.495} & 0.453 & \textbf{0.871} & 0.521 & 0.736 & 0.371 & \textbf{0.593} \\
  % Prog.\ point matching (medium graph) & 0.275 & \textbf{0.502} & \textbf{0.477} & \textbf{0.871} & 0.514 & 0.742 & 0.371 & 0.581 \\
  % Prog.\ point matching (dense graph) & 0.258 & 0.452 & 0.455 & 0.784 & 0.548 & 0.744 & \textbf{0.385} & \textbf{0.597} \\
  \bottomrule
  \end{tabular}
  \label{tab:polaris-dataset}
  \end{table}

% \kf{I guess this entire section needs reworking}

We now investigate whether \methodname\ can be applied to real math reasoning, and we specifically focus on setting in which we desire test-time extrapolation of the learned policies. We use the Polaris dataset \cite{Polaris2025} of math problems with verifiable numerical answers and oracle reference solutions generated by gemini-3-flash \cite{gemini3flash2025}.

In addition to the baselines utilized in the previous section, we also compare against SFT, which maximizes the likelihood of reference trajectories.

As shown in \Cref{tab:polaris-dataset}, \methodname\ reliably extracts signal from hard-to-solve math problems. In particular, policies trained via \methodname\ extrapolate cleanly when they are evaluated at 16K tokens, even though the policies themselves are trained at much lower budgets. As an on-policy RL method, \methodname\ does not suffer from catastrophic collapse, and maintains consistent entropy to achieve considerable improvement at pass@8.

\subsection{Can \methodname\ learn from near-impossible task sets?}\label{sec:near-impossible-tasks}

\begin{wrapfigure}[16]{r}{0.5\textwidth}
    \centering
    \vspace{-2em}
    \includegraphics[width=\linewidth]{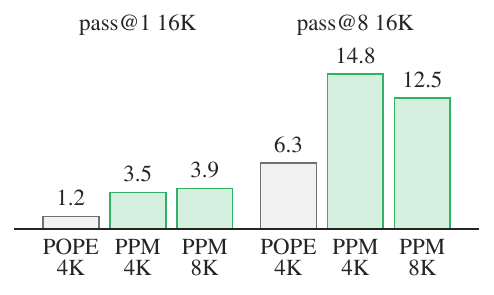}
    \vspace{-2.2em}
    \caption{We train on the POPE-hard dataset \cite{qu2026pope} at length 4K or 8K tokens and evaluate the learned policy at 16K tokens. \Methodname\ outperforms the best baseline (POPE). Moreover, the test-time performance is \textit{comparable or better} when training on shorter sequences (4K).}
    \label{fig:pope-dataset}
\end{wrapfigure}

We now investigate whether \methodname\ can provide partial credit signal to learn from, even in the case where ground-truth reward signal is near zero. We train on a subset of the POPE-hard dataset \cite{qu2026pope} with verifiable numerical answers and with oracle solutions generated by gemini-2.5-pro. Prompts are filtered such that when evaluated at an 8K token budget at 16 trials, the base Qwen3-4B-Instruct model never produces a successful trajectory (see Appendix~\ref{app:environments}. As shown in \Cref{tab:pope-dataset}, the aggregate base model pass rate at 8K tokens is 0.004, and pass rate at 4K tokens is zero.

The best-performing baseline in this setting is the Privileged On-Policy Exploration (POPE) \cite{qu2026pope} algorithm, which modifies tasks by conditioning on prefixes of reference trajectories and instructing the model to continue solving the problem without restating any part of the reference trajectory. In our experiments (\Cref{fig:pope-dataset}, \Cref{tab:pope-dataset}), POPE improves over GRPO and other baselines, but PPM displays the strongest performance overall.

\textbf{Pitfall of greedy optimization.} As shown in \Cref{fig:pope-dataset}, the \methodname\ policy trained at 4K tokens surprisingly outperforms the policy trained at 8K tokens, when evaluated at a test-time budget of 16K. We find that this is due to a collapse in output length -- the policy trained at 8K length tends to greedily guess at the answer and end its reasoning process early. We hypothesize this is due to solve the task within the training length budget, biasing the policy toward greedy behavior. On the other hand, the policy trained at 4K length has no such pressure as its success rate is entirely zero, and thus all signal comes from reaching \textit{intermediate} reasoning points. This finding highlights a surprising phenomenon where training at lower sequence lengths may be \textit{beneficial}, and is a promising direction for future focus.

\section{Discussion}\label{sec:conclusion}

The standard approach to reinforcement learning for LLMs requires exponentially many samples with the task horizon, manifesting in slow learning. To overcome this limitation, we introduce \textbf{\METHODNAME}, a framework that bridges RL and imitation learning by guiding exploration over the reasoning graph induced by a \textit{single} reference trajectory per prompt. By rewarding the size of the set of reached intermediate reasoning points, \methodname\ can converge to the same set of optimal policies as with outcome rewards, learn diverse strategies not present in reference trajectories, and learn from hard math problems where existing RL approaches fail.

\Methodname\ enables strong performance improvements over prior approaches when evaluated at longer inference token budgets, as shown for math reasoning in both the standard (\Cref{sec:extrapolation}) and near-impossible (\Cref{sec:near-impossible-tasks}) data regimes. However, as we briefly discuss in \Cref{sec:near-impossible-tasks}, we hypothesize that there are three regimes based on the training token budgets. Specifically, \textbf{(i)} at very short sequences it is impossible for the policy to complete the task within the budget, hence the optimal behavior is to maximize partial progress; \textbf{(ii)} at medium-length sequences we experience greedy  behavior due to length pressure; \textbf{(iii)} at long budgets there is no length pressure, hence the optimal behavior is to maximize task completion (\Cref{prop:optimality-equivalence}). Our synthetic environments (Multi-Countdown, Matrix Manipulation, GSM-Infinite) generally operate within regime \textbf{(iii)}, while our math reasoning tasks operate in regime \textbf{(i)} or \textbf{(ii)}.

\textbf{Limitations.} We expect that, for each training dataset, there exists a \emph{compute-optimal} training sequence length that maximizes evaluation performance at larger test-time token budgets. Eliciting extrapolation may therefore depend on training at an appropriate sequence length. For highly heterogeneous datasets, this could be achieved either by subsampling problems of suitable difficulty, as in \citet{setlur2025e3learningexploreenables}, or by using our approach to reward partial progress in regimes \textbf{(i)} and \textbf{(iii)}. Standard outcome-reward RL cannot effectively train in regime \textbf{(i)}, because the policy never receives a nonzero reward signal; it must filter data as in prior work. In regime \textbf{(ii)}, however, partial-progress rewards may also be less effective, as both partial-progress and pure outcome-reward training may collapse to greedy behavior and thereby reduce exploration. 

While this work covers such phenomena, empirical results are limited by scale, and would benefit from larger models, larger datasets to properly test generalization and scalability of reference trajectories, and longer-horizon tasks such as extended multi-turn interactions or agentic coding, in which the test-time token budget may be multiple orders of magnitude longer than the training token budget. We encourage future work to study the interplay between partial rewards and response lengths. 
% However, understanding these empirical phenomena at scale will require
% \begin{itemize}
    % \item 
%     Larger models -- we use models with at most 4B parameters. \item Larger datasets -- we have a limited ability to generate sufficiently high-quality reference trajectories at scale. Large open datasets of extremely challenging math problems do not exist to our knowledge, forcing us to train for up to 200 epochs on the POPE-hard \cite{qu2026pope} dataset.
%     \item Longer-horizon tasks -- extended multi-turn interactions or agentic coding, in which the test-time token budget may be multiple orders of magnitude longer than the training token budget.
% \end{itemize}
% We encourage future work to study the interplay between partial rewards and response lengths. 

% Additionally, our framework introduces the concepts of \textit{reasoning points} and the \textit{\methodname\ measure}. One may evaluate them via \Cref{sec:q-judge}, and we instantiate reasonable baselines in each case as an off-the-shelf Gemini model with prompts in Appendix~\ref{app:prompts}. However, this prompt design procedure is somewhat domain-specific and may require multiple iterations of manual inspection. An end-to-end procedure involving finetuning an LLM may be desirable.
% Use bigger model and data \pf{this dataset does not exist, would need to evaluate very long horizon and also have more prompts as POPE is only 200 prompts}

\textbf{Future work.} Practitioners should more rigorously characterize the three regimes described above and identify the interplay between length budget for RL, the prompt (data) mixture, and partial rewards. Another exciting direction -- assuming access to sufficiently long-horizon tasks -- is establishing scaling laws for partial credit assignment methods, with the objective of maximizing success rates at much larger test-time token budgets. Additionally, our framework introduces \textit{reasoning points} and a method for evaluating them via \Cref{sec:q-judge}; future work may explore training an LLM to improve reasoning point generation rather than following a manual and domain-specific procedure. Finally, we highlight that our approach relies \textit{solely} on reference trajectories, from which we implicitly determine goal points. Therefore, \methodname\ and other methods improving on imitation learning may enable training in non-verifiable domains.

\newpage
\appendix

\section*{Acknowledgements}

We thank the TPU research cloud (TRC) program for the compute resources that powered this work, and Google Cloud for the Gemini API.

\bibliography{neurips_2025}

@inproceedings{Pomerleau1988ALVINN,
  title = {{ALVINN}: An Autonomous Land Vehicle in a Neural Network},
  author = {Pomerleau, Dean A.},
  booktitle = {Advances in Neural Information Processing Systems},
  pages = {305--313},
  year = {1988},
  url = {https://papers.nips.cc/paper/1988/hash/812b4ba287f5ee0bc9d43bbf5bbe87fb-Abstract.html}
}

@misc{stojanovski2025reasoninggymreasoningenvironments,
      title={Reasoning Gym: Reasoning Environments for Reinforcement Learning with Verifiable Rewards}, 
      author={Zafir Stojanovski and Oliver Stanley and Joe Sharratt and Richard Jones and Abdulhakeem Adefioye and Jean Kaddour and Andreas Köpf},
      year={2025},
      eprint={2505.24760},
      archivePrefix={arXiv},
      primaryClass={cs.LG},
      url={https://arxiv.org/abs/2505.24760}, 
}

@article{research2026composer,
  title={Composer 2 Technical Report},
  author={Research, Cursor and Chan, Aaron and Shalaby, Ahmed and Wettig, Alexander and Sanger, Aman and Zhai, Andrew and Ajay, Anurag and Nair, Ashvin and Snell, Charlie and Lu, Chen and others},
  journal={arXiv preprint arXiv:2603.24477},
  year={2026}
}

@article{gandhi2024stream,
  title={Stream of search (sos): Learning to search in language},
  author={Gandhi, Kanishk and Lee, Denise and Grand, Gabriel and Liu, Muxin and Cheng, Winson and Sharma, Archit and Goodman, Noah D},
  journal={arXiv preprint arXiv:2404.03683},
  year={2024}
}

@misc{zhou2025gsminfinitellmsbehaveinfinitely,
      title={GSM-Infinite: How Do Your LLMs Behave over Infinitely Increasing Context Length and Reasoning Complexity?}, 
      author={Yang Zhou and Hongyi Liu and Zhuoming Chen and Yuandong Tian and Beidi Chen},
      year={2025},
      eprint={2502.05252},
      archivePrefix={arXiv},
      primaryClass={cs.CL},
      url={https://arxiv.org/abs/2502.05252}, 
}

@inproceedings{Ross2011Reduction,
  title = {A Reduction of Imitation Learning and Structured Prediction to No-Regret Online Learning},
  author = {Ross, Stephane and Gordon, Geoffrey and Bagnell, Drew},
  booktitle = {Proceedings of the Fourteenth International Conference on Artificial Intelligence and Statistics},
  pages = {627--635},
  year = {2011},
  editor = {Gordon, Geoffrey and Dunson, David and Dud{\'i}k, Miroslav},
  volume = {15},
  series = {Proceedings of Machine Learning Research},
  address = {Fort Lauderdale, FL, USA},
  month = {11--13 Apr},
  publisher = {PMLR},
  url = {https://proceedings.mlr.press/v15/ross11a.html}
}

@misc{luo2024improvemathematicalreasoninglanguage,
      title={Improve Mathematical Reasoning in Language Models by Automated Process Supervision}, 
      author={Liangchen Luo and Yinxiao Liu and Rosanne Liu and Samrat Phatale and Meiqi Guo and Harsh Lara and Yunxuan Li and Lei Shu and Yun Zhu and Lei Meng and Jiao Sun and Abhinav Rastogi},
      year={2024},
      eprint={2406.06592},
      archivePrefix={arXiv},
      primaryClass={cs.CL},
      url={https://arxiv.org/abs/2406.06592}, 
}

@misc{andrychowicz2018hindsightexperiencereplay,
      title={Hindsight Experience Replay}, 
      author={Marcin Andrychowicz and Filip Wolski and Alex Ray and Jonas Schneider and Rachel Fong and Peter Welinder and Bob McGrew and Josh Tobin and Pieter Abbeel and Wojciech Zaremba},
      year={2018},
      eprint={1707.01495},
      archivePrefix={arXiv},
      primaryClass={cs.LG},
      url={https://arxiv.org/abs/1707.01495}, 
}

@inproceedings{ng1999policy,
author = {Ng, Andrew Y. and Harada, Daishi and Russell, Stuart J.},
title = {Policy Invariance Under Reward Transformations: Theory and Application to Reward Shaping},
year = {1999},
isbn = {1558606122},
publisher = {Morgan Kaufmann Publishers Inc.},
address = {San Francisco, CA, USA},
booktitle = {Proceedings of the Sixteenth International Conference on Machine Learning},
pages = {278–287},
numpages = {10},
series = {ICML '99}
}

@misc{garg2022iqlearninversesoftqlearning,
      title={IQ-Learn: Inverse soft-Q Learning for Imitation}, 
      author={Divyansh Garg and Shuvam Chakraborty and Chris Cundy and Jiaming Song and Matthieu Geist and Stefano Ermon},
      year={2022},
      eprint={2106.12142},
      archivePrefix={arXiv},
      primaryClass={cs.LG},
      url={https://arxiv.org/abs/2106.12142}, 
}

@misc{lightman2023letsverifystepstep,
      title={Let's Verify Step by Step}, 
      author={Hunter Lightman and Vineet Kosaraju and Yura Burda and Harri Edwards and Bowen Baker and Teddy Lee and Jan Leike and John Schulman and Ilya Sutskever and Karl Cobbe},
      year={2023},
      eprint={2305.20050},
      archivePrefix={arXiv},
      primaryClass={cs.LG},
      url={https://arxiv.org/abs/2305.20050}, 
}

@misc{yang2025qwen3technicalreport,
      title={Qwen3 Technical Report}, 
      author={An Yang and Anfeng Li and Baosong Yang and Beichen Zhang and Binyuan Hui and Bo Zheng and Bowen Yu and Chang Gao and Chengen Huang and Chenxu Lv and Chujie Zheng and Dayiheng Liu and Fan Zhou and Fei Huang and Feng Hu and Hao Ge and Haoran Wei and Huan Lin and Jialong Tang and Jian Yang and Jianhong Tu and Jianwei Zhang and Jianxin Yang and Jiaxi Yang and Jing Zhou and Jingren Zhou and Junyang Lin and Kai Dang and Keqin Bao and Kexin Yang and Le Yu and Lianghao Deng and Mei Li and Mingfeng Xue and Mingze Li and Pei Zhang and Peng Wang and Qin Zhu and Rui Men and Ruize Gao and Shixuan Liu and Shuang Luo and Tianhao Li and Tianyi Tang and Wenbiao Yin and Xingzhang Ren and Xinyu Wang and Xinyu Zhang and Xuancheng Ren and Yang Fan and Yang Su and Yichang Zhang and Yinger Zhang and Yu Wan and Yuqiong Liu and Zekun Wang and Zeyu Cui and Zhenru Zhang and Zhipeng Zhou and Zihan Qiu},
      year={2025},
      eprint={2505.09388},
      archivePrefix={arXiv},
      primaryClass={cs.CL},
      url={https://arxiv.org/abs/2505.09388}, 
}

@misc{wang2024mathshepherdverifyreinforcellms,
      title={Math-Shepherd: Verify and Reinforce LLMs Step-by-step without Human Annotations}, 
      author={Peiyi Wang and Lei Li and Zhihong Shao and R. X. Xu and Damai Dai and Yifei Li and Deli Chen and Y. Wu and Zhifang Sui},
      year={2024},
      eprint={2312.08935},
      archivePrefix={arXiv},
      primaryClass={cs.AI},
      url={https://arxiv.org/abs/2312.08935}, 
}

@inproceedings{zhang2025lessons,
  title={The lessons of developing process reward models in mathematical reasoning},
  author={Zhang, Zhenru and Zheng, Chujie and Wu, Yangzhen and Zhang, Beichen and Lin, Runji and Yu, Bowen and Liu, Dayiheng and Zhou, Jingren and Lin, Junyang},
  booktitle={Findings of the Association for Computational Linguistics: ACL 2025},
  pages={10495--10516},
  year={2025}
}

@misc{cui2025processreinforcementimplicitrewards,
      title={Process Reinforcement through Implicit Rewards}, 
      author={Ganqu Cui and Lifan Yuan and Zefan Wang and Hanbin Wang and Yuchen Zhang and Jiacheng Chen and Wendi Li and Bingxiang He and Yuchen Fan and Tianyu Yu and Qixin Xu and Weize Chen and Jiarui Yuan and Huayu Chen and Kaiyan Zhang and Xingtai Lv and Shuo Wang and Yuan Yao and Xu Han and Hao Peng and Yu Cheng and Zhiyuan Liu and Maosong Sun and Bowen Zhou and Ning Ding},
      year={2025},
      eprint={2502.01456},
      archivePrefix={arXiv},
      primaryClass={cs.LG},
      url={https://arxiv.org/abs/2502.01456}, 
}

@techreport{gemini3flash2025,
  author      = {{Google DeepMind}},
  title       = {Gemini 3 Flash Model Card},
  institution = {Google},
  year        = {2025},
  url         = {https://storage.googleapis.com/deepmind-media/Model-Cards/Gemini-3-Flash-Model-Card.pdf}
}

@techreport{gemini25flashlite2025,
  author      = {{Google DeepMind}},
  title       = {Gemini 2.5 Flash-Lite Model Card},
  institution = {Google},
  year        = {2025},
  url         = {https://storage.googleapis.com/deepmind-media/Model-Cards/Gemini-2-5-Flash-Lite-Model-Card.pdf}
}

@misc{Polaris2025,
    title = {POLARIS: A Post-Training Recipe for Scaling Reinforcement Learning on Advanced Reasoning Models},
    url = {https://hkunlp.github.io/blog/2025/Polaris},
    author = {An, Chenxin and Xie, Zhihui and Li, Xiaonan and Li, Lei and Zhang, Jun and Gong, Shansan and Zhong, Ming and Xu, Jingjing and Qiu, Xipeng and Wang, Mingxuan and Kong, Lingpeng},
    year = {2025}
}

@misc{yu2025dapoopensourcellmreinforcement,
      title={DAPO: An Open-Source LLM Reinforcement Learning System at Scale}, 
      author={Qiying Yu and Zheng Zhang and Ruofei Zhu and Yufeng Yuan and Xiaochen Zuo and Yu Yue and Weinan Dai and Tiantian Fan and Gaohong Liu and Lingjun Liu and Xin Liu and Haibin Lin and Zhiqi Lin and Bole Ma and Guangming Sheng and Yuxuan Tong and Chi Zhang and Mofan Zhang and Wang Zhang and Hang Zhu and Jinhua Zhu and Jiaze Chen and Jiangjie Chen and Chengyi Wang and Hongli Yu and Yuxuan Song and Xiangpeng Wei and Hao Zhou and Jingjing Liu and Wei-Ying Ma and Ya-Qin Zhang and Lin Yan and Mu Qiao and Yonghui Wu and Mingxuan Wang},
      year={2025},
      eprint={2503.14476},
      archivePrefix={arXiv},
      primaryClass={cs.LG},
      url={https://arxiv.org/abs/2503.14476}, 
}

@misc{setlur2025e3learningexploreenables,
      title={e3: Learning to Explore Enables Extrapolation of Test-Time Compute for LLMs}, 
      author={Amrith Setlur and Matthew Y. R. Yang and Charlie Snell and Jeremy Greer and Ian Wu and Virginia Smith and Max Simchowitz and Aviral Kumar},
      year={2025},
      eprint={2506.09026},
      archivePrefix={arXiv},
      primaryClass={cs.LG},
      url={https://arxiv.org/abs/2506.09026}, 
}

@misc{Cheng2026isocompute,
author={Cheng*, Zhoujun and Xie*, Yutao and Qu*, Yuxiao and Setlur*, Amrith and Hao, Shibo and Pimpalkhute, Varad and Liang, Tongtong and Yao, Feng and Liu, Hector and Xing, Eric and Smith, Virginia and Salakhutdinov, Ruslan and Hu, Zhiting and Killian, Taylor and Kumar, Aviral},
title={IsoCompute Playbook: Optimally Scaling Sampling Compute for RL Training of LLMs},
url={https://compute-optimal-rl-llm-scaling.github.io/},
year={2026},
}

@misc{Ross2014Reinforcement,
  title = {Reinforcement and Imitation Learning via Interactive No-Regret Learning},
  author = {Ross, Stephane and Bagnell, J. Andrew},
  year = {2014},
  eprint = {1406.5979},
  archivePrefix = {arXiv},
  primaryClass = {cs.LG},
  doi = {10.48550/arXiv.1406.5979},
  url = {https://arxiv.org/abs/1406.5979}
}

@inproceedings{Sun2017DeeplyAggreVaTeD,
  title = {Deeply {AggreVaTeD}: Differentiable Imitation Learning for Sequential Prediction},
  author = {Sun, Wen and Venkatraman, Arun and Gordon, Geoffrey J. and Boots, Byron and Bagnell, J. Andrew},
  booktitle = {Proceedings of the 34th International Conference on Machine Learning},
  pages = {3309--3318},
  year = {2017},
  editor = {Precup, Doina and Teh, Yee Whye},
  volume = {70},
  series = {Proceedings of Machine Learning Research},
  month = {06--11 Aug},
  publisher = {PMLR},
  url = {https://proceedings.mlr.press/v70/sun17d.html}
}

@inproceedings{Ho2016Generative,
  title = {Generative Adversarial Imitation Learning},
  author = {Ho, Jonathan and Ermon, Stefano},
  booktitle = {Advances in Neural Information Processing Systems},
  pages = {4565--4573},
  year = {2016},
  editor = {Lee, D. and Sugiyama, M. and Luxburg, U. and Guyon, I. and Garnett, R.},
  volume = {29},
  publisher = {Curran Associates, Inc.},
  url = {https://proceedings.neurips.cc/paper/2016/hash/cc7e2b878868cbae992d1fb743995d8f-Abstract.html}
}

@inproceedings{Ke2021Imitation,
  title = {Imitation Learning as {$f$}-Divergence Minimization},
  author = {Ke, Liyiming and Choudhury, Sanjiban and Barnes, Matt and Sun, Wen and Lee, Gilwoo and Srinivasa, Siddhartha},
  booktitle = {Algorithmic Foundations of Robotics XIV},
  pages = {313--329},
  year = {2021},
  editor = {LaValle, Steven M. and Lin, Ming and Ojala, Timo and Shell, Dylan and Yu, Jingjin},
  publisher = {Springer International Publishing},
  address = {Cham},
  doi = {10.1007/978-3-030-66723-8_19},
  isbn = {978-3-030-66723-8},
  url = {https://link.springer.com/chapter/10.1007/978-3-030-66723-8_19}
}

@inproceedings{Zhang2020FGAIL,
  title = {{f}-{GAIL}: Learning {$f$}-Divergence for Generative Adversarial Imitation Learning},
  author = {Zhang, Xin and Li, Yanhua and Zhang, Ziming and Zhang, Zhi-Li},
  booktitle = {Advances in Neural Information Processing Systems},
  pages = {12805--12815},
  year = {2020},
  editor = {Larochelle, H. and Ranzato, M. and Hadsell, R. and Balcan, M. F. and Lin, H.},
  volume = {33},
  publisher = {Curran Associates, Inc.},
  url = {https://proceedings.neurips.cc/paper/2020/hash/967990de5b3eac7b87d49a13c6834978-Abstract.html}
}

@inproceedings{Schroecker2017StateAware,
  title = {State Aware Imitation Learning},
  author = {Schroecker, Yannick and Isbell, Charles},
  booktitle = {Advances in Neural Information Processing Systems},
  pages = {2911--2920},
  year = {2017},
  url = {https://proceedings.neurips.cc/paper_files/paper/2017/hash/08e6bea8e90ba87af3c9554d94db6579-Abstract.html}
}

@inproceedings{Sun2019ProvablyEfficient,
  title = {Provably Efficient Imitation Learning from Observation Alone},
  author = {Sun, Wen and Vemula, Anirudh and Boots, Byron and Bagnell, Drew},
  booktitle = {Proceedings of the 36th International Conference on Machine Learning},
  pages = {6036--6045},
  year = {2019},
  editor = {Chaudhuri, Kamalika and Salakhutdinov, Ruslan},
  volume = {97},
  series = {Proceedings of Machine Learning Research},
  month = {09--15 Jun},
  publisher = {PMLR},
  url = {https://proceedings.mlr.press/v97/sun19b.html}
}

@inproceedings{Nachum2019DualDICE,
  author = {Nachum, Ofir and Chow, Yinlam and Dai, Bo and Li, Lihong},
  title = {DualDICE: Behavior-Agnostic Estimation of Discounted Stationary Distribution Corrections},
  booktitle = {Advances in Neural Information Processing Systems},
  editor = {Wallach, Hanna and Larochelle, Hugo and Beygelzimer, Alina and d'Alch{\'e}-Buc, Florence and Fox, Emily and Garnett, Roman},
  publisher = {Curran Associates, Inc.},
  volume = {32},
  year = {2019},
  url = {https://proceedings.neurips.cc/paper_files/paper/2019/file/cf9a242b70f45317ffd281241fa66502-Paper.pdf}
}

@article{Kostrikov2019ValueDICE,
  author = {Kostrikov, Ilya and Nachum, Ofir and Tompson, Jonathan},
  title = {Imitation Learning via Off-Policy Distribution Matching},
  journal = {CoRR},
  volume = {abs/1912.05032},
  year = {2019},
  url = {http://arxiv.org/abs/1912.05032},
  eprinttype = {arXiv},
  eprint = {1912.05032}
}

@inproceedings{Ma2022SMODICE,
  title = {Versatile Offline Imitation from Observations and Examples via Regularized State-Occupancy Matching},
  author = {Ma, Yecheng and Shen, Andrew and Jayaraman, Dinesh and Bastani, Osbert},
  booktitle = {Proceedings of the 39th International Conference on Machine Learning},
  pages = {14639--14663},
  year = {2022},
  editor = {Chaudhuri, Kamalika and Jegelka, Stefanie and Song, Le and Szepesvari, Csaba and Niu, Gang and Sabato, Sivan},
  volume = {162},
  series = {Proceedings of Machine Learning Research},
  month = {17--23 Jul},
  publisher = {PMLR},
  url = {https://proceedings.mlr.press/v162/ma22a.html}
}

@inproceedings{Reddy2020SQIL,
  title = {{SQIL}: Imitation Learning via Reinforcement Learning with Sparse Rewards},
  author = {Reddy, Siddharth and Dragan, Anca D. and Levine, Sergey},
  booktitle = {International Conference on Learning Representations},
  year = {2020},
  url = {https://openreview.net/forum?id=S1xKd24twB}
}

@inproceedings{Chiang2024ExpertProximity,
  title = {Expert Proximity as Surrogate Rewards for Single Demonstration Imitation Learning},
  author = {Chiang, Chia-Cheng and Lan, Li-Cheng and Sun, Wei-Fang and Feng, Chien and Hsieh, Cho-Jui and Lee, Chun-Yi},
  booktitle = {Proceedings of the 41st International Conference on Machine Learning},
  pages = {8336--8358},
  year = {2024},
  editor = {Salakhutdinov, Ruslan and Kolter, Zico and Heller, Katherine and Weller, Adrian and Oliver, Nuria and Scarlett, Jonathan and Berkenkamp, Felix},
  volume = {235},
  series = {Proceedings of Machine Learning Research},
  month = {21--27 Jul},
  publisher = {PMLR},
  url = {https://proceedings.mlr.press/v235/chiang24a.html}
}

@article{setlur2024rewarding,
  title={Rewarding progress: Scaling automated process verifiers for llm reasoning},
  author={Setlur, Amrith and Nagpal, Chirag and Fisch, Adam and Geng, Xinyang and Eisenstein, Jacob and Agarwal, Rishabh and Agarwal, Alekh and Berant, Jonathan and Kumar, Aviral},
  journal={arXiv preprint arXiv:2410.08146},
  year={2024}
}

@article{ladosz2022exploration,
  title={Exploration in deep reinforcement learning: A survey},
  author={Ladosz, Pawel and Weng, Lilian and Kim, Minwoo and Oh, Hyondong},
  journal={Information Fusion},
  volume={85},
  pages={1--22},
  year={2022},
  publisher={Elsevier}
}

@article{schulman2025lora,
  author = {John Schulman and Thinking Machines Lab},
  title = {LoRA Without Regret},
  journal = {Thinking Machines Lab: Connectionism},
  year = {2025},
  note = {https://thinkingmachines.ai/blog/lora/},
  doi = {10.64434/tml.20250929},
}

@article{yue2025vapo,
  title={Vapo: Efficient and reliable reinforcement learning for advanced reasoning tasks},
  author={Yue, Yu and Yuan, Yufeng and Yu, Qiying and Zuo, Xiaochen and Zhu, Ruofei and Xu, Wenyuan and Chen, Jiaze and Wang, Chengyi and Fan, TianTian and Du, Zhengyin and others},
  journal={arXiv preprint arXiv:2504.05118},
  year={2025}
}

@article{snell2022offline,
  title={Offline rl for natural language generation with implicit language q learning},
  author={Snell, Charlie and Kostrikov, Ilya and Su, Yi and Yang, Mengjiao and Levine, Sergey},
  journal={arXiv preprint arXiv:2206.11871},
  year={2022}
}

@inproceedings{agarwal2024policy,
  title={On-policy distillation of language models: Learning from self-generated mistakes},
  author={Agarwal, Rishabh and Vieillard, Nino and Zhou, Yongchao and Stanczyk, Piotr and Garea, Sabela Ramos and Geist, Matthieu and Bachem, Olivier},
  booktitle={The twelfth international conference on learning representations},
  year={2024}
}

@article{schulman2017proximal,
  title={Proximal policy optimization algorithms},
  author={Schulman, John and Wolski, Filip and Dhariwal, Prafulla and Radford, Alec and Klimov, Oleg},
  journal={arXiv preprint arXiv:1707.06347},
  year={2017}
}

@article{shao2024deepseekmath,
  title={Deepseekmath: Pushing the limits of mathematical reasoning in open language models},
  author={Shao, Zhihong and Wang, Peiyi and Zhu, Qihao and Xu, Runxin and Song, Junxiao and Bi, Xiao and Zhang, Haowei and Zhang, Mingchuan and Li, YK and Wu, Yang and others},
  journal={arXiv preprint arXiv:2402.03300},
  year={2024}
}

@inproceedings{schaul2015universal,
  title={Universal value function approximators},
  author={Schaul, Tom and Horgan, Daniel and Gregor, Karol and Silver, David},
  booktitle={International conference on machine learning},
  pages={1312--1320},
  year={2015},
  organization={PMLR}
}

@article{uesato2022solving,
  title={Solving math word problems with process-and outcome-based feedback},
  author={Uesato, Jonathan and Kushman, Nate and Kumar, Ramana and Song, Francis and Siegel, Noah and Wang, Lisa and Creswell, Antonia and Irving, Geoffrey and Higgins, Irina},
  journal={arXiv preprint arXiv:2211.14275},
  year={2022}
}

@inproceedings{zheng2025processbench,
  title={Processbench: Identifying process errors in mathematical reasoning},
  author={Zheng, Chujie and Zhang, Zhenru and Zhang, Beichen and Lin, Runji and Lu, Keming and Yu, Bowen and Liu, Dayiheng and Zhou, Jingren and Lin, Junyang},
  booktitle={Proceedings of the 63rd Annual Meeting of the Association for Computational Linguistics (Volume 1: Long Papers)},
  pages={1009--1024},
  year={2025}
}

@article{miao2023selfcheck,
  title={Selfcheck: Using llms to zero-shot check their own step-by-step reasoning},
  author={Miao, Ning and Teh, Yee Whye and Rainforth, Tom},
  journal={arXiv preprint arXiv:2308.00436},
  year={2023}
}

@article{khalifa2025process,
  title={Process reward models that think},
  author={Khalifa, Muhammad and Agarwal, Rishabh and Logeswaran, Lajanugen and Kim, Jaekyeom and Peng, Hao and Lee, Moontae and Lee, Honglak and Wang, Lu},
  journal={arXiv preprint arXiv:2504.16828},
  year={2025}
}

@inproceedings{zhang2023wisdom,
  title={The wisdom of hindsight makes language models better instruction followers},
  author={Zhang, Tianjun and Liu, Fangchen and Wong, Justin and Abbeel, Pieter and Gonzalez, Joseph E},
  booktitle={International Conference on Machine Learning},
  pages={41414--41428},
  year={2023},
  organization={PMLR}
}

@article{qu2026pope,
  title={POPE: Learning to Reason on Hard Problems via Privileged On-Policy Exploration},
  author={Qu, Yuxiao and Setlur, Amrith and Smith, Virginia and Salakhutdinov, Ruslan and Kumar, Aviral},
  journal={arXiv preprint arXiv:2601.18779},
  year={2026}
}

@article{qu2025rlad,
  title={RLAD: Training LLMs to Discover Abstractions for Solving Reasoning Problems},
  author={Qu, Yuxiao and Singh, Anikait and Lee, Yoonho and Setlur, Amrith and Salakhutdinov, Ruslan and Finn, Chelsea and Kumar, Aviral},
  journal={arXiv preprint arXiv:2510.02263},
  year={2025}
}

@article{chen2025nudging,
  title={Nudging the boundaries of llm reasoning},
  author={Chen, Justin Chih-Yao and Peng, Becky Xiangyu and Choubey, Prafulla Kumar and Huang, Kung-Hsiang and Zhang, Jiaxin and Bansal, Mohit and Wu, Chien-Sheng},
  journal={arXiv preprint arXiv:2509.25666},
  year={2025}
}

@article{zhao2026self,
  title={Self-Distilled Reasoner: On-Policy Self-Distillation for Large Language Models},
  author={Zhao, Siyan and Xie, Zhihui and Liu, Mengchen and Huang, Jing and Pang, Guan and Chen, Feiyu and Grover, Aditya},
  journal={arXiv preprint arXiv:2601.18734},
  year={2026}
}

@article{hubotter2026reinforcement,
  title={Reinforcement Learning via Self-Distillation},
  author={H{\"u}botter, Jonas and L{\"u}beck, Frederike and Behric, Lejs and Baumann, Anton and Bagatella, Marco and Marta, Daniel and Hakimi, Ido and Shenfeld, Idan and Buening, Thomas Kleine and Guestrin, Carlos and others},
  journal={arXiv preprint arXiv:2601.20802},
  year={2026}
}

@article{shenfeld2026self,
  title={Self-Distillation Enables Continual Learning},
  author={Shenfeld, Idan and Damani, Mehul and H{\"u}botter, Jonas and Agrawal, Pulkit},
  journal={arXiv preprint arXiv:2601.19897},
  year={2026}
}

@article{zelikman2022star,
  title={Star: Bootstrapping reasoning with reasoning},
  author={Zelikman, Eric and Wu, Yuhuai and Mu, Jesse and Goodman, Noah},
  journal={Advances in Neural Information Processing Systems},
  volume={35},
  pages={15476--15488},
  year={2022}
}

@article{zhou2025reinforcing,
  title={Reinforcing general reasoning without verifiers},
  author={Zhou, Xiangxin and Liu, Zichen and Sims, Anya and Wang, Haonan and Pang, Tianyu and Li, Chongxuan and Wang, Liang and Lin, Min and Du, Chao},
  journal={arXiv preprint arXiv:2505.21493},
  year={2025}
}

@article{tang2025learning,
  title={Learning to chain-of-thought with Jensen's evidence lower bound},
  author={Tang, Yunhao and Wang, Sid and Munos, R{\'e}mi and others},
  journal={arXiv e-prints},
  pages={arXiv--2503},
  year={2025}
}

@article{cai2025escaping,
  title={Escaping the verifier: Learning to reason via demonstrations},
  author={Cai, Locke and Provilkov, Ivan},
  journal={arXiv preprint arXiv:2511.21667},
  year={2025}
}

@article{chen2024language,
  title={Language models are hidden reasoners: Unlocking latent reasoning capabilities via self-rewarding},
  author={Chen, Haolin and Feng, Yihao and Liu, Zuxin and Yao, Weiran and Prabhakar, Akshara and Heinecke, Shelby and Ho, Ricky and Mui, Phil and Savarese, Silvio and Xiong, Caiming and others},
  journal={arXiv preprint arXiv:2411.04282},
  year={2024}
}
% \bibliographystyle{neurips_2025}
%%%%%%%%%%%%%%%%%%%%%%%%%%%%%%%%%%%%%%%%%%%%%%%%%%%%%%%%%%%%

% \part*{Appendices}

% \section{Technical Appendices and Supplementary Material}
% % Technical appendices with additional results, figures, graphs and proofs may be submitted with the paper submission before the full submission deadline (see above), or as a separate PDF in the ZIP file below before the supplementary material deadline. There is no page limit for the technical appendices.

% \section{Theoretical Analysis}

% Proofs of various theorems

\section{\METHODNAME\ as Variance Reduction}\label{app:theory}

% We state the assumptions of the theorem and introduce a variety of notation.

Assume that the task horizon is $T = n$, and at turn $t$, the policy always attempts to perform the reasoning at point $t$. Let $P_v$ be the set of prerequisites of $v$, including $v$ and excluding $0$, and $d_v = |P_v|$. Let $B_v \sim\mathrm{Bernoulli}(p)$ be independent across $v$ where $p \in (0,1)$ is fixed, and define $I_v = \prod_{j \in P_v} B_j$ to be an indicator for whether $v$ has been reached. Similarly define $I_g$.

Since the score $z_t = \nabla_\theta \log \pi_\theta(a_t | s_t)$ is high-dimensional (say $z_t, \theta \in \RR^k$), we instead analyze the variance of its projection $x_t = u\T z_t$ along an arbitrary unit vector $u$. However, we consider only the asymptotic behavior of each estimator with $n$, which does not depend on the choice of $u$.

We assume that pairs $(B_v, x_v)$ are independent over $v$. Define constants $\mu$, $\lambda$, $\eta$ independent of $v$ and $n$ satisfying
\begin{equation}
\EE[x_v]=0,\quad
\EE[B_vx_v]=\mu\neq0,\quad
\EE[B_vx_v^2]=\lambda,\quad
\EE[x_v^2]=\eta.
\end{equation}
By Cauchy--Schwarz, we have
\begin{equation}
    \mu^2 = \EE[B_v x_v]^2 \le \EE[B_v] \EE[B_v x_v^2] = p\lambda,
\end{equation}
so $\mu \neq 0$ and $p < 1$ imply $\lambda \ge \mu^2/p > \mu^2$.

We are now ready to define the three gradient estimates:
\begin{align}
    \text{Sparse outcome reward:} \quad Y_1 &= u\T \hat g_1 = I_g \sum_{v=1}^n x_v \\
    \text{\Methodname\ trajectory-level:} \quad Y_2 &= u\T \hat g_2 = \frac 1n \parens{\sum_{v=1}^n I_v} \parens{\sum_{v=1}^n x_v} \\
    \text{\Methodname\ segment-level:} \quad Y_3 &= u\T \hat g_3 = \frac 1n \sum_{v=1}^n I_v x_v.
\end{align}

In this appendix we characterize the relationship between the signal-to-noise ratio, $\snr(Y) := |\EE[Y]| / \sqrt{\Var(Y)}$, and the task horizon $n$ of each of these estimators. Coarsely, our main theoretical results can be summarized as
\begin{equation}
\snr(Y_1) \ll \snr(Y_2) \ll \snr(Y_3),
\end{equation}
where the latter inequality is shown in \Cref{prop:dense-beats-sparse}. Additionally, we provide worst-case analyses for the trajectory-level and turn-level \methodname\ gradient estimators. The general case lower bounds increase \textit{exponentially} from $Y_1$ to $Y_2$ and as $\operatorname{poly}(n)$ from $Y_2$ to $Y_3$, as shown in \Cref{prop:outcome-universal,prop:turn-level}. Subject to additional regularity conditions that upper-bound the correlations between the reachability of points in the graph, these lower bounds can be improved by $\operatorname{poly}(n)$. Under these assumptions, \Cref{prop:sparse,prop:dense-outcome,cor:turn-level-bounded-variance} comprise \Cref{thm:snr}.

We used OpenAI's GPT-5.6 Sol to assist in developing and refining portions of the mathematical proofs. We subsequently checked, revised, and independently verified all resulting arguments.

\subsection{SNR via sparse outcome rewards}

\begin{prop}
$\operatorname{SNR}(Y_1) = \Theta(p^{n/2}).$\label{prop:sparse}
\end{prop}

\begin{proof}
First, for every $v\in\{1,\ldots,n\}$,
\begin{equation}
\EE[I_g x_v] = \EE[B_v x_v] \prod_{\substack{j=1\\j\ne v}}^n \EE[B_j] = \mu p^{n-1},
\end{equation}
hence
$\mathbb E[Y_1] = n\mu p^{n-1}.$
Next, we compute the second moment:
\begin{equation}
\mathbb E[Y_1^2] = \sum_{v=1}^n \mathbb E[I_g x_v^2] + \sum_{v \neq w} \mathbb E[I_g x_v x_w].
\end{equation}
For the diagonal terms,
\begin{equation}
\mathbb E[I_g x_v^2] = \mathbb E[B_v x_v^2] \prod_{\substack{1 \le j \le n\\j\ne v}} \mathbb E[B_j] = \lambda p^{n-1}.
\end{equation}
For the off-diagonal terms,
\begin{equation}
\mathbb E[I_g x_v x_w] = \mathbb E[B_v x_v] \mathbb E[B_w x_w] \prod_{\substack{1 \le j \le n\\j\notin\{v,w\}}} \mathbb E[B_j] = \mu^2 p^{n-2}.
\end{equation}
Subtracting $\EE[Y_1]^2$ from the total, we obtain
\begin{equation}
\operatorname{Var}(Y_1)=n\lambda p^{n-1}+n(n-1)\mu^2p^{n-2}-n^2\mu^2p^{2n-2} = \Theta(n^2 p^{n-2}).
\end{equation}
Finally,
\begin{equation}
\operatorname{SNR}(Y_1)^2 = \frac{\mathbb E[Y_1]^2}{\operatorname{Var}(Y_1)} = \frac{\Theta(n^2 p^{2n-2})}{\Theta(n^2 p^{n-2})} = \Theta(p^n),
\end{equation}
as claimed.
\end{proof}

\subsection{SNR via trajectory-level \methodname}\label{app:snr-ours-outcome}

We analyze the estimator
\begin{equation}
    \quad Y_2 = \frac 1n \parens{\sum_{v=1}^n I_v} \parens{\sum_{v=1}^n x_v}.
\end{equation}
We define random variables
\begin{equation}
C:=\sum_{v=1}^n I_v,
\quad
D:=\sum_{v=1}^n d_vI_v,
\end{equation}
where $C$ is the count of reached points, and $D$ is the depth over reached points. For pairs of points, also define $d_{vw} := |P_v \cup P_w|$. Throughout the proofs, we will leverage the following facts:
\begin{equation}
\EE[C]=\sum_{v=1}^n p^{d_v},\quad \EE[D]=\sum_{v=1}^n d_vp^{d_v},\quad \EE[C^2]=\sum_{v=1}^n\sum_{w=1}^n p^{d_{vw}}.\label{eq:progress-moments-revision}
\end{equation}
Additionally, observe that since every DAG has a source node $v$ with $d_v = 1$, we have 
\begin{equation}
\quad \EE[D] \ge \EE[C] \ge p. \label{eq:ec-ed-lower-bounds}
\end{equation}

\vspace{2pt}

\begin{prop}
\label{prop:outcome-universal}
$\snr(Y_2) = \Omega(\sqrt{\log n} / n)$.
\end{prop}

\vspace{2pt}

\begin{prop}\label{prop:dense-outcome}
Suppose that
\begin{equation}
\EE[C^2]=\Theta\!\left(\EE[D]^2\right)
\end{equation}
and
\begin{equation}
\sum_{v=1}^n\sum_{w=1}^n d_{vw}^2p^{d_{vw}}=O\!\left(\EE[C^2]\right).
\end{equation}
Then $\operatorname{SNR}(Y_2)=\Theta(n^{-1/2})$.
\end{prop}

Before providing a proof of \Cref{prop:outcome-universal,prop:dense-outcome},
we first provide some intuition for the assumptions, which encapsulate the
reasoning graphs induced by both Multi-Countdown and Matrix Manipulation MDPs.

The first condition can be rewritten as
\begin{equation}
\frac{\sqrt{\EE[C^2]}}{\EE[C]}
=
\Theta\!\left(\frac{\EE[D]}{\EE[C]}\right).
\end{equation}
The left hand side measures the ``burstiness'' of the total number of reached reasoning points, and the right hand side measures the average depth of reached reasoning points,
\begin{equation}
\frac{\EE[D]}{\EE[C]}
=
\frac{\sum_v d_vp^{d_v}}{\sum_v p^{d_v}}.
\end{equation}
By requiring this ratio to remain $\Theta(1)$, this condition prevents deep and jointly reachable pairs $(v,w)$ from dominating the second moment of $Y_2$.

As we will see in the following examples, the dandelion graph has highly correlated visitation events, producing bursty $C$. On the other hand, a complete binary tree of depth $h$ with $p = 1/2$ has $\EE[D] = \Theta(h^2)$ and $\EE[C] = \Theta(h)$, achieving high average depth.

For the second condition, observe that
$p^{d_{vw}}=\Pr(I_v=I_w=1)$. Consequently,
\begin{equation}
\frac{\sum_{v,w}d_{vw}^2p^{d_{vw}}}{\EE[C^2]}
=
\frac{\sum_{v,w}d_{vw}^2\Pr(I_v=I_w=1)}
{\sum_{v,w}\Pr(I_v=I_w=1)}.
\end{equation}
The condition requires that this ratio is $\Theta(1)$, which prevents deep and jointly reachable pairs $(v,w)$ from dominating the second moment of $Y_2$.

\begin{itemize}
    \item The \textbf{Multi-Countdown} case satisfies the assumptions of
    \Cref{prop:dense-outcome}. Here $P_v=\set{v}$, so
    \begin{equation}
    \EE[D]=np=\Theta(n),\quad \EE[C^2]=np+n(n-1)p^2=\Theta(n^2)=\Theta\!\left(\EE[D]^2\right).
    \end{equation}
    Moreover,
    \begin{equation}
    \sum_{v,w}d_{vw}^2p^{d_{vw}}=np+4n(n-1)p^2=\Theta(n^2)=O\!\left(\EE[C^2]\right).
    \end{equation}
    \item The \textbf{Matrix Manipulation} case also satisfies the assumptions of \Cref{prop:dense-outcome}.
    Here $P_v=\set{1,\dots,v}$ and $d_{vw}=\max(v,w)$, so
    \begin{equation}
    \EE[D]=\sum_{v=1}^n vp^v=\Theta(1),\quad \EE[C^2]=\sum_{k=1}^n(2k-1)p^k=\Theta(1)=\Theta\!\left(\EE[D]^2\right).
    \end{equation}
    The second condition follows from
    \begin{equation}
    \sum_{v,w}d_{vw}^2p^{d_{vw}}=\sum_{k=1}^n(2k-1)k^2p^k=\Theta(1)=O\!\left(\EE[C^2]\right).
    \end{equation}
    \item The lower bound in \Cref{prop:outcome-universal} is achieved by the \textbf{dandelion graph}. Fix $r = \lfloor\log_{1/p} (n \log n)\rfloor$, and construct the graph $0 \to 1 \to \cdots \to r$, and for each leaf $v > r$ create a dependency $r \to v$. Intuitively, conditioned on passing point $r$, $\Theta(n)$ leaves become reachable together, which produces $\Theta(n^2)$ highly correlated leaf pairs, corresponding to a low SNR. We show in \Cref{ex:dandelion-outcome} that
    \begin{equation}
    \EE[D]=\Theta(1),\quad \EE[C^2]=\Theta\!\left(\frac{n}{\log n}\right),
    \end{equation}
    which notably do not satisfy the assumptions of \Cref{prop:dense-outcome}.
\end{itemize}

In preparation for the proof of \Cref{prop:outcome-universal}, we first bound
the second moment of reached progress $\EE[C^2]$ in terms of expected cumulative depth $\EE[D]$ (\Cref{cor:csquared-to-d}), with the help of the following analytic lemma. This relation appears when combining on-diagonal and off-diagonal terms in the expansion of $\EE[Y_2^2]$ (\Cref{eq:y2_squared_cd_decomp}).

% \begin{lemma}
%     Let $\ell_n = 1 + \log_+ (n / K_n)$, where $\log_+(x) = \max(0, \log x)$. Then $U_n = O(nK_n / \ell_n)$.
% \end{lemma}
% \begin{proof}
% Let $q_v = p^{d_v-1}$, so that $p^{d_{vw}} \le p^{\max(d_v,d_w)} = p \min(p^{d_v-1}, p^{d_w-1}) = p\min(q_v, q_w)$. Thus, $U_n \le p \sum_{v,w} \min(q_v, q_w)$. Since $d_v = 1 + \frac{\log (1/q_v)}{\log (1/p)}$, we have
% \[ K_n = \Theta\!\parens{\sum_{v=1}^n q_v (1 + \log(1/q_v))} \]
% Since $q_v$ decays geometrically with $d_v$, observe that $K_n$ is primarily determined by the largest values of $q_v$. To quantify this, let $F(t) = |\set{v : q_v \ge t}|$, so that
% \[ \sum_{v,w} \min(q_v, q_w) = \int_0^1 \parens{\sum_v \ind\set{q_v \ge t}}^2 dt = \int_0^1 F(t)^2 dt \]
% We next bound $F(t)$. The function 
% \end{proof}

\begin{lemma}
\label{lemma:pairwise-overlap}
Let $q_1,\ldots,q_n\in(0,1]$, and define
\begin{equation}
A_n = \sum_{v=1}^n q_v(1+\log(1/q_v)).
\end{equation}
Then
\begin{equation}
\sum_{v,w}\min(q_v,q_w)
=
O\!\left(
\frac{nA_n}{1+\log(n/A_n)}
\right).
\end{equation}
\end{lemma}

\begin{proof}
    Define $F(t) = |\!\set{v : q_v \ge t}\!|$, so that
    \begin{equation}
    \sum_{v,w} \min(q_v, q_w) = \int_0^1 \parens{\sum_v \ind\set{q_v \ge t}}^2 dt = \int_0^1 F(t)^2 dt.
    \end{equation}
    It remains to bound $F(t)$. Let $g(t) = t(1 + \log(1/t))$ be increasing. If $q_v \ge t$, then 
    \begin{equation}
    q_v (1 + \log(1/q_v)) = g(q_v) \ge g(t).
    \end{equation}
    Summing over the $F(t)$ values of $v$ with $q_v \ge t$, we obtain
    \begin{equation}
    A_n \ge F(t) t (1 + \log(1/t)),
    \end{equation}
    hence
    \begin{equation}
    F(t) \le \min \parens{n, \frac{A_n}{t (1 + \log(1/t))}}.
    \end{equation}
    Take $\ell_n = 1 + \log(n/A_n)$ and $t_0 = A_n / (n\ell_n)$. Since $A_n \le n$, we have $0 < t_0 \le 1$, so
    \begin{equation}
    \int_0^1 F(t)^2 dt \le n^2 t_0 + A_n^2 \int_{t_0}^1 \frac{dt}{t^2 (1 + \log(1/t))^2}.
    \end{equation}
    A change of variables $x = 1/t$ gives
    \begin{equation}
    \int_{t_0}^1 \frac{dt}{t^2 (1 + \log(1/t))^2} = \int_1^{1/t_0} \frac{dx}{(1 + \log x)^2} = O\!\parens{\frac 1{t_0 (1 + \log(1/t_0))^2}}.
    \end{equation}
    It follows that
    \begin{equation}
    \int_0^1 F(t)^2 dt \le n^2 t_0 + O\!\parens{\frac{A_n^2}{t_0 \ell_n^2}} =  O\!\parens{\frac{nA_n}{\ell_n}},
    \end{equation}
    as desired.
\end{proof}

\begin{corollary}\label{cor:csquared-to-d}
Let $\ell_n = 1 + \log_+(n/\EE[D])$, where $\log_+(x) = \max(0, \log x)$. Then
\begin{equation}
\EE[C^2]=O\!\left(\frac{n\EE[D]}{\ell_n}\right).
\end{equation}
\end{corollary}
\begin{proof}
Set $q_v:=\PP(I_v=1)=p^{d_v}$. Since
$d_{vw}\ge\max(d_v,d_w)$,
\begin{equation}
\EE[C^2]=\sum_{v,w}p^{d_{vw}}\le\sum_{v,w}\min(q_v,q_w).
\end{equation}
For the quantity $A_n$ in \Cref{lemma:pairwise-overlap},
\begin{equation}
A_n=\sum_vq_v\left(1+\log(1/q_v)\right)=\Theta\!\left(\sum_vd_vq_v\right)=\Theta\!\left(\EE[D]\right),
\end{equation}
where the constants depend only on the fixed $p$. It follows that
$1+\log(n/A_n)=\Theta(\ell_n)$. Therefore, \Cref{lemma:pairwise-overlap} gives
\begin{equation}
\EE[C^2]\le\sum_{v,w}\min(q_v,q_w)=O\!\left(\frac{nA_n}{1+\log(n/A_n)}\right)=O\!\left(\frac{n\EE[D]}{\ell_n}\right),
\end{equation}
which proves the result.
\end{proof}

\begin{proof}[Proof of \Cref{prop:outcome-universal}]
As shown in the proof of \Cref{prop:sparse},
\begin{equation}
\EE[I_vx_i]=\begin{cases}\mu p^{d_v-1}, & i\in P_v,\\0, & i\notin P_v.\end{cases}
\end{equation}
Therefore
\begin{equation}
\EE[Y_2]=\frac{\mu}{pn}\EE[D].
\end{equation}

For the second moment, let $P_{vw}:=P_v\cup P_w$, whose cardinality is
$d_{vw}$. Since $I_vI_w=\prod_{j\in P_{vw}}B_j$, expanding
$(\sum_i x_i)^2$ gives diagonal terms
\begin{equation}
\EE[I_vI_wx_i^2]=\begin{cases}\lambda p^{d_{vw}-1}, & i\in P_{vw},\\ \eta p^{d_{vw}}, & i\notin P_{vw},\end{cases}
\end{equation}
and, for $i\ne j$, off-diagonal terms
\begin{equation}
\EE[I_vI_wx_ix_j]=\begin{cases}\mu^2p^{d_{vw}-2}, & i,j\in P_{vw},\\ 0, & \text{otherwise}.\end{cases}
\end{equation}
There are respectively $d_{vw}$, $n-d_{vw}$, and
$d_{vw}(d_{vw}-1)$ terms of these three types. Therefore
\begin{equation}
\EE[Y_2^2]
=\frac1{n^2}\sum_{v,w}
\bigl[
d_{vw}\lambda p^{d_{vw}-1}
+(n-d_{vw})\eta p^{d_{vw}} +d_{vw}(d_{vw}-1)\mu^2p^{d_{vw}-2}
\bigr].\label{eq:y2_squared_expansion}
\end{equation}
The middle term is $O(\EE[C^2]/n)$. For the first and third terms, we can bound
\begin{equation}
d_{vw} \lambda p^{d_{vw}-1} + d_{vw} (d_{vw}-1) \mu^2 p^{d_{vw}-2} \le O(d_{vw}^2 p^{d_{vw}-2}) \le O(d_v^2p^{d_v-1}+d_w^2p^{d_w-1}),
\end{equation}
where the second inequality follows from $d_{vw} \ge \max(d_v, d_w)$. Consequently, their total contribution $T$ satisfies
\begin{equation}
T=O\!\left(\frac1{n^2}\sum_{v,w}d_{vw}^2p^{d_{vw}-2}\right)=O\!\left(\frac1n\sum_vd_v^2p^{d_v-1}\right).
\end{equation}
To bound the remaining sum, take
$L=\lceil3\log_{1/p}n\rceil$. points with $d_v\le L$ contribute at most
$O(L\EE[D])$.
Since $d^2p^d$ decreases exponentially beyond $L$, the at most $n$ points
with $d_v>L$ contribute $O(nL^2p^L)=o(1)=O(\EE[D])$, where the last equality follows from \Cref{eq:ec-ed-lower-bounds}. Therefore
\begin{equation}
\sum_vd_v^2p^{d_v-1}=O\!\left(\EE[D]\log n\right),
\end{equation}
and hence $T=O(\EE[D]\log n/n)$. It follows that
\begin{equation}
\EE[Y_2^2]=O\!\left(\frac{\EE[C^2]}n+\frac{\EE[D]\log n}n\right)=O\!\left(\frac{\EE[D]}{\ell_n}\right),\label{eq:y2_squared_cd_decomp}
\end{equation}
where the final equality uses the preceding corollary together with
$\ell_n\log n=O(n)$.
Consequently,
\begin{equation}
\snr(Y_2)^2\ge\frac{\EE[Y_2]^2}{\EE[Y_2^2]}=\Omega\!\left(\frac{\EE[D]\,\ell_n}{n^2}\right).
\end{equation}

\Cref{eq:ec-ed-lower-bounds} gives $\EE[D]\ge p$. If
$\EE[D]\le n$, monotonicity of
$t\mapsto t(1+\log(n/t))$ on $[p,n]$ gives
$\EE[D]\ell_n=\Omega(\log n)$. If $\EE[D]>n$, then $\ell_n=1$ and the same
conclusion is immediate. Thus
$\snr(Y_2)^2=\Omega(\log n/n^2)$ as claimed.
\end{proof}

\begin{proof}[Proof of \Cref{prop:dense-outcome}]
The middle term in \Cref{eq:y2_squared_expansion} satisfies
\begin{equation}
\sum_{v,w}(n-d_{vw})p^{d_{vw}}=n\EE[C^2]-\sum_{v,w}d_{vw}p^{d_{vw}}=\Theta\!\left(n\EE[C^2]\right).
\end{equation}
The second regularity condition makes the remaining terms lower order, so
\begin{equation}
\EE[Y_2^2]=\Theta\!\left(\frac{\EE[C^2]}n\right)=\Theta\!\left(\frac{\EE[D]^2}n\right).
\end{equation}
Since $\EE[Y_2]=\mu\EE[D]/(pn)$, its squared mean is lower order and
\begin{equation}
\snr(Y_2)^2=\frac{\Theta(\EE[D]^2/n^2)}{\Theta(\EE[D]^2/n)}=\Theta(n^{-1})
\end{equation}
as desired.
\end{proof}

\begin{example}
The dandelion graph with $r = \lfloor\log_{1/p} (n \log n)\rfloor$ achieves the $\Theta(\sqrt{\log n} / n)$ lower bound of \Cref{prop:outcome-universal}.\label{ex:dandelion-outcome}
\end{example}
\begin{proof}
In this graph,
\begin{equation}
\EE[D]=\sum_{v=1}^r vp^v+\sum_{v=r+1}^n(r+1)p^{r+1}=\Theta(1).
\end{equation}
Therefore $\EE[Y_2]=\mu\EE[D]/(pn)=\Theta(1/n)$.

By the choice of $r$, we have $r=\Theta(\log n)$ and
$p^r=\Theta(1/(n\log n))$. Decomposing the sum in
$\EE[C^2]$ gives
\begin{align}
\sum_{v,w\le r}p^{d_{vw}}&=\sum_{k=1}^r(2k-1)p^k=\Theta(1),\\
2\sum_{\substack{v\le r\\w>r}}p^{d_{vw}}&=2r(n-r)p^{r+1}=\Theta(1),\\
\sum_{v>r}p^{d_{vv}}&=(n-r)p^{r+1}=\Theta(1/\log n),\\
\sum_{\substack{v,w>r\\v\ne w}}p^{d_{vw}}&=(n-r)(n-r-1)p^{r+2}=\Theta(n/\log n).
\end{align}
Thus $\EE[C^2]=\Theta(n/\log n)$. It follows from \Cref{eq:y2_squared_cd_decomp} that
\begin{equation}
\EE[Y_2^2]=O\!\left(\frac{\EE[C^2]}n+\frac{\EE[D]\log n}n\right)=O(1/\log n).
\end{equation}
We can additionally lower-bound $\EE[Y_2^2]$ via \Cref{eq:y2_squared_expansion},
\begin{equation}
\EE[Y_2^2]\ge\frac1{n^2}\sum_{\substack{v,w>r\\v\ne w}}(n-d_{vw})\eta p^{d_{vw}}=\Theta(1/\log n).
\end{equation}
Hence $\Var(Y_2)=\Theta(1/\log n)$ and
$\snr(Y_2)=\Theta(\sqrt{\log n}/n)$.
\end{proof}

\subsection{SNR via turn-level \methodname}\label{app:snr-ours-dense}

We analyze the estimator
\begin{equation}
    Y_3 = \frac 1n \sum_{v=1}^n I_v x_v.
\end{equation}

\begin{prop}\label{prop:turn-level}
We have
\begin{equation}
\snr(Y_3)=\frac{\EE[C]}{\sqrt{\Var(C)+\kappa\,\EE[C]}},\quad \kappa:=\frac{p\lambda}{\mu^2}-1\ge0.\label{eq:y3-progress-exact-revision}
\end{equation}
Consequently,
\begin{equation}
\snr(Y_3)=\Theta\!\left(\frac{\EE[C]}{\sqrt{\Var(C)}}\right)\ge\Omega(n^{-1/2}).\label{eq:y3-progress-theta-revision}
\end{equation}
% \begin{equation}
%     \operatorname{SNR}(Y_3)
%     = \Omega\!\left(
%         \frac{1}{1 + \log_{1/p}(n/M_n)}
%     \right)
%     \ge \Omega\!\left(\frac{1}{\log n}\right).
% \end{equation}
% which is a tight lower bound per \Cref{ex:turn-level-tight}.
\end{prop}

As we discuss in the examples below, this lower bound is achievable by a dandelion graph, and in this setting $\snr(Y_2) = \snr(Y_3) = \Theta(n^{-1/2})$. However, as we show below, for non-pathological graphs, we realize the benefits of turn-level credit assignment, i.e.\ $\snr(Y_2) \ll \snr(Y_3)$.

\begin{corollary}\label{cor:turn-level-bounded-variance}
If $\Var(C) = O(\EE[C])$, then $\snr(Y_3) = \Theta\big(\sqrt{\EE[C]}\big)$.
\end{corollary}

\begin{prop}\label{prop:turn-level-regular}
If $\EE[C^2]=O\!\left(\EE[D]^2\right)$, then
\begin{equation}
\snr(Y_3)=\Omega\!\left(\frac{\EE[C]}{\EE[D]}\right)=\Omega(1/\log n).\label{eq:y3-depth-ratio-corollary-revision}
\end{equation}
\end{prop}

\begin{prop}\label{prop:dense-beats-sparse}
$\snr(Y_3) / \snr(Y_2) = \Omega(\sqrt n / \log n)$.
\end{prop}

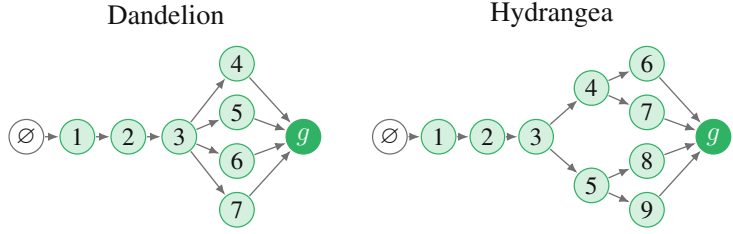
\begin{figure}
    \centering
    \resizebox{0.7\textwidth}{!}{%
    \begin{tikzpicture}[
    x=1cm,
    y=1cm
]

% Fix the dimensions and alignment of the complete two-panel figure.
\path[use as bounding box]
    (-0.45,-2.25) rectangle (16.05,3.05);

% ================================================================
% Dandelion
%
% A chain leads to one bottleneck point. Many independent leaves
% become reachable after that bottleneck.
% ================================================================

\begin{scope}
    \node[graph title]
        at (3.15,2.75)
        {Dandelion};

    % Stem of the dandelion.
    \coordinate (d0) at (0.00,0.00);
    \coordinate (d1) at (1.15,0.00);
    \coordinate (d2) at (2.30,0.00);
    \coordinate (d3) at (3.45,0.00);

    % Leaves attached to the final stem point.
    \coordinate (d4) at (4.75, 1.65);
    \coordinate (d5) at (4.75, 0.55);
    \coordinate (d6) at (4.75,-0.55);
    \coordinate (d7) at (4.75,-1.65);

    % Final goal.
    \coordinate (dg) at (6.25,0.00);

    % Stem.
    \draw[reasoning graph edge] (d0) -- (d1);
    \draw[reasoning graph edge] (d1) -- (d2);
    \draw[reasoning graph edge] (d2) -- (d3);

    % The bottleneck exposes all leaves.
    \draw[reasoning graph edge] (d3) -- (d4);
    \draw[reasoning graph edge] (d3) -- (d5);
    \draw[reasoning graph edge] (d3) -- (d6);
    \draw[reasoning graph edge] (d3) -- (d7);

    % All leaves are prerequisites of the goal.
    \draw[reasoning graph edge] (d4) -- (dg);
    \draw[reasoning graph edge] (d5) -- (dg);
    \draw[reasoning graph edge] (d6) -- (dg);
    \draw[reasoning graph edge] (d7) -- (dg);

    \node[start point]
        at (d0)
        {$\varnothing$};

    \node[reasoning point]
        at (d1)
        {1};

    \node[reasoning point]
        at (d2)
        {2};

    \node[reasoning point]
        at (d3)
        {3};

    \node[reasoning point]
        at (d4)
        {4};

    \node[reasoning point]
        at (d5)
        {5};

    \node[reasoning point]
        at (d6)
        {6};

    \node[reasoning point]
        at (d7)
        {7};

    \node[goal point]
        at (dg)
        {$g$};
\end{scope}

% ================================================================
% Chain + Binary Tree
%
% A chain leads to the root of a complete binary tree. This is a
% small schematic instance with two binary-tree levels.
% ================================================================

\begin{scope}[shift={(8.20,0)}]
    \node[graph title]
        at (3.65,2.75)
        {Hydrangea};

    % Initial chain.
    \coordinate (b0) at (0.00,0.00);
    \coordinate (b1) at (1.10,0.00);
    \coordinate (b2) at (2.20,0.00);

    % Root of the binary tree.
    \coordinate (b3) at (3.30,0.00);

    % First tree level.
    \coordinate (b4) at (4.55, 1.10);
    \coordinate (b5) at (4.55,-1.10);

    % Leaves.
    \coordinate (b6) at (5.80, 1.65);
    \coordinate (b7) at (5.80, 0.55);
    \coordinate (b8) at (5.80,-0.55);
    \coordinate (b9) at (5.80,-1.65);

    % Final goal.
    \coordinate (bg) at (7.30,0.00);

    % Chain.
    \draw[reasoning graph edge] (b0) -- (b1);
    \draw[reasoning graph edge] (b1) -- (b2);
    \draw[reasoning graph edge] (b2) -- (b3);

    % First binary-tree level.
    \draw[reasoning graph edge] (b3) -- (b4);
    \draw[reasoning graph edge] (b3) -- (b5);

    % Second binary-tree level.
    \draw[reasoning graph edge] (b4) -- (b6);
    \draw[reasoning graph edge] (b4) -- (b7);

    \draw[reasoning graph edge] (b5) -- (b8);
    \draw[reasoning graph edge] (b5) -- (b9);

    % All leaves are prerequisites of the goal.
    \draw[reasoning graph edge] (b6) -- (bg);
    \draw[reasoning graph edge] (b7) -- (bg);
    \draw[reasoning graph edge] (b8) -- (bg);
    \draw[reasoning graph edge] (b9) -- (bg);

    \node[start point]
        at (b0)
        {$\varnothing$};

    \node[reasoning point]
        at (b1)
        {1};

    \node[reasoning point]
        at (b2)
        {2};

    \node[reasoning point]
        at (b3)
        {3};

    \node[reasoning point]
        at (b4)
        {4};

    \node[reasoning point]
        at (b5)
        {5};

    \node[reasoning point]
        at (b6)
        {6};

    \node[reasoning point]
        at (b7)
        {7};

    \node[reasoning point]
        at (b8)
        {8};

    \node[reasoning point]
        at (b9)
        {9};

    \node[goal point]
        at (bg)
        {$g$};
\end{scope}

\end{tikzpicture}
    }
    \caption{Dandelion graphs achieve the worst-case SNR bounds in
    \Cref{prop:outcome-universal,prop:turn-level}, since a large number of
    visitation events are highly correlated conditioned on the low-probability
    event of reaching point $r=3$. With an additional condition on pairwise reachability, we achieve a better worst case (\Cref{prop:turn-level-regular}). This worst case is the Hydrangea graph (\Cref{ex:turn-level-tight}), which
    distributes correlations over multiple tree depths.}
    \label{fig:worst_cases}
\end{figure}

As above, we first illustrate how the moments of $C$ encode the relevant graph
structure.
\begin{itemize}
    \item In \textbf{Multi-Countdown}, $C$ is binomial with
    $\EE[C]=np$ and $\Var(C)=np(1-p)$. Hence
    \Cref{cor:turn-level-bounded-variance} gives
    $\snr(Y_3)=\Theta(\sqrt n)$.

    \item In \textbf{Matrix Manipulation}, $d_{vw}=\max(v,w)$, so
    \begin{align}
    \EE[C]&=\sum_{v=1}^np^v=\Theta(1),\\
    \EE[C^2]&=\sum_{v,w}p^{d_{vw}}=\sum_{k=1}^n(2k-1)p^k=\Theta(1).
    \end{align}
    Together with $\Var(C)\ge(1-p)\EE[C]$, this gives
    $\Var(C)=\Theta(1)$, so \Cref{cor:turn-level-bounded-variance} gives
    $\snr(Y_3)=\Theta(1)$.

    \item For the \textbf{dandelion graph} with
    $r=\lfloor\log_{1/p}n\rfloor$, we have $p^r=\Theta(1/n)$ and
    \begin{align}
    \EE[C]&=\sum_{v=1}^rp^v+(n-r)p^{r+1}=\Theta(1),\\
    \EE[D]&=\sum_{v=1}^rvp^v+(n-r)(r+1)p^{r+1}=\Theta(\log n).
    \end{align}
    Distinct leaf pairs contribute
    \begin{equation}
    (n-r)(n-r-1)p^{r+2}=\Theta(n)
    \end{equation}
    to $\EE[C^2]$. The remaining pair classes satisfy
    \begin{align}
    \sum_{v,w\le r}p^{d_{vw}}&=\Theta(1),\\
    2\sum_{\substack{v\le r\\w>r}}p^{d_{vw}}&=2r(n-r)p^{r+1}=\Theta(\log n),\\
    \sum_{v>r}p^{d_{vv}}&=(n-r)p^{r+1}=\Theta(1).
    \end{align}
    Hence $\EE[C^2]=\Theta(n)$, so $\Var(C)=\Theta(n)$ and
    \Cref{prop:turn-level} gives $\snr(Y_3)=\Theta(n^{-1/2})$. This graph does not satisfy
    $\EE[C^2]=O(\EE[D]^2)$.

    \item Under the condition $\EE[C^2]=O(\EE[D]^2)$, the
    \textbf{hydrangea graph} in
    \Cref{ex:turn-level-tight} attains the lower bound in
    \Cref{prop:turn-level-regular}. It has
    \begin{equation}
    \EE[C]=\Theta(1),\quad \EE[D]=\Theta(\log n),\quad \EE[C^2]=\Theta((\log n)^2).
    \end{equation}
    Consequently
    $\Var(C)=\Theta((\log n)^2)$ and
    \Cref{prop:turn-level} gives $\snr(Y_3)=\Theta(1/\log n)$.
\end{itemize}
\color{black}

\begin{proof}[Proof of \Cref{prop:turn-level}]
For each $v$,
\begin{equation}
\EE[I_vx_v]=\EE[B_vx_v]\prod_{j\in P_v\setminus\{v\}}\EE[B_j]=\mu p^{d_v-1},
\end{equation}
where the factorization follows from independence over $v$. Therefore, by linearity and \Cref{eq:progress-moments-revision},
\begin{equation}
\EE[Y_3]=\frac1n\sum_{v=1}^n\EE[I_vx_v]=\frac{\mu}{n}\sum_{v=1}^np^{d_v-1}=\frac{\mu}{pn}\EE[C].
\end{equation}
For the second moment, the diagonal terms satisfy
$\EE[I_vx_v^2]=\lambda p^{d_v-1}$, while for $v\ne w$,
$\EE[I_vI_wx_vx_w]=\mu^2p^{d_{vw}-2}$. Hence
\begin{equation}
\EE[Y_3^2]=\frac1{n^2}\left[\frac{\lambda}{p}\EE[C]+\frac{\mu^2}{p^2}\bigl(\EE[C^2]-\EE[C]\bigr)\right]=\frac{\mu^2\EE[C^2]+(p\lambda-\mu^2)\EE[C]}{p^2n^2}.
\end{equation}
Subtracting $\EE[Y_3]^2=(\mu^2/(p^2n^2))\EE[C]^2$ gives
\begin{equation}
\Var(Y_3)=\frac{\mu^2\Var(C)+(p\lambda-\mu^2)\EE[C]}{p^2n^2}.\label{eq:y3-variance-progress-revision}
\end{equation}
Factoring $\mu^2/(p^2n^2)$ from the variance and cancelling
$|\mu|/(pn)$ from the signal and standard deviation proves
\begin{equation}
\snr(Y_3)=\frac{\EE[C]}{\sqrt{\Var(C)+\left(\frac{p\lambda}{\mu^2}-1\right)\EE[C]}},
\end{equation}
which is \Cref{eq:y3-progress-exact-revision}.

For distinct $v,w$, we have $d_{vw}\le d_v+d_w$, and hence
\begin{equation}
\mathrm{Cov}(I_v,I_w)=p^{d_{vw}}-p^{d_v+d_w}\ge0.
\end{equation}
Also, $d_v\ge1$ implies
\begin{equation}
\Var(I_v)=p^{d_v}(1-p^{d_v})\ge(1-p)\EE[I_v].
\end{equation}
Summing gives
\begin{equation}
\Var(C)\ge\sum_v\Var(I_v)\ge(1-p)\EE[C].\label{eq:progress-variance-lower-revision}
\end{equation}
Thus the additional $\kappa\EE[C]$ term in \Cref{eq:y3-progress-exact-revision} is $O(\Var(C))$, which proves the $\Theta$ characterization in \Cref{eq:y3-progress-theta-revision}.

For the universal lower bound, $0\le C\le n$ gives
\begin{equation}
\Var(C)\le\EE[C^2]\le n\EE[C].
\end{equation}
By \Cref{eq:ec-ed-lower-bounds}, $\EE[C]\ge p$. Since $\kappa$ is fixed, we obtain
\begin{equation}
\snr(Y_3)^2\ge\frac{\EE[C]}{n+\kappa}\ge\frac{p}{n+\kappa}=\Omega(n^{-1}),
\end{equation}
as desired.
\end{proof}

\begin{proof}[Proof of \Cref{cor:turn-level-bounded-variance}]
If $\Var(C)=O(\EE[C])$, then
\Cref{eq:progress-variance-lower-revision} gives
$\Var(C)=\Theta(\EE[C])$, and the result follows from \Cref{prop:turn-level}.
\end{proof}

\begin{proof}[Proof of \Cref{prop:turn-level-regular}]
Suppose $\EE[C^2]=O(\EE[D]^2)$. Since $\EE[D]\ge\EE[C] \ge p$ per \Cref{eq:ec-ed-lower-bounds}, we have
$\EE[C]\le\EE[D]^2/p$. Therefore
\begin{equation}
\Var(C)+\kappa\EE[C]\le\EE[C^2]+\frac{\kappa}{p}\EE[D]^2=O(\EE[D]^2).
\end{equation}
\Cref{prop:turn-level} yields
\begin{equation}
\snr(Y_3)=\Omega\!\left(\frac{\EE[C]}{\EE[D]}\right).
\end{equation}
To bound this ratio, define
\begin{equation}
r_v:=\frac{p^{d_v}}{\EE[C]},\quad \sum_{v=1}^n r_v=1.
\end{equation}
It follows that
\begin{equation}
\frac{\EE[D]}{\EE[C]}=\sum_{v=1}^n r_vd_v 
=\frac{\sum_{v=1}^n r_v\log(1/r_v)-\log\EE[C]}{\log(1/p)} 
\le\log_{1/p}\!\left(\frac{n}{\EE[C]}\right)=O(\log n),
\label{eq:depth-ratio-entropy-revision}
\end{equation}
where the inequality uses the entropy bound
$\sum_v r_v\log(1/r_v)\le\log n$.
\end{proof}

\begin{proof}[Proof of \Cref{prop:dense-beats-sparse}]
We first compute the expectations,
\begin{equation}
\EE[Y_2]=\frac{\mu}{pn}\EE[D],\quad \EE[Y_3]=\frac{\mu}{pn}\EE[C].
\end{equation}

Next, we consider the asymptotics of $\Var(Y_3)$. Recall \Cref{eq:y3-variance-progress-revision}:
\begin{equation}
\Var(Y_3)=\frac{1}{p^2n^2}\left[\mu^2\Var(C)+(p\lambda-\mu^2)\EE[C]\right].
\end{equation}
Since $\Var(C)\ge(1-p)\EE[C]$, we have
\begin{equation}
\frac{\mu^2}{p^2n^2}\Var(C)\le\Var(Y_3)\le\frac{1}{p^2n^2}\left(\mu^2+\frac{p\lambda-\mu^2}{1-p}\right)\Var(C),
\end{equation}
and hence $\Var(Y_3)=\Theta(\Var(C)/n^2)$.

We next show that
\begin{equation}
\Var(Y_2)=\Omega\!\left(\frac{\Var(C)}n\right).\label{eq:vary2-progress-revision}
\end{equation}
Consider first the $v=w$ summands in the expansion of
$n^2\Var(Y_2)$ obtained from \Cref{eq:y2_squared_expansion}. Writing
$q_v:=p^{d_v}=\EE[I_v]$, the $v$th diagonal summand is
\begin{equation}
\frac{d_vq_v}{p}\left(\lambda-\frac{\mu^2q_v}{p}\right)+(n-d_v)\eta q_v+\frac{\mu^2d_v(d_v-1)q_v}{p^2}(1-q_v).
\end{equation}
Because $q_v/p\le1$, the first parenthesis is at least
$\lambda-\mu^2>0$. Thus, if $d_v\ge n/2$, the first term is
$\Omega(nq_v)$; otherwise the second term is $\Omega(nq_v)$. Since
$q_v\le p<1$, we have $\Var(I_v)=q_v(1-q_v)=\Theta(q_v)$, so the
diagonal contribution is
\begin{equation}
\Omega\!\left(n\sum_v\Var(I_v)\right).
\end{equation}

For fixed $v\ne w$, let
\begin{equation}
\gamma_{vw}:=\mathrm{Cov}(I_v,I_w)=p^{d_{vw}}-p^{d_v+d_w}\ge0.
\end{equation}
The corresponding summand from \Cref{eq:y2_squared_expansion} can be rearranged as
\begin{equation}
\begin{aligned}
&\left[\eta(n-d_{vw})+\frac{\mu^2}{p^2}d_{vw}(d_{vw}-1)\right]\gamma_{vw}
+\frac{\lambda d_{vw}}{p}p^{d_{vw}}\\
&\qquad+\left[\eta(n-d_{vw})+\frac{\mu^2}{p^2}\bigl(d_{vw}(d_{vw}-1)-d_vd_w\bigr)\right]q_vq_w.
\end{aligned}
\end{equation}
The first term is $\Omega(n\gamma_{vw})$: use its $\eta(n-d_{vw})$
term when $d_{vw}\le n/2$ and its quadratic term when $d_{vw}>n/2$.
The second term is nonnegative. The third term is also nonnegative because
$P_v\ne P_w$ in a DAG, so their cardinalities are at most $d_{vw}$ and
$d_{vw}-1$ in some order. Hence
$d_vd_w\le d_{vw}(d_{vw}-1)$. Thus,
\begin{equation}
n^2\Var(Y_2)=\Omega\!\left(n\sum_v\Var(I_v)+n\sum_{v\ne w}\mathrm{Cov}(I_v,I_w)\right)=\Omega\!\left(n\Var(C)\right),
\end{equation}
which proves \Cref{eq:vary2-progress-revision}.

Combining, we obtain
\begin{equation}
\frac{\snr(Y_3)}{\snr(Y_2)}=\frac{\EE[C]}{\EE[D]}\sqrt{\frac{\Var(Y_2)}{\Var(Y_3)}}=\Omega\!\left(\sqrt n\,\frac{\EE[C]}{\EE[D]}\right) \ge \Omega\!\parens{\frac{\sqrt n}{\log n}},
\end{equation}
where the final inequality follows from \Cref{eq:depth-ratio-entropy-revision}.
\end{proof}

\begin{example}\label{ex:turn-level-tight}
Supposing that $\EE[C^2]=O(\EE[D]^2)$, the Hydrangea graph constructed below achieves the $\Theta(1/\log n)$ lower bound of \Cref{prop:turn-level-regular}.
\end{example}
\begin{proof}
For simplicity, first take $p=1/2$. Fix a depth $d$, let
$r=\lfloor\log_2d\rfloor$, construct a stem
$0\to1\to\cdots\to r$, and attach a complete binary tree of depth $d$ at
$r$. The graph has $n=\Theta(2^d)$ points. See \Cref{fig:worst_cases}.

For a tree point at depth $k$, $d_v=r+k$. Therefore
\begin{align}
\EE[C]&=\sum_{j=1}^r2^{-j}+\sum_{k=1}^d2^k2^{-(r+k)}=\Theta(1),\\
\EE[D]&=\sum_{j=1}^rj2^{-j}+\sum_{k=1}^d2^k(r+k)2^{-(r+k)}=\Theta(d)=\Theta(\log n).
\end{align}
To compute $\EE[C^2]$, we split the sum into cases. First consider ordered pairs $(v,w)$ of tree points at depths $a,b$ (relative to $r$) whose
lowest common prerequisite has tree depth $\ell<\min(a,b)$. There are
$\Theta(2^{a+b-\ell})$ such pairs, and
$d_{vw}=r+a+b-\ell$. Each fixed $(a,b,\ell)$ therefore contributes
$\Theta(2^{-r})$. Summing,
\begin{equation}
\EE[C^2]=\Theta\!\left(2^{-r}\sum_{a=1}^d\sum_{b=1}^d\min(a,b)\right)=\Theta(2^{-r}d^3)=\Theta(d^2).
\end{equation}
Next, consider ordered pairs of tree points $(v,w)$ where one point is a prerequisite of the other. For tree depths $1 \le a \le b \le d$, there are $O(2^b)$ such pairs, each with $d_{vw} = r+b$. So their total contribution is
\begin{equation}
\sum_{1\le a\le b\le d}O\!\left(2^b2^{-(r+b)}\right)=O(2^{-r}d^2)=O(d).
\end{equation}
For a fixed stem point $j$ and tree depth $k$, there are $O(2^k)$
ordered stem--tree pairs, each with $d_{vw}=r+k$. Hence these pairs
contribute
\begin{equation}
\sum_{j=1}^r\sum_{k=1}^dO\!\left(2^k2^{-(r+k)}\right)=O(rd2^{-r})=O(r).
\end{equation}
Finally, two stem points $j,\ell$ have $d_{j\ell}=\max(j,\ell)$, so
\begin{equation}
\sum_{j,\ell=1}^r2^{-\max(j,\ell)}=\sum_{m=1}^r(2m-1)2^{-m}=O(1).
\end{equation}
Because $r=\Theta(\log d)$, the total is $\EE[C^2] = \Theta(d^2) = \Theta(\EE[D]^2)$, satisfying the assumption.

Since $\EE[C]=\Theta(1)$,
\begin{equation}
\Var(C)=\EE[C^2]-\EE[C]^2=\Theta(d^2).
\end{equation}
Applying \Cref{prop:turn-level} yields
\begin{equation}
\snr(Y_3)=\Theta\!\left(\frac{\EE[C]}{\sqrt{\Var(C)}}\right)=\Theta(1/d)=\Theta(1/\log n).
\end{equation}

For general fixed $p\in(0,1)$, take
$r=\lfloor\log_{1/p}d\rfloor$ and use a balanced rooted tree whose
$k$th level contains $\lceil p^{-k}\rceil$ points. Here balanced means
that descendants are distributed as evenly as possible, so each point at
level $i$ has $\Theta(p^{-(j-i)})$ descendants at level $j>i$. Then
$n=\Theta(p^{-d})$, so $d=\Theta(\log n)$, and the same moment
calculations give
$\EE[C]=\Theta(1)$, $\EE[D]=\Theta(d)$, and
$\EE[C^2]=\Theta(d^2)$.
\end{proof}

\section{Experiment Details}\label{app:experiment-details}

\begin{wraptable}[23]{r}{0.5\textwidth}
    \centering
    \vspace{-2em}
    \caption{Hyperparameters for RL training.}
    \vspace{0.5em}
    \begin{tabular}{cc}
    \toprule
        Hyperparameter & Value \\
    \midrule
        Sampling temperature & 1.0 \\
        Clip ratio (low/high) & 0.20 / 0.28 \\
        Learning rate & 3e-6 \\
        Training steps & 150 \\
        Trajectories per prompt & 16 \\
        Prompts per batch & 8 \\
        Global batch size & 128 \\
    \bottomrule
    \end{tabular}
    \label{tab:experiment-details-rl}
\vspace{0.5em}
    \caption{\Methodname\ ablations on GSM-Infinite $n = 24$ at 400M training tokens.}
    \vspace{0.5em}
\begin{tabular}{lc}
\toprule
Method & Success rate \\
\midrule
Base model & 0.057 \\
Sparse outcome reward & 0.091 \\
PPM & \textbf{0.191} \\
PPM, no outcome reward & 0.063 \\
PPM, stdev normalization & 0.071 \\
PPM, trajectory-length chunks & 0.148 \\
% PPM, advantage mask + global norm & 0.165 \\
PPM, no length filter & 0.171 \\
\bottomrule
\end{tabular}
\label{tab:method-ablations}
\end{wraptable}

In our experiments, we train the Qwen3-1.7B, Qwen3-4B, and Qwen3-4B-Instruct models on v5e-32 TPU pods. Each experiment ran for 12--24 hours, with longer runtimes for standard GRPO due to a large number of sampling iterations with no successful  trajectories. We adopt a synchronous RL framework with GRPO. We train fully on-policy, i.e.\ we alternate between 1 sampling step and 1 training step. We do not use either entropy or KL regularization, and found empirically that tuning the clip high ratio was sufficient to stabilize entropy over training. Following \citet{yu2025dapoopensourcellmreinforcement}, we filter out prompts where all advantages are zero. The full hyperparameters are shown in Table~\ref{tab:experiment-details-rl}. By default, we mask out tokens with an at least $2\x$ train-inference probability mismatch.

Throughout our experiments we instantiate the judge model as gemini-2.5-flash-lite \cite{gemini25flashlite2025}, with prompts provided in Appendix~\ref{app:prompts}. Our reasoning points and reasoning graphs are generated by gemini-3-flash.

Throughout the paper we refer to token budgets as multiples of $\text K = 1024$, e.g.\ a 4K token budget is 4096 tokens.

For simplicity, we tuned the hyperparameters in \Cref{tab:experiment-details-rl} for standard GRPO, and we generally reuse the same configurations for all other methods. Thus, your favorite GRPO hyperparameters will generally work well with \methodname\ out of the box.

\Methodname\ differs from standard GRPO in several ways. We show ablations for each of these configuration options in \Cref{tab:method-ablations} and describe them in more detail below:

\begin{itemize}
    \item \textbf{Outcome reward.} Empirically, we found that training on the \textit{sum} of the outcome reward and the method-specific reward significantly improves training speed, so that the policy can learn to correctly format answers and output them within the training token budget. For methods involving splitting trajectories into segments, we add the outcome reward to each segment's reward, then proceed with GRPO advantage calculations.

    \vspace{0.4em}

    % We remark that on the POPE-hard dataset (Section~\ref{sec:near-impossible-tasks}), the outcome reward is zero 99.5\% of the time starting from the base policy, i.e.\ the training reward almost always equals the \methodname\ reward in this setting.

    \item \textbf{No standard deviation (stdev) normalization.} Typically, advantages are computed within each GRPO group by subtracting the group mean and dividing by the group standard deviation. We observe that with \methodname, it is beneficial to only perform the subtraction but not division.

    \item \textbf{Length filter.} In the loss, we mask out trajectories that are truncated by the training token budget. This alleviates length pressure.
\end{itemize}

\begin{wrapfigure}[9]{r}{0.5\textwidth}
    \vspace{-1.2em}
    \resizebox{\linewidth}{!}{%
        \begin{tikzpicture}[
    x=1cm,
    y=1cm,
    block/.style={
        draw,
        line width=0.55pt
    },
    boundary/.style={
        line width=0.4pt
    }
]

\def\barW{4.8}
\def\barH{0.44}
\def\usedW{2.8125} % 4.8 * 2400 / 4096

% Top row: generated tokens within total budget.
\path[fill=ppmlightgreen] (0,1.34) rectangle (\usedW,1.34+\barH);
\draw[block] (0,1.34) rectangle (\barW,1.34+\barH);
\node[font=\normalsize] at (0.5*\usedW,1.34+0.5*\barH) {2400 tokens};
\node[anchor=west,font=\normalsize] at (4.92,1.34+0.5*\barH) {4096 token budget};

% Middle row: training-budget chunks.
\path[fill=ppmlightgreen] (0,0.77) rectangle (\usedW,0.77+\barH);
\draw[block] (0,0.77) rectangle (\barW,0.77+\barH);
\foreach \x in {1.2,2.4,3.6} {
    \draw[boundary] (\x,0.77) -- (\x,0.77+\barH);
}
\foreach \x/\chunk in {0.6/1,1.8/2,3.0/3,4.2/4} {
    \node[font=\normalsize] at (\x,0.77+0.5*\barH) {\chunk};
}
\node[anchor=west,font=\normalsize] at (4.92,0.77+0.5*\barH) {1024 chunk length};

% Bottom row: response-length chunks.
\path[fill=ppmlightgreen] (0,0.20) rectangle (\usedW,0.20+\barH);
\draw[block] (0,0.20) rectangle (\barW,0.20+\barH);
\foreach \x in {0.703125,1.40625,2.109375,2.8125} {
    \draw[boundary] (\x,0.20) -- (\x,0.20+\barH);
}
\foreach \x/\chunk in {0.3515625/1,1.0546875/2,1.7578125/3,2.4609375/4} {
    \node[font=\normalsize] at (\x,0.20+0.5*\barH) {\chunk};
}
\node[anchor=west,font=\normalsize] at (4.92,0.20+0.5*\barH) {600 chunk length};

\end{tikzpicture}%
    }
    \caption{\Methodname\ (middle) follows the heuristic of equal-length chunks regardless of sampled trajectory length.}
    \label{fig:chunking}
\end{wrapfigure}

When segmenting trajectories without natural turn-level boundaries, we use a simple baseline of \textbf{equal-length chunks}. As in \Cref{fig:chunking} (middle), we may have some empty chunks if the trajectory is considerably shorter than the training token budget, which are assigned \methodname\ rewards of zero for advantage calculations. Empirically, we find that this chunking scheme generally outperforms and leads to more stable training than trajectory-length chunks (\Cref{fig:chunking}, bottom). Trajectory-length chunks may result in length pressures, as token-level advantages are applied to chunks of different lengths within the same GRPO group.

\subsection{Methods and baselines}\label{app:baseline-impls}

We next describe the methods and baselines to which we compared \methodname. As above, we train on the sum of outcome rewards and \methodname\ rewards for Verifree and process rewards. We always use gemini-2.5-flash-lite as the judge for process rewards and \methodname; we evaluate this in Appendix~\ref{sec:instantiating-judge}.
\begin{itemize}
    \item \textbf{Supervised fine-tuning (SFT).} Throughout the paper, we assume access to human- or oracle-written \textit{solutions}, rather than full reasoning traces (after all, this is simply distillation from a stronger model). Thus, we directly fine-tune the model on sequences \small\texttt{[prompt, solution]}\normalsize, using the \textit{non-thinking} tokenizer mode in Qwen3 models.
    \item \textbf{Verifree \cite{zhou2025reinforcing}.} The original method  maximizes the likelihood of generating the dataset answer. Specifically, the model is prompted to generate responses formatted as \small

        \begin{center}\color{gray}\verb|[preceding tokens] <answer|\color{black}\verb|> \boxed{answer here} </answer> <eos> |\end{center}\normalsize
        
        where ``preceding tokens'' include the prompt, thinking, and all output tokens except for the answer.
        Then, the reward is the likelihood of the \small\texttt{black}\normalsize\ tokens, with \small\texttt{answer here}\normalsize\ replaced by the dataset answer, conditioned on the \small\color{gray}\texttt{gray}\normalsize\color{black}\ tokens.
        
        \vspace{0.4em}

        However, we found empirically that this instantiation of Verifree is sensitive to the format of the answer tokens and underperforms outcome rewards, even when clipping token logprobs to $[-10, 0]$ for stability or summing Verifree rewards over multiple LLM-reformatted dataset answers. For instance, the answer $-2/3$ can be written as \small\verb|-2/3|\normalsize, \small\verb|-\frac 23|\normalsize, or \small\verb|$\dfrac{-2}{3}$|\normalsize, among a variety of answer formats. This instantiation of Verifree matched outcome rewards only in the setting where the dataset was filtered to only integer answers, thus reducing variance due to answer formatting. But in this setting these rewards are conceptually identical: if, for example, the conditioning tokens end with
        
        \begin{center}\small\color{gray}\verb|The final answer is 60.\n\n<answer|\normalsize\color{black}\end{center}
        
        then the likelihood of

        \begin{center}\small\verb|> \boxed{dataset answer} </answer> <eos> |\normalsize\end{center}

        is effectively 1 if the dataset answer is 60 and 0 otherwise, matching outcome rewards.

        \vspace{0.4em}

        Instead, we observe empirically that this reward decreases in variance when averaging token logprobs over longer target sequences, simply by maximizing the likelihood of the reference solution:

        \begin{center}\small\color{gray}\verb|[prompt] <think> [thinking] |\color{black}\verb|</think> [reference solution] <eos> | \end{center}\normalsize

        Specifically, we decompose the sampled sequence as $x_{0:T}=[x^{\mathrm{prompt}},x^{\mathrm{think}},x^{\mathrm{out}}]$ and let $\tilde{x}=[x^{\mathrm{prompt}},x^{\mathrm{think}}]$ denote the sequence obtained by discarding the model's final output. For an oracle solution $x^{\mathrm{ref}}$ with length $L$ tokens, we maximize
    \begin{equation}\mathbb{E}_{x\sim\pi_\theta}\left[\exp\left(\frac{1}{L}\sum_{t=1}^{L}\log\pi_\theta\left(x^{\mathrm{ref}}_t\mid\tilde{x},x^{\mathrm{ref}}_{<t}\right)\right)\right]\end{equation}

    \item \textbf{On-policy self distillation (OPSD) \cite{zhao2026self}.} We disable group filtering and use a \textit{fixed} teacher initialized as the base policy, which we found empirically led to more stable training than using a teacher parameterized as a moving average of the policy weights. We otherwise follow the OPSD official code. The objective is
    % \small
\begin{equation}
\mathbb{E}_{x^{\mathrm{gen}}\sim\pi_\theta(\cdot\mid x^{\mathrm{prompt}})}\!\left[\frac{1}{T}\sum_{t=1}^{T}D_{\mathrm{JS}}\!\left(\pi_{\theta_0}\!\left(\cdot\mid x^{\mathrm{prompt}},x^{\mathrm{ref}},x^{\mathrm{gen}}_{<t}\right)\,\middle\|\,
  \pi_\theta\!\left(\cdot\mid x^{\mathrm{prompt}},x^{\mathrm{gen}}_{<t}\right)\right)\right]
\end{equation}
\normalsize
where $D_{\mathrm{JS}}$ denotes the Jensen--Shannon divergence.
% For OPSD, let $x^{\mathrm{gen}}=[x^{\mathrm{think}},x^{\mathrm{out}}]=(x^{\mathrm{gen}}_1,\ldots,x^{\mathrm{gen}}_N)$ and let $\pi_{\theta_0}$ denote the frozen base-model teacher. The teacher receives the reference solution as privileged context, while the student receives only the original prompt and its sampled prefix. OPSD minimizes

    \item \textbf{Process rewards.} We follow the training configuration for \methodname. However, following prior work in process rewards \cite{miao2023selfcheck, khalifa2025process}, there are three key differences: (i) the judge prompt evaluates correctness of each step rather than progress to the goal, (ii) no shortcutting, and (iii) rewards are obtained from judging each trajectory segment independently $r_k := \phi'(c_k, x^{\text{ref}})$ rather than taking finite differences on the full prefix $\phi((c_1, \dots, c_k), x^{\text{ref}})$. As we show in \Cref{prop:optimality-equivalence}, (i) and (ii) result in a \textit{biased} objective. Furthermore, in \methodname\ experiments, we have found that (iii) conditioning on full prefixes led to lower-variance segment-level rewards, as it is beneficial to have access to relevant context such as definitions and notation introduced early in the trajectory.

    \item \textbf{Privileged on-policy exploration (POPE) \cite{qu2026pope}.} We adapt the prompt in the original paper. In each batch, half of the prompts contain a prefix of the reference solution, and half of the prompts do not.

    \item \textbf{\Methodname, exact order baseline.} We provide a reward only if the reasoning points are visited in the same order as in the reference trajectory. This baseline provides an approximation to imitation learning (without teacher forcing). However, as we show in \Cref{tab:pope-dataset}, this baseline significantly reduces exploration compared to standard \methodname, perhaps because some reasoning points may be ``out of support'' for the policy.

    \item \textbf{\Methodname, no shortcutting baseline.} We use the reward in \Cref{eq:reward} without shortcutting. As we argue in the proof to \Cref{prop:optimality-equivalence}, shortcutting is required for the set of optimal policies to coincide with those from outcome rewards.
\end{itemize}

\subsection{Environment-specific training details}

We follow the defaults in \Cref{tab:experiment-details-rl} and as described above in Appendix~\ref{app:baseline-impls}. The environments and datasets themselves are described in Appendix~\ref{app:environments}.

\begin{itemize}
    \item \textbf{Multi-Countdown.} We train Qwen3-1.7B, and each sub-problem is given a 512 token budget (for example, in the case $n=8$, the policy generates $8 \cdot 512 = 4K$ tokens). We observed that when the agent failed to produce an answer within this budget, it would sometimes continue working on the unfinished sub-problem in subsequent turns. To prevent this effect, 50 tokens before the end of each turn's budget, we force the model to terminate its reasoning with \texttt{</think>} and begin its final response with \texttt{<answer>}, leaving sufficient budget to produce an answer.
    \item \textbf{Matrix Manipulation.} We train Qwen3-4B-Instruct-2507 with thinking mode disabled, and each manipulation operation is given a 512 token budget. We use gemini-2.5-flash-lite as the judge model with thinking disabled. We show horizon scaling results in \Cref{fig:matrix-manipulation}.
\begin{figure}
    \centering
    \includegraphics[width=\linewidth]{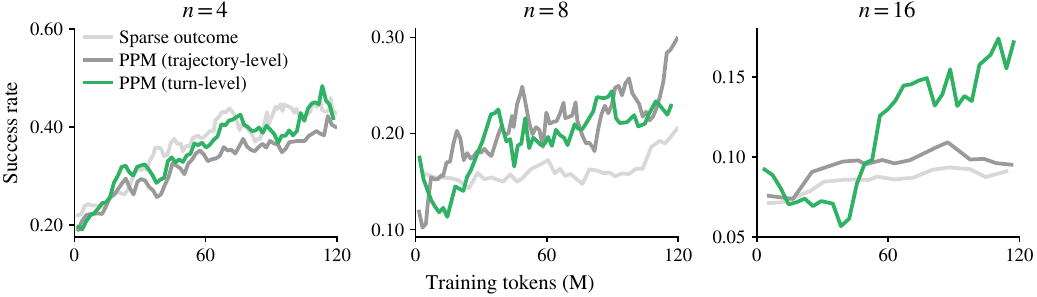}
    \caption{When training on the sequential Matrix Manipulation task, \methodname\ yields increasing performance improvements over sparse outcome rewards with the task horizon ($n$). However, agreeing with \Cref{thm:snr}, this gap increases at a slower rate than in the Multi-Countdown case (\Cref{fig:multicountdown}).}
    \label{fig:matrix-manipulation}
\end{figure}
    \item \textbf{GSM-Infinite.} We train Qwen3-1.7B with thinking disabled at a training budget of 4K tokens.
    \item \textbf{Polaris.} We train Qwen3-4B at a training budget of 8K tokens, and we evaluate at 16K tokens. Our RL runs use a learning rate of 1e-6.
    \begin{itemize}
        \item \textbf{SFT.} We train for 200 steps with a global batch size of 128 and learning rate of 5e-6.
    \end{itemize}
    \item \textbf{POPE-hard.} We train Qwen3-4B-Instruct at a training budget of 4K or 8K tokens, and we evaluate 16K tokens.
    \begin{itemize}
        \item \textbf{POPE.} We use 1-paragraph prefixes from oracle solutions. Note in \Cref{tab:pope-dataset} that the training reward is measured on the mixed dataset (both guided and non-guided), and thus cannot be directly compared to the other training rewards.
        \item \textbf{\Methodname.} We remark that our method makes it possible to train on almost all trajectories (around 95\% averaged over training) compared to almost none (around 0.2\% averaged over training) with sparse outcome rewards.

        We use 16 prompts per batch and a KL penalty with the k3 estimator and 0.002 coefficient. We train only on the PPM reward, rather than summing the outcome reward and PPM reward.
    \end{itemize}
\end{itemize}

We define segments by splitting the training token budget into 4 equal-length chunks. For example,  Following \Cref{eq:reward}, we obtain segment-level rewards $r_k$. 

Throughout the paper we refer to pass@$k$ (or p@$k$ for brevity) as the probability that
at least one of $k$ trajectories sampled i.i.d.\ from the policy for a given prompt attains
an outcome-level reward of $1$. For each prompt we draw $n \gg k$ trajectories, of which
$c$ receive a reward of $1$, and use the standard unbiased estimator of this probability
% \citep{chen2021codex},
\begin{equation}
\widehat{\text{pass@}k} = 1 - \frac{\binom{n-c}{k}}{\binom{n}{k}}
\end{equation}
which is the probability that a uniformly random size-$k$ subset of the $n$ drawn
trajectories contains at least one success. We report the mean of
$\widehat{\text{pass@}k}$ over the prompts in the evaluation set, which is likewise unbiased for
the population quantity $\mathbb{E}_{x^{\text{prompt}} \sim \mathcal{D}}\big[\text{pass@}k(x^{\text{prompt}})\big]$.
%Error bars are 95\% bootstrap confidence intervals over prompts
% (10000 resamples), and therefore reflect variation across prompts rather than
% Monte Carlo error in the per-prompt estimate.

\subsection{Reasoning points, graphs, and judges}\label{sec:instantiating-judge}

\Methodname\ relies on LLMs besides the policy to (i) generate reasoning points, (ii) generate reasoning graphs, and (iii) judging policy-generated responses. How sensitive is the method to the choice of LLM in each case? In this section, we show that any model with sufficient instruction following capabilities can serve as any of the three roles.

Before evaluating each of these three requirements, we remark that the use of external LLMs in the 4-chunk case is fairly wall clock efficient: reasoning points and graphs are constructed once along with the dataset, and judging largely relies on prefill (vs.\ autoregressive generation).

\textbf{Evaluating reasoning points.} We compare against ground truth reasoning points via GSM-Infinite. Qualitatively, an example is shown in \Cref{fig:gt-rubric-comparison}. Numerically, we find that various models achieve roughly identical performance in defining reasoning points (\Cref{tab:rubric-generation}).

\begin{figure}
    \centering
 \begin{minipage}[t]{0.49\textwidth}
  \textbf{\small Ground-truth reasoning points}

  \vspace{0.5em}
  {\ttfamily\scriptsize\raggedright\setlength{\parindent}{0pt}
  \diff{1. Defines the target quantity as a variable: let $x$ be the public highschool in Clearwater Bay.}\\[0.35em]
  2. Derives that the average number of teachers per private middle school in Ruby Bay is 2.\\[0.35em]
  3. Derives that the average number of teachers per public highschool in Ruby Bay is 8.\\[0.35em]
  4. Derives that the average number of teachers per public highschool in Clearwater Bay is 8.\\[0.35em]
  5. Derives that the number of private middle school in Clearwater Bay is 8.\\[0.35em]
  6. Derives that the total number of schools in Clearwater Bay is $x + 8$.\\[0.35em]
  7. Uses the given fact that the total number of schools in Clearwater Bay equals 11.\\[0.35em]
  8. Solves the equation $x + 8 = 11$ to obtain $x = 3$.\\[0.35em]
  9. Provides the final answer in the format $\boxed{3}$.
  \par}
  \end{minipage}
  \hfill
  \begin{minipage}[t]{0.49\textwidth}
  \textbf{\small Qwen3-4B-Instruct reasoning points}

  \vspace{0.5em}
  {\ttfamily\scriptsize\raggedright\setlength{\parindent}{0pt}
  \diff{1. The number of public highschool in Ruby Bay is correctly identified as 4.}\\[0.35em]
  2. The average number of teachers per private middle school in Ruby Bay is correctly identified as 2.\\[0.35em]
  3. The average number of teachers per public highschool in Ruby Bay is correctly calculated as 8.\\[0.35em]
  4. The average number of teachers per public highschool in Clearwater Bay is correctly calculated as 8.\\[0.35em]
  5. The number of private middle school in Clearwater Bay is correctly calculated as 8.\\[0.35em]
  6. The total number of schools in Clearwater Bay is correctly expressed as the sum of the number of public highschool
  in Clearwater Bay and the number of private middle school in Clearwater Bay, so it is correctly expressed as $x +
  8$.\\[0.35em]
  7. The total number of schools in Clearwater Bay is correctly given as 11, so the equation $x + 8 = 11$ is
  formed.\\[0.35em]
  8. The number of public highschool in Clearwater Bay is correctly solved as 3.\\[0.35em]
  9. The final answer is provided in the format $\boxed{3}$.
  \par}
  \end{minipage}
  % \caption{The only rubric-generation failure at $\text{ops}=8$ (1 of 32 problems).
  % Items~2--9 correspond one-to-one. The generated rubric omits the variable declaration
  % (left, \diff{highlighted}) and spends its first item on a given that the reference folds
  % away (right, \diff{highlighted}), though it uses $x$ correctly in items~6--8.}
    \caption{We evaluate LLM generated reasoning points against ground truth reasoning points from GSM-Infinite. Above is the only failure at $n=8$ operations of 32 held-out examples.}
    \label{fig:gt-rubric-comparison}
\end{figure}

\begin{table}
\centering
\caption{\textbf{Reasoning point evaluation.} Fraction of ground truth reasoning points that are also LLM-generated reasoning points, on 32 held-out prompts, judged by
gemini-3-flash-preview.}
\label{tab:rubric-generation}
\begin{tabular}{rcccc}
\toprule
Horizon $n$ & Qwen3-4B-Instruct & Qwen3-32B & gemini-2.5-flash-lite & gemini-3-flash-preview \\
\midrule
 8 & \textbf{0.9965} & \textbf{0.9965} & \textbf{0.9965} & 0.9773 \\
16 & \textbf{0.9765} & 0.9455 & 0.9555 & 0.9418 \\
24 & \textbf{0.9712} & 0.9591 & 0.9551 & 0.9654 \\
\bottomrule
\end{tabular}
% \begin{tabular}{rccc}
% \toprule
% Horizon $n$ & Qwen3-32B & gemini-2.5-flash-lite & gemini-3-flash-preview \\
% \midrule
%  8 & 0.9965 & 0.9965 & 0.9773 \\
% 16 & 0.9455 & 0.9555 & 0.9418 \\
% 24 & 0.9591 & 0.9551 & 0.9654 \\
% \bottomrule
% \end{tabular}
\end{table}

\textbf{Evaluating reasoning graphs.} In \Cref{tab:reasoning-graph}, we evaluate what fraction of prerequisite pairs in ground truth graphs via GSM-Infinite that are also present in LLM-generated graphs. While there is degradation in reasoning graphs with horizon, this degradation occurs uniformly across models, and may occur if the models have not been trained sufficiently on XML-like data. %A hypothesis for the performance degradation in gemini-2.5-flash-lite and Qwen3-32B is that they are undertrained in generating XML code. For example, they are much worse at generating SVGs than gemini-3-flash-preview \cite{johnbean2024svgbench}.

\begin{table}
\centering
\label{tab:reasoning-graph}
\caption{Reasoning graph evaluation.}
% \begin{tabular}{rccc}
% \toprule
% Horizon $n$ & Qwen3-32B & gemini-2.5-flash-lite & gemini-3-flash-preview \\
% \midrule
% 8  & 0.9259 & 0.9465 & 0.9714 \\
% 16 & 0.9048 & 0.9205 & 0.9682 \\
% 24 & 0.8467 & 0.8735 & 0.9819 \\
% \bottomrule
% \end{tabular}
  \begin{tabular}{rcccc}
  \toprule
  Horizon $n$ & Qwen3-4B-Instruct & Qwen3-32B & gemini-2.5-flash-lite & gemini-3-flash-preview \\
  \midrule
   8 & \textbf{0.9965} & \textbf{0.9965} & \textbf{0.9965} & 0.9773 \\
  16 & \textbf{0.9765} & 0.9455 & 0.9555 & 0.9418 \\
  24 & \textbf{0.9712} & 0.9591 & 0.9551 & 0.9654 \\
  \bottomrule
  \end{tabular}
\end{table}

\textbf{Evaluating the PPM measure $\phi$.} We follow the evaluation described in \Cref{sec:q-judge} on the POPE-hard dataset \cite{qu2026pope}, which is very challenging for the base policy (Qwen3-1.7B). We sample 8 prompts, 5 reasoning states to condition on in the prompt (0\%, 25\%, 50\%, 75\%, 100\% of the reference solution), and 8 trajectory continuations per setting. We sample 4K tokens, with \texttt{</think>} forcing at 1024 tokens remaining and \texttt{<answer>} forcing at 50 tokens remaining.

We find that $\phi$ increases roughly linearly with the length of the reference trajectory prefix we condition on, and that judge scores are consistent across models.

\begin{figure}
    \centering
    \includegraphics[width=\linewidth]{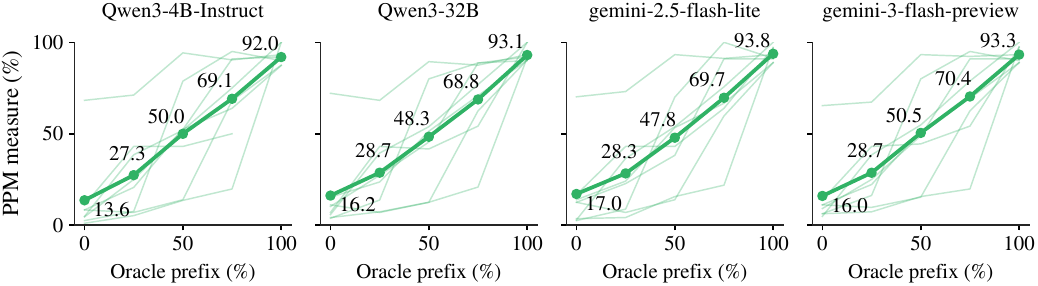}
    \caption{\textbf{PPM measure $\phi$ evaluation.} Across models, $\phi$ increases roughly linearly with the proportion of the reference trajectory we condition on in the prompt.}
    \label{fig:gemini-prefix}
\end{figure}

% \begin{figure}
%     \centering
%     \includegraphics[width=0.5\linewidth]{imgs/gemini_prefix_pope_rubric.pdf}
%   \caption{Providing longer prefixes of the oracle solution in context to our \methodname\ judge leads to roughly linearly increasing \methodname\ measure.}
%     \label{fig:gemini-prefix}
% \end{figure}

\begin{table}[t]
\centering
\caption{Qwen3-4B-Instruct trained on POPE-hard dataset and evaluated at larger test-time budget.}
\begin{tabular}{lcccc}
\toprule
& \multicolumn{2}{c}{Train} & \multicolumn{2}{c}{Test Time} \\
\cmidrule(lr){2-3} \cmidrule(lr){4-5}
Method
& \shortstack{4K Tokens \\ pass@1}
& \shortstack{4K Tokens \\ pass@8}
& \shortstack{16K Tokens \\ pass@1}
& \shortstack{16K Tokens \\ pass@8} \\
\midrule
Base model & 0.000 & 0.000 & 0.004 & 0.023 \\
GRPO & 0.000 & 0.000 & 0.004 & 0.023 \\
% GRPO + Value Baseline & -- & -- & -- & -- \\
% Teacher SFT & -- & -- & -- & -- \\
% Reference SFT & -- & -- & 0.442 & 0.689 \\
% OPSD & 0.097 & 0.196 & 0.140 & 0.367 \\
Verifree & 0.008 & 0.008 & 0.006 & 0.040 \\
POPE & 0.013 & 0.078 & 0.012 & 0.063 \\
Process rewards & 0.0003 & 0.002 & 0.009 & 0.054 \\
PPM, exact order & 0.0009 & 0.014 & 0.0004 & 0.032 \\
PPM, no shortcutting & 0.001 & 0.019 & 0.021 & 0.093 \\
PPM & 0.005 & 0.019 & 0.035 & \textbf{0.148} \\
\midrule
& \shortstack{8K Tokens \\ pass@1}
& \shortstack{8K Tokens \\ pass@8}
& \shortstack{16K Tokens \\ pass@1}
& \shortstack{16K Tokens \\ pass@8} \\
\midrule
Base model & 0.004 & 0.031 & 0.004 & 0.023 \\
GRPO & 0.002 & 0.012 & 0.004 & 0.023 \\
PPM, exact order & 0.005 & 0.027 & 0.006 & 0.023 \\
PPM, no shortcutting & 0.026 & 0.069 & 0.041 & 0.117 \\
PPM & 0.018 & 0.074 & 0.039 & \textbf{0.125} \\

% State-matching + POPE (ours) & -- & -- & -- & -- \\
\bottomrule
\end{tabular}
\label{tab:pope-dataset}
\end{table}

% \begin{wrapfigure}[17]{r}{0.5\textwidth}
%     \centering
%   \label{fig:gemini-prefix-pope-rubric}
%   \vspace{-1.5em}
%     \includegraphics[width=0.48\textwidth]{imgs/gemini_prefix_pope_rubric.pdf}
%   \vspace{-0.5em}
%   \caption{Providing longer prefixes of the oracle solution in context to our \methodname\ judge leads to roughly linearly improving judge scores.}
% \end{wrapfigure}

\section{Environments}\label{app:environments}

\begin{itemize}
    \item \textbf{Multi-Countdown.} We use the standard 4-number Countdown task in a multi-turn setting. At each turn, the agent receives a new sub-problem and is given a budget of 512 tokens to solve it.

    \item \textbf{Matrix Manipulation.} We adapt the Matrix Manipulation environment from Reasoning Gym \cite{stojanovski2025reasoninggymreasoningenvironments}. We resolved several bugs with the original environment, such as ambiguous or incorrectly labeled text descriptions of the matrix operations, and eliminating operations such as cropping that can artificially reduce the task horizon. The adapted environment is included as part of our code release. Each manipulation operation is presented as its own turn.

    \item \textbf{GSM-Infinite.} The standard GSM-Infinite environment \cite{zhou2025gsminfinitellmsbehaveinfinitely} generates datasets containing problems with a mixture of task horizons. To isolate the effect of task complexity, we instead construct datasets at a fixed complexity $n$ by rejection sampling from the original environment. For each prompt, we additionally fix the graph composition such that 40\% of edges are task-relevant and the remaining 60\% are distractors. \textbf{The datasets with 5000 prompts each are available here:}

\vspace{0.4em}
    $n=8$: {\small\url{https://huggingface.co/datasets/prestonfu/gsm_infinite_hard_r0.4_ops8}}
    
    $n=16$: {\small\url{https://huggingface.co/datasets/prestonfu/gsm_infinite_hard_r0.4_ops16}}
    
    $n=24$: {\small\url{https://huggingface.co/datasets/prestonfu/gsm_infinite_hard_r0.4_ops24}}
    \vspace{0.4em}

    The datasets are automatically generated with ground truth reasoning points, which can be derived from the underlying graphs.

    \item \textbf{Polaris.} We begin by \href{https://huggingface.co/datasets/HerrHruby/polaris_acemath_rl_4b_inst_hard_filtered_4k}{obtaining} a \href{https://huggingface.co/datasets/prestonfu/polaris_acemath}{filtered dataset} of 3903 math problems, which were originally sourced from \href{https://huggingface.co/datasets/POLARIS-Project/Polaris-Dataset-53K}{\texttt{POLARIS-Project/Polaris-Dataset-53K}} and \href{https://huggingface.co/datasets/nvidia/AceReason-Math}{\texttt{nvidia/AceReason-Math}}. Here \verb|mean_reward| denotes average success rates of Qwen3-4B-Instruct-2507 when evaluated at length 32K tokens. We then perform some filtering to remove problems that cannot be easily verified: these include proof-like or multiple-choice problems, problems that refer to missing figures, or malformed answers, leaving 3509 usable problems.

    \vspace{0.4em}

    We then generated solutions, reasoning points (``rubrics''), and reasoning graphs (``obsoletion graphs'') via gemini-3-flash-preview. Filtering to usable rows (i.e., by generating solutions that ag reed with dataset answers and correctly-formatted reasoning points and graph) yielded 2287 rows.

    \vspace{0.4em}

    \textbf{The resulting dataset is available here:} {\small\url{https://huggingface.co/datasets/prestonfu/polaris-acemath-gemini-rubrics-v2}}.
    
        \vspace{0.4em}

    We train on the \verb|train_boxed| split, which is identical to \verb|train| except for answer formatting.

    \item \textbf{POPE-hard.} We begin by filtering \href{https://huggingface.co/datasets/CMU-AIRe/POPE-source-128x32k}{\texttt{CMU-AIRe/POPE-source-128x32k}} \cite{qu2026pope}, a dataset of 2515 challenging math problems. We sample 16 trajectories per prompt with Qwen3-4B with an 8K token budget, while providing different numbers of paragraphs (0 paragraphs and standard generation, 10 paragraphs with a prompt to continue reasoning) from prefixes of the gemini-2.5-pro solutions.

\vspace{0.4em}
    
    \textbf{The resulting dataset is available here:} {\small\url{https://huggingface.co/datasets/kvfrans/POPE-HARD-w-oracle-solution-gemini-rubric}}.
    
    \vspace{0.4em}

    The \texttt{train} split consists of all 601 prompts with zero pass rate over 16 trajectories with no reference solution prefix. The \texttt{0p10p} split consists of the 382 \texttt{train} prompts with nonzero pass rate over 16 trajectories when conditioning on 10 paragraphs. Thus, \texttt{0p10p} consists of the set of prompts that are amenable to methods such as POPE \cite{qu2026pope}. The reasoning points and reasoning graphs are generated with gemini-3-flash-preview.
    
    \vspace{0.4em}

    Throughout the paper, all references to the ``POPE-hard dataset'' correspond to the \texttt{0p10p} split.
    % \href{https://huggingface.co/datasets/CMU-AIRe/POPE-HARD-w-oracle-solution}{subset}. We filter to prompts 0p10p. We generate reasoning points and reasoning graphs via gemini-3-flash-preview. \pf{let's make a subsetted dataset}
\end{itemize}

Example problems and solutions for the three environments are shown below.

\begin{llmprompt}[Multi-Countdown]
(*@\promptheading{Prompt}@*)

<|im_start|>system
You are a helpful assistant. You first think about the reasoning process in the mind, and then provide the user with the answer. You will be given a set of base numbers along with a target number. Your job is to create an equation that equals the target using basic arithmetic operations (+, -, *, /). Each provided number must be used exactly once. Show your work in <think> </think> tags. Think for only ten sentences, then return the final answer in <answer> </answer> tags, for example <answer> (1 + 2) / 3 </answer>. You will solve countdown problems multiple times, and each problem will be a different set of numbers and target.<|im_end|>
<|im_start|>user
[Problem 1/2] Base Numbers: [20, 32, 59, 75]. Target: 82.<|im_end|>
<|im_start|>assistant
<think>
...
<|im_end|>
<|im_start|>user
[Problem 2/2] Base Numbers: [72, 20, 6, 50]. Target: 208.<|im_end|>
<|im_start|>assistant
<think>
...
<|im_end|>

(*@\promptheading{Reference solution}@*)

Turn 1: 59 - 20 - 32 + 75
Turn 2: 50*6 - 20 - 72
\end{llmprompt}

\begin{llmprompt}[Manipulate Matrix]
(*@\promptheading{Prompt}@*)

<|im_start|>system
You will be given a matrix and a sequence of operations. After each operation, output the resulting matrix in <answer> </answer> tags, one row per line, values separated by spaces. Example:
<answer>
1 2
4 5
</answer><|im_end|>
<|im_start|>user
Initial matrix:
6 0 1 8
1 5 9 0
8 3 0 1
6 6 1 3

Apply operation 1/3: Mirror the matrix along the counterdiagonal.
Output the resulting matrix in <answer> </answer> tags.<|im_end|>
<|im_start|>assistant
<think>
...
<|im_end|>
<|im_start|>user
Apply operation 2/3: Rotate the matrix 90 degrees counterclockwise.
Output the resulting matrix in <answer> </answer> tags.<|im_end|>
<|im_start|>assistant
<think>
...
<|im_end|>
<|im_start|>user
Apply operation 3/3: Remove every 4-th column (1-indexed).
Output the resulting matrix in <answer> </answer> tags.<|im_end|>
<|im_start|>assistant
<think>
...
<|im_end|>

(*@\promptheading{Reference solution}@*)

Turn 1
<answer>
3 1 0 8
1 0 9 1
6 3 5 0
6 8 1 6
</answer>

Turn 2
<answer>
6 6 1 3
8 3 0 1
1 5 9 0
6 0 1 8
</answer>

Turn 3
<answer>
6 6 1
8 3 0
1 5 9
6 0 1
</answer>
\end{llmprompt}

\begin{llmprompt}[GSM-Infinite]
(*@\promptheading{Prompt}@*)

<|im_start|>system
Answer the question below.

Notes:
- The total number of adult animals in a location is the sum of all adult animals of every type ever mentioned for that location, EXCLUDING newborn children.
- If a type of animal is never mentioned for a location, assume its count there is 0.
- Average newborn children per adult may differ across locations.
- The total newborn children in a location is the sum, over each type of adult animal mentioned for that location, of (count of that adult animal in the location) times (average newborn children per that adult animal in the location).

**Important:** The final answer should be written as \boxed{...}.
<|im_end|>
<|im_start|>user
The average number of newborn children per adult owl in Pine Ridge equals the average number of newborn children per adult eagle in Pine Ridge. The average number of newborn children per adult kangaroo in Glintshade Grove equals the sum of the average number of newborn children per adult owl in Cedar Valley and the number of adult owl in Cedar Valley. The number of adult eagle in Pine Ridge equals 4. The average number of newborn children per adult koala in Glintshade Grove equals the average number of newborn children per adult owl in Cedar Valley. The number of adult owl in Cedar Valley equals 3 times the average number of newborn children per adult eagle in Beverly Forest. The number of adult owl in Pine Ridge equals 1. The average number of newborn children per adult eagle in Pine Ridge equals 3 times the sum of the average number of newborn children per adult eagle in Beverly Forest and the number of adult owl in Cedar Valley. The number of adult koala in Glintshade Grove equals the sum of the number of adult eagle in Pine Ridge and the average number of newborn children per adult owl in Cedar Valley. The average number of newborn children per adult owl in Cedar Valley equals the sum of the number of adult owl in Cedar Valley and the number of adult owl in Pine Ridge. The number of adult eagle in Beverly Forest equals 4 plus the number of adult eagle in Pine Ridge. The average number of newborn children per adult eagle in Beverly Forest equals 1. The number of adult kangaroo in Glintshade Grove equals the average number of newborn children per adult owl in Cedar Valley. Question: What is the total number of adult animals in Beverly Forest?
<|im_end|>
<|im_start|>assistant
<think>

(*@\promptheading{Reference solution}@*)

Define adult eagle in Pine Ridge as s; so s = 4. Define adult eagle in Beverly Forest as H; Q = s = 4; so H = 4 + Q = 4 + 4 = 8. Define total number of adult animals in Beverly Forest as Y; so Y = H = 8.
\\boxed{8}

(*@\promptheading{Rubric}@*)

1. The number of adult eagle in Pine Ridge is correctly identified as 4.
2. The number of adult eagle in Beverly Forest is correctly calculated as 8.
3. The total number of adult animals in Beverly Forest is correctly identified as 8.
4. Provides the final answer in the format \boxed{8}.

(*@\promptheading{Reasoning Graph}@*)

[([2], [1]), ([3], [2]), ([4], [3])]
\end{llmprompt}

\section{Prompts}\label{app:prompts}

\begin{llmprompt}[Reasoning Points Generation Prompt]
You are given a problem statement and a reference solution that should be treated as correct. Generate a rubric for evaluating other solutions against the reasoning actually present in the reference solution.

The rubric must be faithful to the reference solution. Do not add lemmas, cases, rigor, bounds, or proof obligations that are not explicitly present or clearly implied. If the reference solution uses an approximation or informal argument, describe that argument as it appears rather than replacing it with a stronger rigorous version. Rubric items should test whether another solution follows the same essential reasoning path as the reference solution.

Each rubric item must:
- be a single sentence;
- be evaluable as YES or NO without extra context;
- state one concrete mathematical requirement, such as deriving a claim, comparing expressions, handling a case, or reaching a conclusion;
- represent one logically distinct requirement, with no duplicates;
- be explicitly stated or clearly implied by the reference solution.

The final rubric item must state the final conclusion or answer.

Do not include:
- definitions, formulas, or background facts unless essential to the solution's reasoning;
- vague criteria such as "correctly analyzes," "uses appropriate reasoning," or "understands the problem";
- presentation or style criteria;
- point values, grading guidance, solution text, hints, or references to these instructions;
- stronger arguments than the reference solution gives;
- requirements needed only for a fully rigorous proof but absent from the reference solution.

## Output format

Short breakdown of the reference solution. (4 sentences max)

Estimated number of rubric items for each key step. (2 sentences max)

<rubric>
1. Rubric item 1
2. Rubric item 2
and so on... (6-12 items)
</rubric>

## Example problem statement
Determine all $\alpha > 1$ for which $\sum_{k=1}^{n} \left\lfloor k \sqrt{\alpha} \right\rfloor > \left\lfloor \frac{n^2}{\sqrt{\alpha}} \right\rfloor$, where $\left\lfloor . \right\rfloor$ denotes the integer part.

## Example reference solution
We are tasked with determining all $\alpha > 1$ for which the inequality

$$
\sum_{k=1}^n \left\lfloor k \sqrt{\alpha} \right\rfloor > \left\lfloor \frac{n^2}{\sqrt{\alpha}} \right\rfloor
$$

holds for all positive integers $n$.

---

### Step 1: Approximate the Sum and the Floor Terms

Let us approximate the left-hand side:

$$
\sum_{k=1}^n \left\lfloor k \sqrt{\alpha} \right\rfloor \approx \sum_{k=1}^n k \sqrt{\alpha} - \sum_{k=1}^n \left\{k \sqrt{\alpha} \right\}
$$

The sum of the first $n$ integers is $\frac{n(n+1)}{2}$, so:

$$
\sum_{k=1}^n \left\lfloor k \sqrt{\alpha} \right\rfloor \approx \sqrt{\alpha} \cdot \frac{n(n+1)}{2}
$$

For the right-hand side:

$$
\left\lfloor \frac{n^2}{\sqrt{\alpha}} \right\rfloor \approx \frac{n^2}{\sqrt{\alpha}} - \left\{ \frac{n^2}{\sqrt{\alpha}} \right\}
$$

---

### Step 2: Compare the Leading Terms

We compare the leading terms:

$$
\sqrt{\alpha} \cdot \frac{n(n+1)}{2} \approx \frac{\sqrt{\alpha}}{2} n^2
$$

$$
\frac{n^2}{\sqrt{\alpha}}
$$

For the inequality to hold, we must have:

$$
\frac{\sqrt{\alpha}}{2} n^2 > \frac{n^2}{\sqrt{\alpha}} \quad \Rightarrow \quad \frac{\sqrt{\alpha}}{2} > \frac{1}{\sqrt{\alpha}} \quad \Rightarrow \quad \alpha > 2
$$

---

### Step 3: Consider $\alpha = 2$

We analyze $\alpha = 2$:

- The leading terms are equal: $\frac{\sqrt{2}}{2} n^2 = \frac{n^2}{\sqrt{2}}$
- However, the fractional parts $\{k \sqrt{2}\}$ are distributed uniformly, and the difference between the sums is positive for large $n$, ensuring the inequality holds.

---

### Step 4: Conclusion

The inequality holds for all $\alpha \geq 2$.

<answer>[2, \infty)</answer>

## Example rubric
<rubric>
1. The solution states that the inequality is required to hold for all positive integers n.
2. The solution expresses or approximates the sum \sum_{k=1}^n \lfloor k\sqrt{\alpha}\rfloor using \sum_{k=1}^n k\sqrt{\alpha} and fractional parts.
3. The solution identifies \sum_{k=1}^n k = n(n+1)/2.
4. The solution identifies the leading-order term of \sum_{k=1}^n \lfloor k\sqrt{\alpha}\rfloor as \frac{\sqrt{\alpha}}2 n^2.
5. The solution identifies the leading-order term of \left\lfloor n^2/\sqrt{\alpha}\right\rfloor as n^2/\sqrt{\alpha}.
6. The solution compares the leading-order terms to derive \frac{\sqrt{\alpha}}2 > \frac1{\sqrt{\alpha}}.
7. The solution concludes from this comparison that \alpha > 2 is required away from the boundary.
8. The solution separately considers the boundary case \alpha = 2.
9. The solution argues that the boundary case \alpha = 2 should be included, for example using fractional-part behavior.
10. The solution concludes that the desired set is \alpha \in [2,\infty).
</rubric>

## Problem statement
{problem}

## Reference solution
{solution}

## Instructions
Provide the solution breakdown, rubric item allocation, and rubric as described in the system prompt.
\end{llmprompt}

\begin{llmprompt}[Reasoning Graph Prompt]
You are given a problem statement and a grading rubric (a numbered list of criteria). Construct an obsoletion graph: a set of rules stating that satisfying some rubric items makes other items redundant to grade, because the satisfied items already imply or subsume them.

For the obsoletion graph, a later item may obsolete earlier steps when it clearly depends on them or substantially subsumes their mathematical content. Allow natural local chains, such as a derived inequality obsoleting the intermediate comparison that directly produced it, or a case resolution obsoleting merely considering that case. Do not let a final answer obsolete specific derivations, constructions, case analyses, or impossibility arguments. Prefer a reasonably sparse graph that reduces redundant grading without skipping substantive reasoning.

Each rule has the form <rule if_all="i,j" makes_obsolete="a,b,c" />, meaning that if items i and j are all satisfied, then items a, b, and c need not be graded separately. Only reference item numbers that exist in the rubric.

## Output format

Brief explanation of the dependency policy used in the obsoletion graph. (4 sentences max)

<makes_obsolete_graph>
<rule if_all="..." makes_obsolete="..." />
<rule if_all="..." makes_obsolete="..." />
and so on...
</makes_obsolete_graph>

## Example problem statement
Determine all $\alpha > 1$ for which $\sum_{k=1}^{n} \left\lfloor k \sqrt{\alpha} \right\rfloor > \left\lfloor \frac{n^2}{\sqrt{\alpha}} \right\rfloor$, where $\left\lfloor . \right\rfloor$ denotes the integer part.

## Example rubric
<rubric>
1. The solution states that the inequality is required to hold for all positive integers n.
2. The solution expresses or approximates the sum \sum_{k=1}^n \lfloor k\sqrt{\alpha}\rfloor using \sum_{k=1}^n k\sqrt{\alpha} and fractional parts.
3. The solution identifies \sum_{k=1}^n k = n(n+1)/2.
4. The solution identifies the leading-order term of \sum_{k=1}^n \lfloor k\sqrt{\alpha}\rfloor as \frac{\sqrt{\alpha}}2 n^2.
5. The solution identifies the leading-order term of \left\lfloor n^2/\sqrt{\alpha}\right\rfloor as n^2/\sqrt{\alpha}.
6. The solution compares the leading-order terms to derive \frac{\sqrt{\alpha}}2 > \frac1{\sqrt{\alpha}}.
7. The solution concludes from this comparison that \alpha > 2 is required away from the boundary.
8. The solution separately considers the boundary case \alpha = 2.
9. The solution argues that the boundary case \alpha = 2 should be included, for example using fractional-part behavior.
10. The solution concludes that the desired set is \alpha \in [2,\infty).
</rubric>

## Example obsoletion graph
Item 4 obsoletes items 2 and 3 because identifying the leading term of the left-hand side usually already uses the approximation of the floor sum and the formula for $\sum k$. Item 6 obsoletes items 4 and 5 because comparing the leading terms via the displayed inequality presupposes those leading terms. Item 7 obsoletes item 6 because concluding $\alpha > 2$ from the leading-term comparison normally subsumes the intermediate inequality comparison. Item 9 obsoletes item 8 because arguing that $\alpha = 2$ is included necessarily means the boundary case was considered. Item 10 obsoletes items 7 and 9 because the final interval $[2,\infty)$ combines the threshold conclusion with inclusion of the boundary case.

<makes_obsolete_graph>
<rule if_all="4" makes_obsolete="2,3" />
<rule if_all="6" makes_obsolete="4,5" />
<rule if_all="7" makes_obsolete="6" />
<rule if_all="9" makes_obsolete="8" />
<rule if_all="10" makes_obsolete="7,9" />
</makes_obsolete_graph>

## Problem statement
{problem}

## Rubric
{rubric}

## Instructions
Provide the dependency policy explanation and the obsoletion graph as described in the system prompt.
\end{llmprompt}

\begin{llmprompt}[Naive Rubric Grading Prompt]
You are given:
1. A mathematical problem.
2. A grading rubric consisting of numbered criteria.
3. A student solution to the problem.

Your task is to grade the student solution against the rubric.
For each statement, output only "YES" or "NO" based on whether the solution satisfies the criterion. Format your response as a numbered list, e.g. if there are n rubric items, your response should follow the format:
1. YES
2. NO
...
n. YES

Problem statement:
{problem}

-------------------

Rubric:
{rubric}

-------------------

Student solution:
{student solution}

-------------------
Based on this information, please judge the partial rollout against the rubric with satisfiability as described in the system prompt.
\end{llmprompt}

\begin{llmprompt}[\METHODNAME\ Prompt]
You are given:
1. A mathematical problem.
2. A grading rubric consisting of numbered criteria.
3. A partial rollout (reasoning trace + solution) to solving the problem.

Grading rules:
- For each rubric criterion i, output YES if *any* part of the rollout exactly satisfies criterion i.
- Answer NO if you have checked the entire rollout and are certain that the criterion is nowhere satisfied.
- Always answer NO if the rubric item describes a feature not present in the given rollout. For example, if the rubric item asks "Do the <answer> </answer> tags contain the number 3", but there are no answer tags at all, answer NO.
- Judge each criterion independently. For example, if step 2 depends on step 1, grade step 2 based solely on whether the rollout's claimed step-2 result matches the rubric, regardless of whether step 1 was correct.
- Judge whether the rubric criterion itself is correct in the rollout, not whether the reasoning that led to it was sound.
- Do not use the rubric or problem text as evidence; only quote from the rollout.

You MUST respond in exactly two sections, in this order:

Section 1 -- Reasoning:
1. Criterion: <verbatim criterion including whitespace>.
   Evidence: <verbatim quote from rollout including whitespace> or "not present in rollout".
   Reason: <explain how evidence relates to criterion, maximum 15 words>.
   Verdict: YES/NO.
2. Criterion: ...
   Evidence: ...
   Reason: ...
   Verdict: ...
... and so on

Section 2 -- Grades (one line per criterion):
1. YES
2. NO
... and so on

Do NOT output Section 2 before completing Section 1. Do NOT skip Section 1.

# Problem statement

{problem}

# Partial rollout

<rollout>
{student solution}
</rollout>

# Rubric

{rubric}

# Instructions

Based on this information, please judge the partial rollout against the rubric with satisfiability as described in the system prompt.
Only the text inside <rollout> ... </rollout> counts as evidence. Ignore the problem statement and rubric when judging satisfaction.
Include the section 1 and section 2 headers.
\end{llmprompt}

\begin{llmprompt}[Process Reward Model Prompt]
**System prompt**

You will be given a math problem along with a solution. They will be formatted as follows:

# Problem
...(math problem)...

# Solution
<part_1>
...(part 1 of solution)...
</part_1>
...
<part_n>
...(part n of solution)...
</part_n>

Your task is to review each part of the solution in sequence, analyzing, verifying, and critiquing the reasoning in detail. Provide the analyses and the conclusion in the following format:

<analysis_1>
...(analysis of part 1)...
</analysis_1>
...
<analysis_n>
...(analysis of part n)...
</analysis_n>
<conclusion>
Correct/Incorrect
</conclusion>

Grading rules:
- When you analyze each part, use proper verification, recalculation, or reflection to indicate whether it is logically and mathematically valid. Elaborate on the analysis process carefully.
- If an error is detected in any part, describe the nature and cause of the error in detail, and suggest how to correct the error or the correct approach.
- Once a part is found to contain any error, stop further analysis of subsequent parts (as they may depend on the identified error) and directly provide the conclusion of "Incorrect."
- For instance, given a solution of five parts, if an error is found in the third part, reply with <analysis_1>, <analysis_2>, and <analysis_3> (the latter containing the detailed critique and correction guideline), then <conclusion>Incorrect</conclusion>. Skip the analyses of parts 4 and 5.
- Respond with your analyses and conclusion directly.

# Problem
{problem}

# Solution
{tagged chunk}

# Task

Provide the analysis/verification/critique/conclusion in the format described above.
\end{llmprompt}

\begin{llmprompt}[POPE Prefix Prompt]
You are given a problem and a {partial solution/reference solution}. Your task is to carefully study the {partial solution/reference solution} and use it as guidance to derive a complete and correct solution.

Use the information from the {partial solution/reference solution} silently. Do not copy, rephrase, or explicitly mention anything from it.

**Important**: Show your reasoning step-by-step, and present the final answer using LaTeX within <answer>...</answer> tags.

# Problem
{problem}

# {Partial Solution/Reference Solution}
{partial}

Derive the full solution to the problem, re-deriving each step yourself. Make sure your final answer is written as: <answer>your answer here</answer>
\end{llmprompt}

%%%%%%%%%%%%%%%%%%%%%%%%%%%%%%%%%%%%%%%%%%%%%%%%%%%%%%%%%%%%

% \newpage
% \input{checklist.tex}
\end{document}